\documentclass[journal,12pt,draftclsnofoot,onecolumn]{IEEEtran}

\usepackage{cite}

\usepackage[pdftex]{graphicx}

\usepackage{graphicx}
\usepackage{comment}
\usepackage{booktabs}
\usepackage{amsmath,amssymb}
\usepackage{bm}

\usepackage[ruled]{algorithm2e}
\usepackage{algorithmic}
\usepackage[tight,footnotesize]{subfigure}
\usepackage{amsthm}
\usepackage{array}
\usepackage{multirow}
\usepackage{multicol}
\usepackage{ctable}
\newtheorem{mytheorem}{Theorem}

\newcommand{\argmin}{\operatornamewithlimits{argmin}}   

\def\bmath#1{\mbox{\boldmath$#1$}}
\makeatletter
\DeclareRobustCommand\onedot{\futurelet\@let@token\@onedot}
\def\@onedot{\ifx\@let@token.\else.\null\fi\xspace}

\def\eg{\emph{e.g}\onedot} 
\def\ie{\emph{i.e}\onedot}

\makeatother
\newlength{\Oldarrayrulewidth}

\usepackage{rotating}
\definecolor{cl1}{rgb}{0,0,1}
\definecolor{cl2}{rgb}{0.8039,0.8039,0}
\definecolor{cl3}{rgb}{0.8039,0.1608,0.5647}
\definecolor{cl4}{rgb}{0.5451,0.2784,0.1490}
\definecolor{cl5}{rgb}{0.4980,1,0.8314}
\definecolor{cl6}{rgb}{0.5686,0.1725,0.9333}
\definecolor{cl7}{rgb}{1,1,0}
\definecolor{cl8}{rgb}{0.8549,0.4392,.8392}
\definecolor{cl9}{rgb}{1,0.2157,0.1686}
\definecolor{cl10}{rgb}{0.0941,0.6863,0.4706}
\definecolor{cl11}{rgb}{0.0471,0.3333,0.7529}
\definecolor{cl12}{rgb}{0.0784,0.8314,0.6392}
\definecolor{cl13}{rgb}{0.0431,0.5686,0.8784}

\newcolumntype{C}{>{\scriptsize}c}

\usepackage{graphicx}
\graphicspath{{./}}
\begin{document}
%

\title{Nonlinear Dynamic Field Embedding: On Hyperspectral Scene Visualization}

%
\author{Dalton~Lunga and~Okan~Ersoy%
\thanks{D. Lunga is with the Department
of Electrical and Computer Engineering, Purdue University. He is also affiliated with the CSIR-Meraka Institute,\ Brummeria, \ Pretoria,\ South Africa e-mail: dlunga$@$purdue-edu.}
\thanks{O. Ersoy is with the Department
of Electrical and Computer Engineering, Purdue University, West Lafayette,
IN, 47907-0501, USA e-mail: ersoy$@$purdue-edu.}}
\maketitle

\begin{abstract}
Graph embedding techniques are useful to characterize spectral signature relations for hyperspectral images. However, such images consists of disjoint classes due to spatial details that are often ignored by existing graph computing tools. Robust parameter estimation is a challenge for kernel functions that compute such graphs. Finding a corresponding high quality coordinate system to map signature relations remains an open research question. We answer positively on these challenges by first proposing a kernel function of spatial and spectral information in computing neighborhood graphs. Secondly, the study exploits the force field interpretation from mechanics and devise a unifying nonlinear graph embedding framework. The generalized framework leads to novel unsupervised multidimensional artificial field embedding techniques that rely on the simple additive assumption of pair-dependent attraction and repulsion functions. The formulations capture long range and short range distance related effects often associated with living organisms and help to establish algorithmic properties that mimic mutual behavior for the purpose of dimensionality reduction. The main benefits from the proposed models includes the ability to preserve the local topology of data and produce quality visualizations \ie maintaining disjoint meaningful neighborhoods. As part of evaluation, visualization, gradient field trajectories, and semisupervised classification experiments are conducted for image scenes acquired by multiple sensors at various spatial resolutions over different types of objects. The results demonstrate the superiority of the proposed embedding framework over various widely used methods.
\end{abstract}

%

\section{Introduction}
{\bf N}onlinear dimensionality reduction has emerged as a key preprocessing step for extracting and visualizing regular structures within complex data sets. Prior to its popularity, the easy to implement linear methods have spanned decades of research on this task dating back to principal component analysis(PCA) \cite{Pearson1901,Joliffe86,Martinez2001}, the classical scaling or multidimensional scaling(MDS) technique\cite{Torgerson52}, more recently the local Fisher discriminant analysis (LFDA)\cite{Sugiyama07}, and semi-supervised local Fisher discriminant analysis (SELF)\cite{Sugiyama10}. With modern technology enabling the capability to gather and combine data from various sensing mechanisms, linear dimensionality reduction techniques have met their limitations in seeking meaningful structures from complex data \cite{Lee2007}. For example, hyperspectral sensors have enabled the acquisition of greater details about objects on the earth surface which poses a challenge for linear dimensionality reduction techniques. Feature extraction in such data sets can be accomplished by employing nonlinear methods such as the maximum variance unfolding ({\tt MVU})\cite{Song07}- a method that computes maximum variance embedding maps subject to preserving local distances, the locally linear embedding({\tt LLE})\cite{Roweis00} - a method that represents the relations of each neighborhood by linear coefficients that best reconstruct each data point from its neighbors, and the laplacian eigenmaps Laplacian({\tt LE}) \cite{Belkin03}, which draws on the correspondence between the graph Laplacian, the Laplace Beltrami operator on a manifold, and the connections to the heat equation, to devise a geometrically motivated algorithm for constructing a representation for data sampled from a low $m$-dimensional manifold embedded in a higher $d$-dimensional space.

Further widely popular methods include the generalized-MDS\cite{DeSilva03,Bronstein08} which extends the properties of classical scaling to nonlinear embedding, the self-organizing feature map ( SOFM)\cite{Kohonen07} which is an artificial unsupervised neural network for learning low-dimensional maps, graph embedding and extensions for dimensionality reduction\cite{Yan07}. Probabilistic approaches include the stochastic neighbor embedding({\tt SNE})\cite{Hinton02}, a method that represents each object by a mixture of widely separated low dimensional factors capturing some of the local structure, and establishes global formations of clusters of similar maps. With improved modeling assumptions, variants of SNE have led to high quality embedding visualizations, {\em e.g.} the student $t$-distribution based stochastic neighbor embedding(tSNE)\cite{Laurens08}, the elastic embedding algorithm \cite{Carreira10}, and the spherical stochastic neighbor embedding( sSNE)\cite{Lunga12}. In most cases, the approaches rely on high dimensional neighborhood graphs to capture the geometrical relations of observations and use graph weights as inputs for embedding purposes. The techniques are formulated under metric measures that include the geodesic distances, and for others, non-metric probability measures including Kullback Leibler divergences.

Notwithstanding individual differences in efficiency and their specific data applications, nonlinear dimensionality reduction methods share some features including, better compression, better visualization, and reduction of classifier input features. Each embedding technique represents an attempt to search for a coordinate representation that resides on the data manifold. Recent comparative studies conducted on various nonlinear embedding techniques strongly indicated the inability of existing manifold learning methods to handle disjoint class structures that exists in hyperspectral data \cite{Crawford2011}. In addition, many dimensionality reduction algorithms suffer from what is known as the crowding problem in the machine learning community\cite{Laurens08}. The crowding problem can be described as a tendency by which embedding techniques collapse maps towards the center of the embedding space resulting in increased class overlaps. Under such a phenomenon, many embedding algorithms fail to establish discriminative boundaries between structures of different classes. In this study several functional forms are proposed to address this challenge within a unifying framework for developing nonlinear dimensionality reduction algorithms with benefits that spans visualization, compression, and classification tasks.

While incorporating some benefits from existing methods, the study provides a unified multidimensional artificial force embedding( MAFE) framework for nonlinear dimensionality reduction with new perspectives and high quality embedding functional forms. MAFE models are based on seeking a minimum energy configuration state of a high dimensional neighborhood graph in a lower dimensional space. Underlying the framework is the premise that given a predefined graph structure, the optimal embedding coordinates should generate a corresponding low dimensional neighborhood graph that preserve the high dimensional pairwise relations. The embedding framework draws from the force field interpretation to suggest insightful design properties that lead to pair-dependent attraction and repulsion functions for composing novel potential embedding fields. The functions are superposed to generate an odd function that enforce the pairwise interaction fields in the embedding space. The generated interaction is such that at longer range the attraction force, which pulls similar maps towards each other, dominates, while the repulsion force dominates at short range distances. Under such an environment, all maps will experience a pushing and a pulling force from their relative neighbors due to the field generated by the potential functions. Of importance is to observe that the attraction function emphasizes the pulling together of similar maps while the repulsion function acts as a barrier that generates a pushing force for maps to be dispersed. Much of the local relations that reveal meaningful structures in optimal embeddings should already have been captured by the similarity kernel function that constructs the neighborhood input graph. As a second contribution, the study combines MAFE models with a local bilateral similarity kernel for mapping spatial and spectral signatures onto the lower dimensional space. This is demonstrated by embedding real world remote sensing images that are characterized by high spectral resolution pixels with several hundred channels. However, having many channels and nonlinearities in each pixel poses several challenges to conventional land cover methods since visually different objects tend to exhibit overlapping or similar spectral signatures.

The paper is structured as follows. A review of the force field formulations is introduced in Section \ref{sec:force field}. A general formulation of the multidimensional artificial field embedding(MAFE) framework is presented in Section \ref{sec:mafegeneral}. In Section \ref{sec:mafe connections}, we establish MAFE connections to existing popular nonlinear embedding techniques. In Section \ref{sec:mafe new}, using MAFE framework, we propose two novel techniques: the multidimensional artificial field embedding - bounded repulsion(MAFE-BR), and the multidimensional artificial field embedding - unbounded repulsion(MAFE-UR). In Section \ref{sec:experiments}, we describe the experimental setup and results on various data sets. A discussion of experiments and future work is presented in Section \ref{sec:discussion}. Finally, we present conclusions in Section \ref{sec:conclusions}.

\section{Force Field Motivation}\label{sec:force field}
Force field interpretations have a long history that relates to nature studies. In Biology, researchers have long studied nature and discovered that populations often appear in patterns of aggregation such as flocks of birds, schools of fish, and herds of mammals. Biological models that use forces between individuals that are analogous to physical forces, have since been developed. The study in \cite{Parr1927}, is an example of early work that proposed an idea of mutual interactions between individuals that were composed of attractions and repulsions with the goal of maintaining the group as a stable mass. This idea was developed further in \cite{Breder1951} by discussing the possibility of modeling forces between individual fish upon classical gravitation and electromagnetism. In a later paper\cite{Breder1954}, the author considered inverse power laws to model the repulsions and attractions between individuals, with repulsion stronger at short inter-individual distances. The work in \cite{Breder1954}, was compared to actual data collected from schools of fish, to obtain model parameters. The model was based on a simple constant attraction and a repulsion inversely proportional to the square of the inter-individual distance. The same laws governing animal behavior have been observed to be generalizable to guide artificial systems to carry out complex tasks by relying only on local interactions. In fact, in individual-based frameworks, the basis rule is that aggregation behavior is a result of an interplay between a short-ranged repulsion and a long-ranged attraction between the individuals. Such an intuition has been applied with success in control of multi-agent systems, stability analysis of social foraging swarms \cite{Gazi02c}, and robotic motion planning\cite{Latombe91}.

Most Biology studies \cite{Parr1927,Breder1951,Breder1954,Grunbaum97,Grunbaum98} and Control Engineering studies \cite{Xue11,Gazi02c,Latombe91}, seek to address questions that relates to maintaining a group as a stable mass, stability analysis, group cohesion and obstacle avoidance. Notwithstanding such growing recognition of the importance of force field formulation in those areas, their role in the allied area of signal and image processing remains to be fully appreciated.  This study is most concerned on drawing on the collective dynamics of aggregation behavior to devise algorithms that establish manifold formations given a predefined structure, \ie formation of multiple manifolds given a high dimensional neighborhood graph as a constraint. As such, we focus on the application of the force field intuition to the problem of dimensionality reduction and manifold learning based on a graph embedding optimization framework.

\section{Dynamic Graph Embedding Formulation}\label{sec:mafegeneral}
Let $\mathcal{G}=(\mathcal{E},\mathcal{V})$ be a finite undirected graph with vertex set $\mathcal{V}$, edge set $\mathcal{E}$ and with no self loops. The elements of $\mathcal{E}$ are designated as ideal springs. Furthermore, let $\mathcal{S}=\left\{(w_{ij},k_{ij})\right\}$ be the spring properties between each vertex $i$ and $j$ for all $\{(v_{i},v_{j})\}\in\mathcal{E}$, where $w_{ij}$ is the normalized or unnormalized length without compression or extension computed for each observed pair in $\bmath{Y}=\left\{\bmath{y}_{1},\bmath{y}_{2},\cdots,\bmath{y}_{N}\right\}$, where $\bmath{y}_{i}\in\mathbb{R}^{d}$, and $k_{ij}=1$ is the corresponding force constant. Let $\mathcal{G}_{\mathcal{S}}$ be a neighborhood spring graph, where $\mathcal{S}$ denotes the spring relation. An embedding of $\mathcal{G}_{\mathcal{S}}$ is an assignment of vertices into a $m$-dimensional Euclidean space $\mathbb{R}^{m}$. Let $\bmath{Z}=\left\{\bmath{z}_{1},\bmath{z}_{2},\cdots,\bmath{z}_{N}\right\}^{T}$ be the assigned embedding of $\mathcal{G}_{\mathcal{S}}$, where $\bmath{z}_{i}\in\mathbb{R}^{m}$ is the position of vertex $i$'s map. When framed as a graph embedding task, where on each vertex the approach imagines a particle in motion and each edge weight is dictated by spring force laws, the corresponding problem of dimensionality reduction simply becomes that of establishing a minimum energy configuration that is governed by the structure in $\bmath{W}=[w_{ij}]$. The modeling assumptions are such that the configuration yields maps that preserve pairwise relations characterized by the neighborhood graph $\mathcal{G}_{\mathcal{S}}$. Finding such a mapping is at the heart of every dimensionality reduction model, and it is the subject that is discussed shortly.

A mechanics interpretation of the graph embedding framework is presented as follows. Imagine the existence of a particle on every $\bmath{z}_{i}\in\bmath{Z}$, that is moving with the velocity of $\bmath{Z}$'s centroid. With the following change of notation to denote the embedding positions as a {\em state} of a graph, let $\mathcal{Z}=\left\{\bmath{z}_{1}^{T},\bmath{z}_{2}^{T},\cdots,\bmath{z}_{N}^{T}\right\}^{T}$ be a long vector in $\mathbb{R}^{Nm}$. Thus only consider the motion dynamics of individual maps, not the motion of the group. The approach assumes that all individual maps move simultaneously and each map $i$ is aware of the position of other vertices (maps), as well as the strength of corresponding force that defines each edge weight in the graph. The positions, $\bmath{z}_i$'s, of individuals relative to the group centroid can change through the rearrangements informed by pair-dependent force field interactions. Assuming such motion is to change in a continuous time, then the velocity as determined by the effect of group members on each vertex $i$, at position $\bmath{z}_i$ is described by
\begin{eqnarray}
 \dot{\bmath{z}}_{i}=\sum_{j=1, j\ne i}F^{ij}(\bmath{z}_{i}-\bmath{z}_{j}),\ \ i=1,\cdots, N
\label{eqn:latentmotion}
\end{eqnarray}
where $F^{ij}(\bmath{z}_{i}-\bmath{z}_{j})=(\bmath{z}_{i} - \bmath{z}_{j})\left\{F^{ij}_{r}(\|\bmath{z}_{i} - \bmath{z}_{j}\|)- F^{ij}_{a}(\|\bmath{z}_{i} - \bmath{z}_{j}\|)\right\}$ describes pairwise symmetric interactions between the $i$th and $j$th maps. Symmetry of the function follows from the fact that if map $i$ is attracted to map $j$, then $j$ is attracted to $i$. $F^{ij}_{r}:\mathbb{R}^{+}\rightarrow\mathbb{R}^{+}$ denotes the magnitude of the repulsion term, whereas $F^{ij}_{a}:\mathbb{R}^{+}\rightarrow\mathbb{R}^{+}$ represents the magnitude of the attraction term. The superposition of these two terms defines an interactive function that is the basis for the MAFE framework. Model insights and properties governing the choice of this function are presented in the following sections.

\subsection{Force Fields Function Properties}
Artificial force field formulations assume that at large distances, the attraction function dominates, and that on short distances the repulsion function dominates, while in between there is a unique distance at which both terms will balance - defining a central path in a similar manner to the barrier method\cite{Boyd04}. The choice of suitable force field embedding interaction functions, $F^{ij}(\bmath{z}_{i}-\bmath{z}_{j})$, are guided by the following properties:
\begin{enumerate}
\item There is pair-equilibrium distance $\epsilon_{ij}$ at which $F^{ij}_{r}(\epsilon_{ij})=F^{ij}_{a}(\epsilon_{ij})$, else $F^{ij}_{a}(\|\bmath{z}_{i}-\bmath{z}_{j}\|)>F^{ij}_{r}(\|\bmath{z}_{i}-\bmath{z}_{j}\|)$ for $\|\bmath{z}_{i}-\bmath{z}_{j}\|>\epsilon_{ij}$ or $F^{ij}_{a}(\|\bmath{z}_{i}-\bmath{z}_{j}\|)<F^{ij}_{r}(\|\bmath{z}_{i}-\bmath{z}_{j}\|)$ for $\|\bmath{z}_{i}-\bmath{z}_{j}\|<\epsilon_{ij}.$
\item $F^{ij}$ is an odd function, {\em i.e.} $F^{ij}(-(\bmath{z}_{i}-\bmath{z}_{j})) = -F^{ij}(\bmath{z}_{i}-\bmath{z}_{j})$, therefore symmetric with respect to the origin.
\item There exist pair dependent functions $U^{ij}_{att}\rightarrow\mathbb{R}^{+}\rightarrow\mathbb{R}^{+}$ and $U^{ij}_{rep}\rightarrow\mathbb{R}^{+}\rightarrow\mathbb{R}^{+}$ such that
\begin{eqnarray}
\nabla_{\bmath{z}_i}w_{ij}U_{att}^{ij}(\|\bmath{z}_{i}-\bmath{z}_{j}\|) = F^{ij}_{a}(\|\bmath{z}_{i}-\bmath{z}_{j}\|)(\bmath{z}_{i} - \bmath{z}_{j})\nonumber\\  \nabla_{\bmath{z}_i}U^{ij}_{rep}(\|\bmath{z}_{i}-\bmath{z}_{j}\|) = F^{ij}_{r}(\|\bmath{z}_{i}-\bmath{z}_{j}\|)(\bmath{z}_{i} - \bmath{z}_{j})\nonumber
\end{eqnarray}
\end{enumerate}
$U^{ij}_{att}$ and $U^{ij}_{rep}$ are viewed as artificial attraction and repulsion potential energy functions. The combined term $(\bmath{z}_{i} - \bmath{z}_{j})F^{ij}_{r}(\|\bmath{z}_{i} - \bmath{z}_{j}\|)$ represents the actual repulsion effect, whereas the term $-(\bmath{z}_{i} - \bmath{z}_{j})F^{ij}_{a}(\|\bmath{z}_{i} - \bmath{z}_{j}\|)$ represents the actual attraction effect. The vector $(\bmath{z}_{i} - \bmath{z}_{j})$ establishes the alignment on which the attraction and repulsion interaction forces acts along in opposing directions. These functions describe the reactive approach by potential fields in which trajectories of particles motion are not planned explicitly. Instead the interactions of every map with its neighbors is a superposition of fields that enable its position to cope with the changing environment of other maps.  The motion dynamics can be rewritten to reflect the resultant forces on each individual map as
\begin{eqnarray}
\dot{\bmath{z}}_{i} = -\sum_{j=1, j\ne i}\left\{\nabla_{\bmath{z}_{i}}w_{ij}U_{att}^{ij}(\|\bmath{z}_{i}-\bmath{z}_{j}\|) -\nabla_{\bmath{z}_{i}}U^{ij}_{rep}(\|\bmath{z}_{i}-\bmath{z}_{j}\|)\right\}\nonumber
\label{eqn:grad-motion}
\end{eqnarray}
The assumption made to envision each map as moving along the negative gradient has an implication, {\em i.e.} to achieve a minimum-energy configuration of the graph $\mathcal{G}_{\mathcal{S}}$, a choice of the attraction and repulsion potential functions should be such that the minimum of $U^{ij}_{att}(\|\bmath{z}_{i}-\bmath{z}_{j}\|)$ occurs on or around $\|\bmath{z}_{i}-\bmath{z}_{j}\|=0$, whereas the minimum of $-U^{ij}_{rep}(\|\bmath{z}_{i}-\bmath{z}_{j}\|)$ (or maximum of $U^{ij}_{rep}(\|\bmath{z}_{i}-\bmath{z}_{j}\|)$) occurs on or around $\|\bmath{z}_{i}-\bmath{z}_{j}\|\rightarrow \infty$, and that the minimum of the combination $U^{ij}_{att}(\|\bmath{z}_{i}-\bmath{z}_{j}\|)-U^{ij}_{rep}(\|\bmath{z}_{i}-\bmath{z}_{j}\|)$ occurs at $\|\bmath{z}_{i}-\bmath{z}_{j}\|=\epsilon_{ij}$, thus defining the stationery state of motion that exist between pairs $i$ and $j$.

Using the above framework, the reactive potentials that are effective on each individual map $i$ can be represented as
\begin{eqnarray}
U_{i}(\mathcal{Z}) = \sum_{j=1, j\ne i}\left\{w_{ij}U_{att}^{ij}(\|\bmath{z}_{i}-\bmath{z}_{j}\|) - U^{ij}_{rep}(\|\bmath{z}_{i}-\bmath{z}_{j}\|)\right\}
\label{vertex-pot-field}
\end{eqnarray}%
while the total superposed potential function on the neighborhood graph $\mathcal{G}_{\mathcal{S}}$ is defined by
\begin{eqnarray}
U(\mathcal{Z}) = \sum_{i=1}^{N}U_{i}(\mathcal{Z})
\label{eqn:pot-cost}
\end{eqnarray}%
Letting $\Omega$ be the set of attraction and repulsion functions $F^{ij}(\cdot)$ satisfying the embedding field properties, new embedding models can simply be derived by solving the following general optimization problem:
\begin{eqnarray}
\mathcal{Z}^{\star} = \argmin_{\mathcal{Z}\in\mathbb{R}^{Nm}}U(\mathcal{Z})
\label{eqn:generalopt}
\end{eqnarray}
where $\mathcal{Z}^{\star}$ describes the minimum-energy configuration state of $\mathcal{G}_{\mathcal{S}}$ in the lower dimensional space. With some parameter adjustments on $F^{ij}(\cdot)$, the embedding maps will converge to a minimum-energy configuration that yields the optimal maps. Such an embedding framework can be adapted as a general platform to develop new nonlinear dimensionality reduction algorithms but first, we establish its links to existing popular embedding techniques in the following section. 

\section{Connections To Existing Methods}\label{sec:mafe connections}
Many nonlinear techniques have been proposed for the tasks of visualization and dimensionality reduction. Even though with differences, they are applied in various fields as preprocessing building blocks for compression, visualization, and classification tasks. A basic question that we ask is whether some of the existing methods can be derived as special cases of the MAFE framework? A positive illustration to this question is presented together with reformulations of popular techniques including the stochastic neighbor embedding(SNE)\cite{Hinton02}, student-t stochastic neighbor embedding\cite{Laurens08}, Laplacian eigenmaps (LE) \cite{Belkin03}, Cauchy graph embedding \cite{Luo11}, and the spherical stochastic neighbor embedding(sSNE)\cite{Lunga12}. Such interpretations show that MAFE is a unifying framework that not only inherits some algorithmic benefits of various techniques, but also provides extended functional properties combined with strong intuitive insights for creating new algorithms.

\subsection{Stochastic Neighbor Embedding}\label{sec:sne}
Stochastic neighbor embedding \cite{Hinton02} is a method for preserving probabilities on lower dimensional manifolds that are nonlinear. SNE assumes that edge weights are antisymmetric probabilities $w_{ij}$ (i.e. $w_{ij}\neq w_{ji}$) of pairs of vertices being neighbors in the higher dimensional space. However, our presentation focuses on the symmetric version where  $w_{ij} = w_{ji}$ for all pairs of vertices. The high dimensional edge weights are defined using the Gaussian functions of the form,
\begin{eqnarray}
w_{ij} = \frac{\exp\{-\frac{\|\bmath{y}_{i}-\bmath{y}_{j}\|^{2}}{2\sigma_{i}}\}}{\sum_{r=1, r\ne i}\exp\{-\frac{\|\bmath{y}_{r}-\bmath{y}_{i}\|^{2}}{2\sigma_{i}}\}}
\label{eqn:Gaussianweights}
\end{eqnarray}
where $\sigma_{i}$ is computed using a binary search method ensuring that the entropy of the distribution $W_{i}$ is approximately $\log(k)$, with $k$ defining the effective number of neighbors. In the lower dimensional space, a symmetric Gaussian probability $\hat{w}_{ij}$ is assumed between each pair of embedding maps, {\em i.e.} the embedding graph weights are computed as
\begin{eqnarray}
\hat{w}_{ij} = \frac{\exp\{-\|\bmath{z}_{i}-\bmath{z}_{j}\|^{2}\}}{\sum_{r=1, r\ne i}\exp\{-\|\bmath{z}_{r}-\bmath{z}_{i}\|^{2}\}}
\end{eqnarray}
Each $\bmath{z}_{i}\in\mathbb{R}^{m}$ is the corresponding lower dimensional map of the observation $\bmath{y}_{i}\in\mathbb{R}^{d}$ . SNE proceeds to compute for the maps by minimizing a sum of Kullback Leibler(KL) objective functions
\begin{eqnarray}
\sum_{i}KL(W_{i}||\Hat{W}_{i}) = \sum_{i}\sum_{j=1, j\ne i} w_{ij}\log(\frac{w_{ij}}{\hat{w}_{ij}})
\label{eqn:sne}
\end{eqnarray}
The goal of \eqref{eqn:sne} is to minimize the distortion between each of the $N$ high dimensional neighborhood distributions $W_{i}$ and their corresponding lower dimensional neighborhood distributions $\hat{W}_i$. The results obtained from this approach have so far demonstrated its superiority when compared to methods that include locally linear embedding(LLE)\cite{Roweis00}, MDS\cite{Torgerson52}, and  Isomap\cite{Tenenbaum00}. However, the optimization algorithm is very unstable which leads to a lot of experimentally defined parameters in order to attain meaningful results. A further expansion on \eqref{eqn:sne} while ignoring terms that are not a function of the lower dimensional maps (terms that do not depend on $\hat{w}_{ij}$), much in parallel to the work of \cite{Carreira10}, reveals the log-sum term as a source of difficulty when computing the gradient and its a term that increases the nonlinearity of the model.

Computing the negative gradient of \eqref{eqn:sne} yields the corresponding MAFE motion dynamics or force field equations of the form
\begin{eqnarray}
\dot{\bmath{z}_{i}}^{SNE}=-\nabla_{\bmath{z}_{i}}U^{SNE} = -4\sum_{r=1, r\ne i}F^{SNE}(\bmath{z}_{i}-\bmath{z}_{j})
\label{eqn:sne-gradient}
\end{eqnarray}
where the expression under summation is defined as
\begin{eqnarray}
F^{SNE}(\bmath{z}_{i}-\bmath{z}_{j}) = (\bmath{z}_{i}-\bmath{z}_{j})\left\{w_{ij} - \frac{\exp\{-\|\bmath{z}_{i}-\bmath{z}_{j}\|^{2}\}}{\sum_{r=1, r\ne i}\exp\{-\|\bmath{z}_{r}-\bmath{z}_{i}\|^{2}\}}\right\}
\label{eqn:sne-repulsion-attraction}
 \end{eqnarray}
Under the MAFE formulation, equation \eqref{eqn:sne-gradient} describes the motion state of vertex maps seeking the graph's minimum energy configuration in the lower dimensional space. A simple observation identifies the attractive gradient force field to be  $-w_{ij}(\bmath{z}_{i}-\bmath{z}_{j})$, and a repulsion gradient force field to be $\frac{(\bmath{z}_{i}-\bmath{z}_{j})\exp\{-\|\bmath{z}_{i}-\bmath{z}_{j}\|^{2}\}}{\sum_{r=1, r\ne i}\exp\{-\|\bmath{z}_{r}-\bmath{z}_{i}\|^{2}\}}$. The interpretation of \eqref{eqn:sne-gradient} is such that, at longer distances with parameters set carefully, the embedding maps start to form clusters due to the strong attraction force field, while the repulsion force field is very negligible. As $\|\bmath{z}_{i}-\bmath{z}_{j}\|\rightarrow 0$ for each  $(i,j)$ pair, the repulsion magnitude dominates the interaction force vector. This causes the maps to push away from each other. The convergence of the algorithm is established when the forces balance, much in the same way as described for the general MAFE framework.

\subsection{t-Stochastic Neighbor Embedding}
$t$-Stochastic Neighbor Embedding\cite{Laurens08} is similar to SNE except that the lower dimensional maps are assumed to be better modeled by a {\em Student t-distribution} of degree one. This simple modification leads to a complete improvement of results over SNE. The improvement is due to the pair-dependent inverse distance relation introduced by the Student t-distribution. Using this distribution, the joint embedding probabilities $\hat{w}_{ij}$ are defined by
\begin{eqnarray}
\hat{w}_{ij} = \frac{(1 +\|\bmath{z}_{i}-\bmath{z}_{j}\|^{2})^{-1}}{\sum_{r=1, r\ne i} (1 +\|\bmath{z}_{r}-\bmath{z}_{i}\|^{2})^{-1}}
\end{eqnarray}%
In tSNE, the cost function is also based on the sum of Kullback Leibler(KL) divergences (same as in \eqref{eqn:sne}). The corresponding negative gradient of the expanded tSNE cost function gives the following MAFE motion dynamics or force field equations,
\begin{eqnarray}
\dot{\bmath{z}_{i}}^{tSNE}&=&-\nabla_{\bmath{z}_{i}}U^{tSNE}\nonumber\\
&=&-4\sum_{r=1, r\ne i}\left\{w_{ij}\frac{(\bmath{z}_{i}-\bmath{z}_{j})}{1 +\|\bmath{z}_{i}-\bmath{z}_{j}\|^{2}} - \frac{(\bmath{z}_{i}-\bmath{z}_{j})(1 +\|\bmath{z}_{i}-\bmath{z}_{j}\|^{2})^{-2}}{\sum_{r=1, r\ne i} (1 +\|\bmath{z}_{r}-\bmath{z}_{i}\|^{2})^{-1}}\right\}\nonumber\\
&=&-4\sum_{(i,j)\in\mathcal{E}}F^{tSNE}(\bmath{z}_{i}-\bmath{z}_{j})
\label{eqn:tsne-gradient}
\end{eqnarray}
where the term under summation in \eqref{eqn:tsne-gradient} is defined as
\begin{eqnarray}
F^{tSNE}(\bmath{z}_{i}-\bmath{z}_{j}) = (\bmath{z}_{i}-\bmath{z}_{j})\left\{\frac{w_{ij}}{1 +\|\bmath{z}_{i}-\bmath{z}_{j}\|^{2}} - \frac{(1 +\|\bmath{z}_{i}-\bmath{z}_{j}\|^{2})^{-2}}{\sum_{r=1, r\ne i} (1 +\|\bmath{z}_{r}-\bmath{z}_{i}\|^{2})^{-1}}\right\}
\label{eqn:tsne-repulsion-attraction}
 \end{eqnarray}%
In a similar manner, equation \eqref{eqn:tsne-gradient} describes the motion state of vertex maps seeking the graph's minimum energy configuration in the lower dimensional space. However, the attractive force field is described by $-(\bmath{z}_{i}-\bmath{z}_{j})\frac{w_{ij}}{1 +\|\bmath{z}_{i}-\bmath{z}_{j}\|^{2}}$, a term whose magnitude approaches an inverse square law for large pairwise distances $\|\bmath{z}_{i}-\bmath{z}_{j}\|$. This property makes the coordinate representation of joint probabilities invariant to changes in scale for maps that are far-apart. The repulsion force field is described by $(\bmath{z}_{i}-\bmath{z}_{j})\frac{(1 +\|\bmath{z}_{i}-\bmath{z}_{j}\|^{2})^{-2}}{\sum_{r=1, r\ne i} (1 +\|\bmath{z}_{r}-\bmath{z}_{i}\|^{2})^{-1}}$. The magnitude of the inverse square law approximation by both terms dictates that formation of clusters be established at long range distances, while the repulsion force field is magnified at short range distances, causing all maps to disperse from each other. The convergence of the algorithm is established when the two force fields balance.

\subsection{sPherical Stochastic Neighbor Embedding}
A spherical stochastic neighbor embedding model\cite{Lunga12} is also a variant of SNE, whose embedding representations are constrained to reside on a spherical manifold. As such, it assumes the joint probability distribution $\Hat{W}_{i}$ to be defined on a spherical surface, $\mathbb{S}^{m} = \left\{\bmath{z}\in \mathbb{R}^{m+1} : \|\bmath{z}\|_{2}=1\right\}$, with its corresponding probable values $\hat{w}_{ij}$ obtained by means of an Exit distribution\cite{Durrett51}. sSNE, similarly to SNE, exhibits a desirable property that enables the unfolding of many-to-one mappings from which the same class objects are in several disparate locations that are spatially driven. The probable values $\hat{w}_{ij}$ are computed from the Exit expression
\begin{eqnarray}
\hat{w}_{ij} = \frac{\frac{(1-\varrho)}{\|\bmath{z}_{j}-\varrho\bmath{z}_{i}\|^{m}}}{\sum_{r=1, r\ne i} \frac{(1-\varrho)}{\|\bmath{z}_{r}-\varrho\bmath{z}_{i}\|^{m}}}
\end{eqnarray}
where $\varrho$ is a concentration parameter that controls the assignment of neighborhood weights for a given $\bmath{z}_{i}^{th}$ central map. sSNE's total cost function given by
\begin{eqnarray}
\sum_{i}^{N}KL(W_{i}||\Hat{W}_{i}) + \lambda(\bmath{z}^{T}_{i}\bmath{z}_{i} - 1) = \sum_{i}\sum_{j=1, j\ne i} w_{ij}\log(\frac{w_{ij}}{\hat{w}_{ij}}) +  \lambda(\bmath{z}^{T}_{i}\bmath{z}_{i} - 1)
\label{eqn:ssne}
\end{eqnarray}%
where $\lambda$ is the Lagrangian parameter enabling the implicit incorporation of the unit spherical constraint in the objective function.
The negative gradient of the expanded sSNE function yields the corresponding MAFE motion dynamic equations of the form
\begin{eqnarray}
\dot{\bmath{z}_{i}}^{sSNE}&=&-2m\varrho\sum_{j=1, j\ne i}(\bmath{z}_{j}-\varrho\bmath{z}_{i})\left\{\frac{w_{ij}}{\|\bmath{z}_{j}-\varrho\bmath{z}_{i}\|^{2}} -\frac{(\|\bmath{z}_{j}-\varrho\bmath{z}_{i}\|)^{-m-2}}{\sum_{r=1, r\ne i}\|\bmath{z}_{r}-\varrho\bmath{z}_{i}\|^{-m}}\right\} - 2\lambda\bmath{z}_{i}\nonumber\\
&=&2m\varrho\sum_{j=1, j\ne i}F^{sSNE}(\bmath{z}_{i}-\bmath{z}_{j}) - 2\lambda\bmath{z}_{i}
\label{eqn:ssne-gradient}
\end{eqnarray}
where the pairwise gradient force field on each map is defined by $F^{sSNE}(\cdot) = F^{E}(\cdot) - 2\lambda\bmath{z}_{i}$, with
\begin{eqnarray}
F^{E}(\bmath{z}_{i}-\bmath{z}_{j}) = -(\bmath{z}_{j}-\varrho\bmath{z}_{i})\left\{\frac{w_{ij}}{\|\bmath{z}_{j}-\varrho\bmath{z}_{i}\|^{2}} -\frac{(\|\bmath{z}_{j}-\varrho\bmath{z}_{i}\|^{-m-2})}{\sum_{r=1, r\ne i}\|\bmath{z}_{r}-\varrho\bmath{z}_{i}\|^{-m}}\right\}
\label{eqn:ssne-repulsion-attraction}
 \end{eqnarray}
Contrary to both SNE and tSNE, equation \eqref{eqn:ssne-gradient} describes a MAFE dynamic model with maps constrained to be on a spherical manifold and in turn defining a unit sphere neighborhood graph configuration. The attractive force field is described by $-(\bmath{z}_{j}-\varrho\bmath{z}_{i})\frac{w_{ij}}{\|\bmath{z}_{j}-\varrho\bmath{z}_{i}\|^{2}}$, a term whose magnitude is based on the inverse unbounded square law for large pairwise distances $\|\bmath{z}_{i}-\bmath{z}_{j}\|$. This property also makes the spherical coordinate representation of joint probabilities invariant to changes in scale for maps that are far-apart. The repulsion force field is described by $(\bmath{z}_{j}-\varrho\bmath{z}_{i})\frac{(\|\bmath{z}_{j}-\varrho\bmath{z}_{i}\|^{-m-2})}{\sum_{r=1, r\ne i} \|\bmath{z}_{r}-\varrho\bmath{z}_{i}\|^{-m}}$, also an unbounded inverse distance term. The unbounded magnitude of the inverse square law by both terms enables sSNE to exhibit much needed capability to introduce a splitting nature on clusters that overlap. However, due to the unbounded nature on both terms, as $\|\bmath{z}_{j}-\varrho\bmath{z}_{i}\|\rightarrow 0$, the minimum energy configuration of the graph becomes unstable. This introduces algorithmic inefficiencies when seeking the optimal embedding coordinates for large data sets\cite{Lunga12}.

\subsection{Laplacian and Cauchy Embedding}
All graph embedding techniques presented in this study make the assumption that if observation $i$ and $j$ are similar, {\em i.e.} $w_{ij}$ is large, then the proximity of their corresponding maps $\bmath{z}_i$ and  $\bmath{z}_j$ in the embedding space should also be closer. The $m$-dimensional embedding coordinates can also be obtained using the Laplacian eigenmaps (LE) approach \cite{Belkin03}. Considering the $1$-dimensional example, the idea is to minimize
\begin{eqnarray}
\argmin_{\bmath{z}\ s.t. \|\bmath{z}\|^{2}=1, \bmath{z}^{T}\bmath{1}=0} U^{LE}(\bmath{z})=\sum_{j=1, j\ne i}w_{ij}(\bmath{z}_{i}-\bmath{z}_{j})^{2}
\label{eqn:laplacian}
\end{eqnarray}
where $\bmath{1} =[1,\cdots,1]^{T}$ and $\|\bmath{z}\|^{2}=1$ is a magnitude constraint imposed to avoid obtaining a solution where all $\bmath{z}_{i}$'s are zero, whereas the constraint $\bmath{z}^{T}\bmath{1}=0$, reduces the uncertainty on the non-uniqueness of the embedding solution. Under such constraints the problem is equivalent to
\begin{eqnarray}
\argmin_{\mathcal{Z}} U^{LE}(\bmath{z})=\bmath{z}^{T}\bmath{L}\bmath{z}
\label{eqn:laplacian2}
\end{eqnarray}
where $\bmath{L} = \bmath{D} - \bmath{W}$ is the graph Laplacian,  $\bmath{D}=\mbox{diag}(d_{1},\cdots,d_{N}),d_{i}=\sum_{j}w_{ij}$ and $\bmath{W} =[w_{ij}]$ is the edge weight matrix. The solution is often computed by solving equation \eqref{eqn:laplacian2} as an eigenvector problem. A simple observation on \eqref{eqn:laplacian} connects the Laplacian eigenmaps formulation to the MAFE framework. A force field interpretation of the Laplacian embedding identifies, $\sum_{j=1, j\ne i}w_{ij}(\bmath{z}_{i}-\bmath{z}_{j})^{2}$, as the artificial attractive potential whose corresponding negative artificial force field provides the motion equation for maps as defined by
\begin{eqnarray}
\dot{\bmath{z}}^{LE}_{i} = -\frac{d U^{LE}(\bmath{z})}{dz_{i}} = -4\sum_{j=1, j\ne i}F^{LE}(z_{i}-z_{j})
\label{eqn:laplacian force}
\end{eqnarray}
where $F^{LE}(\bmath{z}_{i}-\bmath{z}_{j}) = (\bmath{D} - \bmath{W})_{ij}(z_{i}-z_{j})$. The dynamic Laplacian model of \eqref{eqn:laplacian force} does incorporate the magnitude and non-uniqueness constraints from its objective function. However, these constraints do not generate a repulsion field that can mitigate the crowding of embeddings. The Cauchy graph embedding \cite{Luo11} approach has a MAFE interpretation that follows the same steps as applied in reformulation of Laplacian eigenmaps.

\section{New MAFE Image Embedding Models}\label{sec:mafe new}
Apart from the key insights obtained from the force field properties to design suitable pair-depended attraction and repulsion potential energies for a MAFE embedding model, another essential component involves choosing a similarity function suitable for incorporating all local information that is relevant for computing a neighborhood graph that captures the local structure of high dimensional observations. In most popular techniques, the Gaussian kernel described in equation \eqref{eqn:Gaussianweights} is used to compute such relations. However, for certain application areas with spatial information, {\em e.g.} hyperspectral imagery, the Gaussian weights may not be suitable. As such, a parametric local kernel for computing the spectral neighborhood is discussed in the next section. The parametric kernel can be seen as an extension of bilateral filter function\cite{Honghong09,Kotwal07,Tomasi1998} with a robust high dimensional covariance estimation component that enhances pixel differences and generates a sparse neighborhood graph. A discussion on the design of practical and better kernel functions tend to be application-driven and it remains an open research problem. Its complete investigation is therefore not fully addressed in this study.

\subsection{Spectral Neighborhood Graph Similarities}\label{sec:similarity}
In various image processing applications, spatial preprocessing methods are often applied to remove noise and smooth images. These methods also enhance spatial texture information resulting in features that improve the performance of classification techniques. For example in \cite{Velasco-Forero09}, nonlinear diffusion partial differential equations (PDEs) and wavelet shrinkage were used for spatial preprocessing of hyperspectral images, and the results obtained demonstrated a significant improvement on classification performance. Unlike the use of PDEs, we adapt a {\em local bilateral filtering} approach to devise a pairwise pixel similarity function or kernel over the observed neighborhood graph. The local kernel function is defined by
\begin{eqnarray}
w(\bmath{s}_{i},\bmath{s}_{j},\bmath{y}_{i},\bmath{y}_{j}) = \exp\left\{\frac{-\|\bmath{s}_{i}-\bmath{s}_{j}\|^{2}}{\sigma_{s}^2}\right\}\cdot \exp\left\{-(\bmath{y}_{i}-\bmath{y}_{j})^{T}\bmath{\Sigma}_{y}^{-1}(\bmath{y}_{i}-\bmath{y}_{j})\right\}
\label{eqn:bisim}
\end{eqnarray}
where $\bmath{s}_{i}$ denotes the spatial coordinates of pixel $i$, $\bmath{y}_{i}$ denotes the photometric $d$-dimensional spectral vector, with $d$ corresponding to the number of spectral channels. The expression $\|\bmath{s}_{i}-\bmath{s}_{j}\|^{2}$ weighs image pixel values as a function of the spatial distance from the center pixel and $h_{s}$ is the variance parameter. The kernel also employs a nonlinear term,
\begin{eqnarray}  \tilde{w}_{p}(i,j)=\exp\left\{-\frac{1}{2}(\bmath{y}_{i}-\bmath{y}_{j})^{T}\bmath{\Sigma}_{y}^{-1}(\bmath{y}_{i}-\bmath{y}_{j})\right\}
 \label{eqn:Ky2}
\end{eqnarray}
which simply weighs pixel values as a function of the photometric differences between the center pixel and its neighbor pixels. For example, given $N$ hyperspectral pixels, organized into a zero-mean data matrix $\bmath{Y}=[\bmath{y}_{1}\bmath{y}_{2}\cdots\bmath{y}_{N}]\in\mathbb{R}^{d\times N}$, the sample covariance is computed as $\bmath{S}=\frac{1}{N}\bmath{Y}\bmath{Y}^{T} = \langle\bmath{y}\bmath{y}^{T}\rangle$, with the angle brackets denoting the average over $N$ pixels. Thus, $\bmath{S}$ is a $d\times d$ matrix whose diagonal components indicate the magnitude of noise variation in each of the $d$ spectral channels, and the off-diagonal elements denote the extent to which noise co-vary with each pair of spectral bands. It is easy to show\cite{Lunga12} that the photometric weights can be summarized by
\begin{eqnarray}
\tilde{w}_{p}(i,j) = \exp\left\{\frac{-N}{2}tr(\bmath{\Sigma}^{-1}\bmath{S})\right\}
\label{eqn:Ky22}
\end{eqnarray}

We make an observation that one can represent the unnormalized kernel $\mathcal{K}$ as a product of unnormalized gaussian functions, one for each pixel $\bmath{y}_{i}$, yielding
\begin{eqnarray}
\mathcal{K} &=& \exp\left\{\frac{-1}{2}\sum_{j=1}^{N}(\bmath{y}_{j}-\bmath{y}_{i})^{T}\bmath{\Sigma}^{-1}(\bmath{y}_{j}-\bmath{y}_{i})\right\}\nonumber\\
      &=& \exp\left\{\frac{-tr(\bmath{S}\bmath{\Sigma}^{-1})}{2} + \sum_{j=1}^{N}\bmath{y}_{j}\bmath{\Sigma}^{-1}\bmath{y}_{i} -\frac{N}{2}\bmath{y}^{T}_{i}\bmath{\Sigma}^{-1}\bmath{y}_{i}\right\}\nonumber
\label{eqn:Ky}
\end{eqnarray}
where $\bmath{\Sigma}^{-1}$ is the inverse covariance matrix, and $tr(\bmath{B})$ denotes the trace of matrix $\bmath{B}$. We could assume a zero-mean unnormalized Gaussian noise model over the pixels, i.e. we can simply subtract the center $\bmath{y}_{i}$ from the data, to obtain a simplified expression as
\begin{eqnarray}
\mathcal{K}&=& \exp\left\{\frac{-1}{2}\sum_{j=1}^{N}\bmath{y}_{j}^{T}\bmath{\Sigma}^{-1}\bmath{y}_{j}\right\}\nonumber\\
      &=& \exp\left\{\frac{-1}{2}tr(\bmath{Y}^{T}\bmath{\Sigma}^{-1}\bmath{Y})\right\}
\label{eqn:Ky2}
\end{eqnarray}
Note that $ tr(\bmath{Y}^{T}\bmath{\Sigma}^{-1}\bmath{Y}) = tr(\bmath{\Sigma}^{-1}\bmath{Y}^{T}\bmath{Y})=N tr(\bmath{\Sigma}^{-1}\bmath{S})$. This shows that $\bmath{S}$ is a sufficient statistic for characterizing the unnormalized likelihood (herein the photometric similarity) of data $\bmath{Y}$, and we can further write
\begin{eqnarray}
\mathcal{K} &=& \exp\left\{\frac{-N}{2}tr(\bmath{\Sigma}^{-1}\bmath{S})\right\}
\label{eqn:Ky2}
\end{eqnarray}

In practice the weights are computed by first decomposing the true covariance matrix into a product $\bmath{\Sigma}=\bmath{E}\Lambda\bmath{E}^{T}$, where $\bmath{E}$ is the orthogonal eigenvector matrix and $\bmath{\Lambda}$ is the corresponding diagonal matrix of eigenvalues, that easily compute the covariance matrix whose inverse is required. We adapt the efficient sparse matrix transform (SMT) approach in estimating the covariance matrix $\bmath{\Sigma}$ \cite{Theiler2011}. The SMT approach solves the optimization problem, $\hat{\bmath{E}} =\argmin_{\bmath{E}\in\bmath{\Omega}}\left\{|\mbox{diag}(\bmath{E}^{T}\bmath{S}\bmath{E})|\right\}$, and set $\hat{\bmath{\Lambda}} =\mbox{diag}(\hat{\bmath{E}}^{T}\bmath{S}\hat{\bmath{E}})$, where $\bmath{\Omega}$ is the set of allowed orthogonal transforms that can be computed using a series of {\em Givens rotations}\cite{Theiler2011}. A simple manipulation can show that $\bmath{\Sigma}^{-1} = \hat{\bmath{E}}\hat{\bmath{\Lambda}}^{-1}\hat{\bmath{E}}^{T}$ so that we have
\begin{eqnarray}
\tilde{w}_{p}(i,j) = \exp\left\{-\frac{1}{2}(\hat{\bmath{E}}^{T}\bmath{y}_{i}-\hat{\bmath{E}}^{T}\bmath{y}_{j})^{T}\hat{\bmath{\Lambda}}^{-1}(\hat{\bmath{E}}^{T}\bmath{y}_{i}-\hat{\bmath{E}}^{T}\bmath{y}_{j})\right\}\nonumber
\label{eqn:Kyfinal}
\end{eqnarray}
with the final graphs weights computed from
\begin{eqnarray}
w(\bmath{s}_{i},\bmath{s}_{j},\bmath{y}_{i},\bmath{y}_{j}) = \exp\left\{\frac{-\|\bmath{s}_{i}-\bmath{s}_{j}\|^{2}}{\sigma_{s}^2}\right\}\cdot \tilde{w}_{p}(i,j)
\label{eqn:bisim-final}
\end{eqnarray}
The SMT approach to computing the covariance matrix $\bmath{\Sigma}$ is efficient and robust in handling the singularities of $\bmath{\Sigma}$. Other approaches to computing $\bmath{\Sigma}$ have been used in the literature including the PCA adaptation approach \cite{Honghong09}, where the singularity of $\bmath{\Sigma}$ is not carefully addressed.

\subsection{Attractive Potential Function}
The main fundamental idea underpinning the design of multidimensional artificial field embedding models is to treat the pair-equilibrium distances $\epsilon_{ij}$ for vertices in $\mathcal{G}_{\mathcal{S}}$ as attractive wells. That is, consider the minimum-energy configuration between maps $i$ and $j$ as a sink for the potential energy function. For example, the attractive potential energy functions that we consider are bounded from below to allow for the existence of constant attraction force effects at all distances, that is $U_{att}^{ij}(\|\bmath{z}_{i}-\bmath{z}_{j}\|) \geq \alpha$, where $\alpha$ is a positive constant $\forall \ \|\bmath{z}_{i}-\bmath{z}_{j}\|$. In this study, we explore potential functions of the form
\begin{eqnarray}
U_{att}^{ij}(\|\bmath{z}_{i}-\bmath{z}_{j}\|) = \xi_{a}\|\bmath{z}_{i}-\bmath{z}_{j}\|^{p}
\label{attractive-potential}
\end{eqnarray}
where for values $0<p\leq 1$, the pairwise function is conic in shape, and the resulting attractive force field has a constant cluster formation amplitude determined from the graph's edge weights $w_{ij}$. $\xi_{a}$ is an attraction force magnitude related adaptive parameter. Figure \ref{fig:total-potential-bounded} shows the attractive potential energy generated from equation \eqref{attractive-potential} for $p=2,\xi_{a}=1$, and $w_{ij}=0.5$. The shape corresponds to quadratic potential, {\em i.e.} we have a global optimal that acts to pull all force fields effects in its direction, thereby demonstrating the sink nature of the minimum point.

\subsection{Repulsive Potential Functions}\vspace{-0.1cm}
The nature of a repulsion function is chosen such that as the distance between pair-points increases, its properties are deemed to have negligible influence on maps ({\em i.e.} maps are in long range zone where $F^{ij}_{r}<F^{ij}_{a}$). However, when the distance is small, the function generates a barrier or a repulsive force between maps ({\em i.e.} maps are in the short range zone where $F^{ij}_{r}>F^{ij}_{a}$).
The procedure to selecting a repulsive potential function $U^{ij}_{rep}$ starts by thinking of an indicator function of the form
\begin{eqnarray}
{\small
I_{+}(\|\bmath{z}_{i} - \bmath{z}_{j}\|)=\left\{
                                   \begin{array}{ll}
                                     0 & \ \|\bmath{z}_{i} - \bmath{z}_{j}\|>\epsilon_{ij} \\
                                     \infty & \ \  \|\bmath{z}_{i} - \bmath{z}_{j}\|\le\epsilon_{ij}.
                                   \end{array}
                                 \right.}
\label{eqn:indicator}
\end{eqnarray}
which is a  non-increasing function of distance.  Equation \eqref{eqn:indicator} best captures the behavioral form for a repulsive function that has negligible effects at large distances and has a large dominance at short range distances. With all its appealing intuitive properties, the indicator function is not a differentiable function. As such, we require its approximation by a differentiable function whose gradient can create a repulsion force $F^{ij}_r$ with a magnitude that is inversely proportional to the distance between pairs of maps {\em i.e.} $\|F^{ij}_{r}(\bmath{z}_{i} - \bmath{z}_{j})\|_{2} = \frac{1}{\textbf{dist}(\bmath{z}_{i}, \bmath{z}_{j})}.$  Such approximations can be chosen from {\em e.g.} Gaussian, Exponential, Cauchy, Hyperbolic Tangent and Inverse distance power functions. The behavior of such functions in approximating $I_{+}(\|\bmath{z}_{i} - \bmath{z}_{j}\|)$ is shown in Figure \ref{fig:decayfunctions}.
 \begin{figure}
  \centering
    \includegraphics[width=4.6in]{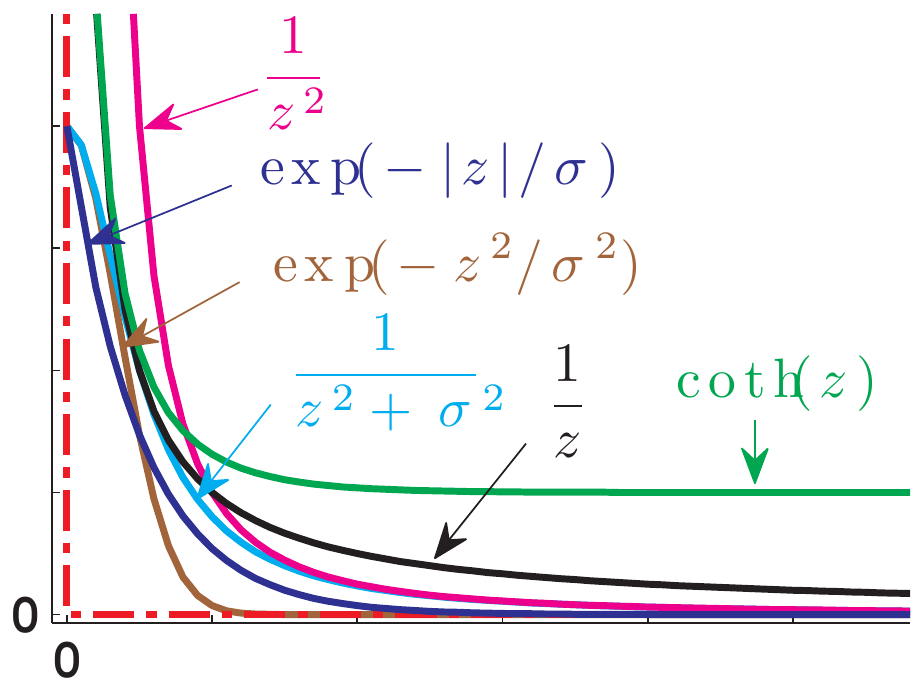}
  \caption{Dashed lines show the function $I_{+}(\bmath{z})$, and the solid curves show different forms of continuous decaying functions suitable for approximating $I_{+}(\bmath{z})$.}
\label{fig:decayfunctions}
\end{figure}

\subsubsection{Exponential Bounded Repulsion}
The Exponential or unnormalized Gaussian curves in Figure \ref{fig:decayfunctions} generates a continuous bounded approximation of \eqref{eqn:indicator}. As the distance between maps grows the magnitude of the curves approaches zero while a maximum magnitude is assigned for maps that are very close in distance. A general unnormalized Gaussian bounded repulsion function can be devised in the form,

\begin{eqnarray}
U_{rep}(\|\bmath{z}_{i}-\bmath{z}_{j}\|) =\xi_{r}\sigma\exp\{-\frac{\|\bmath{z}_{i}-\bmath{z}_{j}\|^{q}}{\sigma}\}
\label{repulsive-exponential-potential}
\end{eqnarray}
Where $\sigma_{r}$ is the bounded repulsion variance parameter. For $q=2$, the function has spherical symmetry as shown in Figure \ref{fig:total-potential-bounded}. For values $0<q\leq 1$ the repulsion potential field has the shape of a harmonic function often used in modeling obstacles in robotic path planning\cite{Latombe91}, while for $1<q<2$ it has the form of a tower  centered at the origin to generate equidistributed repulsive force fields in the direction of the gradient as shown in Figure \ref{fig:total-potential-bounded}.

\subsubsection{Inverse Power Unbounded Repulsion}
The inverse distance function as shown in Figure \ref{fig:decayfunctions} has a continuous best approximation of the indicator function in equation \eqref{eqn:indicator}. Its corresponding general repulsion potential function is given by,
\begin{eqnarray}
U^{ij}_{rep}(\|\bmath{z}_{i}-\bmath{z}_{j}\|) = \frac{\xi_{r}}{\|\bmath{z}_{i}-\bmath{z}_{j}\|^{q}}
\label{eqn:unbounded-rep}
\end{eqnarray}
where $q$ is a positive number. It is clear that as $\|\bmath{z}_{i}-\bmath{z}_{j}\| \rightarrow 0^{+}$ the repulsion becomes unbounded, i.e. for $q=2$, $\lim_{\|\bmath{z}_{i}-\bmath{z}_{j}\| \rightarrow 0^{+}}U^{ij}_{rep}(\|\bmath{z}_{i}-\bmath{z}_{j}\|)\|\bmath{z}_{i}-\bmath{z}_{j}\|=\infty$. The inverse square law structure of this function promotes an invariant representation of neighborhood similarities even with a change in scale for embedding points that are far apart. $\xi_{r}$ is an repulsion magnitude related parameter. Figure \ref{fig:total-potential-unbounded} shows the unbounded repulsion potential field generated from equation \eqref{eqn:unbounded-rep}. In sharp contrast to existing embedding methods, this function affords properties that exhibit collision avoidance of maps as their coordinates change in search for the pair-equilibrium distances $\epsilon_{ij}$. The strong repulsive force field generated by this function renders it more favorable for rejoining of clusters that may have split during the search for a minimum energy configuration state of $\mathcal{G}_{\mathcal{S}}$.

\subsection{Multidimensional Artificial Field Embedding with Bounded Repulsion}
A combination of the attractive and bounded unnormalized Gaussian repulsion functions yields a multidimensional artificial field embedding (MAFE-BR) model described by
\begin{eqnarray}
U(\mathcal{Z}) &=& \sum_{i=1}^{N}\sum_{j=1, j\ne i}\left\{W_{ij}U_{att}^{ij}(\|\bmath{z}_{i}-\bmath{z}_{j}\|) - U_{rep}^{ij}(\|\bmath{z}_{i}-\bmath{z}_{j}\|)\right\}\\
&=&  \sum_{i=1}\sum_{j=1, j\ne i}\left\{\xi_{a}w_{ij}\|\bmath{z}_{i}-\bmath{z}_{j}\|^{p} - \xi_{r}\sigma\exp\{-\frac{\|\bmath{z}_{i}-\bmath{z}_{j}\|^{q}}{\sigma}\}\right\}
\label{total-gaussian-potential}
\end{eqnarray}
For $q=2$, \eqref{total-gaussian-potential} yields a special case elastic embedding model \cite{Carreira10}. Computing the gradient provides information related to the direction and magnitude of motion for each individual map in the embedding space. This motion is described by
\begin{eqnarray}
 \dot{\bmath{z}}_{i} &=& -\sum_{j=1, j\ne i}(\bmath{z}_{i}-\bmath{z}_{j})\left\{\xi_{a}w_{ij}p\|\bmath{z}_{i}-\bmath{z}_{j}\|^{p-2} - \xi_{r}q\|\bmath{z}_{i}-\bmath{z}_{j}\|^{q-2}\exp\{-\frac{\|\bmath{z}_{i}-\bmath{z}_{j}\|^{q}}{\sigma}\}\right\}\nonumber
\label{pot-embedding-gaussian-repulsion}
\end{eqnarray}
An illustration of the gradient field for a point with strong attraction force field is shown in Figure \ref{fig:total-potential-bounded}.
\begin{figure}
\centering
\vspace{-3pt}
\hspace{-0.1cm}\includegraphics[width=3in]{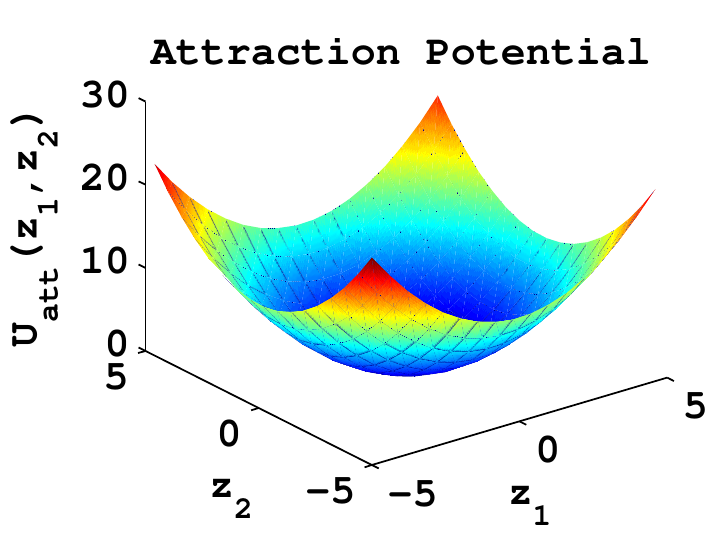}
\hspace{-0.1cm}\includegraphics[width=3in]{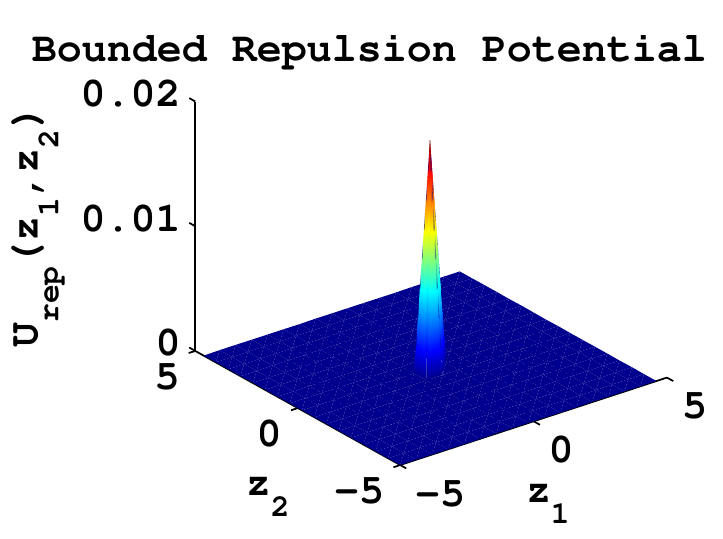}\\
\hspace{-0.1cm}\includegraphics[width=3in]{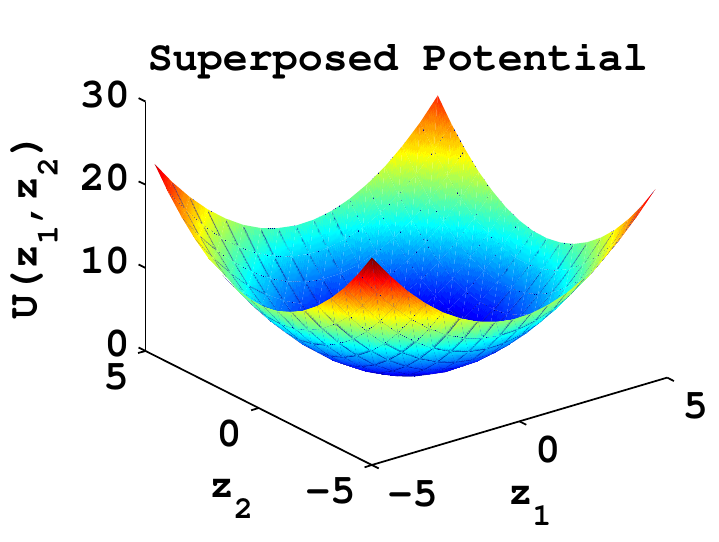}
\hspace{-0.1cm}\includegraphics[width=3in]{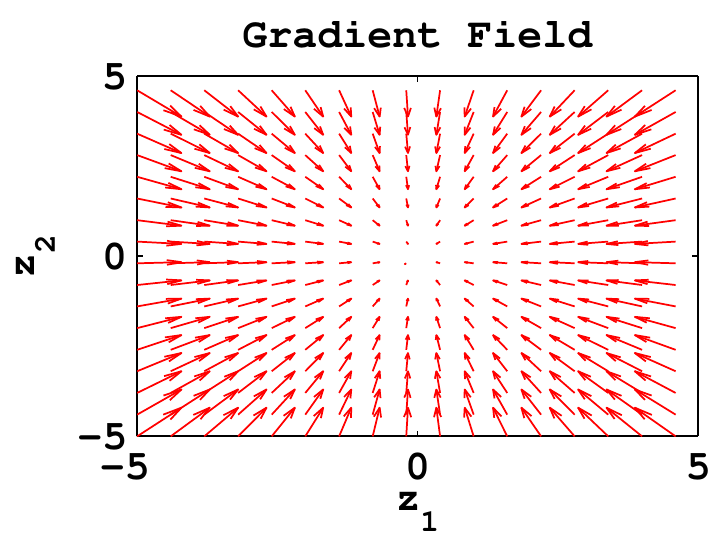}
\vspace{-4pt}
\caption{Illustrations of a superposed-potential function for $p=q=2$. Arrows indicate the negative gradient force field. The map located at $(z_{1},z_{2})=(0,0)$ has a very strong attraction force over a large range of distance. The repulsion force is significantly small and mostly effective over a very short range.}
\label{fig:total-potential-bounded}
\end{figure}

\subsection{Multidimensional Artificial Field Embedding with Unbounded Repulsion}
The second model that is proposed entails a multidimensional artificial field embedding model with an unbounded repulsion inverse distance function of \eqref{eqn:unbounded-rep}. By considering all pairwise map interactions for vertices of $\mathcal{G}_{\mathcal{S}}$, the resultant total potential function of a MAFE-UR model is given by
\begin{eqnarray}
U(\mathcal{Z}) &=& \sum_{i=1}\sum_{j=1, j\ne i}\left\{w_{ij}U_{att}^{ij}(\|\bmath{z}_{i}-\bmath{z}_{j}\|) - U^{ij}_{rep}(\|\bmath{z}_{i}-\bmath{z}_{j}\|)\right\}\nonumber\\
&=&  \sum_{i=1}\sum_{j=1, j\ne i}\left\{\xi_{a}w_{ij}\|\bmath{z}_{i}-\bmath{z}_{j}\|^{p} - \frac{\xi_{r}}{\|\bmath{z}_{i}-\bmath{z}_{j}\|^{q}}\right\}
\label{total-unbounded-potential}
\end{eqnarray}
Under the supposed artificial force field model, the motion of each map is given by $\dot{\bmath{z}}_{i}=-\nabla_{\bmath{z}_{i}}U(\bmath{z})$, {\em i.e.}
\begin{eqnarray}
 \dot{\bmath{z}}_{i} &=& -\sum_{j=1, j\ne i}\left\{\nabla_{\bmath{z}_{i}}w_{ij}U^{ij}_{att}(\|\bmath{z}_{i}-\bmath{z}_{j}\|) -\nabla_{\bmath{z}_{i}}U_{rep}^{ij}(\|\bmath{z}_{i}-\bmath{z}_{j}\|)\right\}\nonumber\\
&=& -2\sum_{j=1, j\ne i}(\bmath{z}_{i}-\bmath{z}_{j})\left\{\xi_{a}w_{ij}p\|\bmath{z}_{i}-\bmath{z}_{j}\|^{p-2} + \xi_{r}q\|\bmath{z}_{i}-\bmath{z}_{j}\|^{-q-2}\right\}
\label{pot-embedding-unbounded-repulsion}
\end{eqnarray}
Fig. \ref{fig:total-potential-unbounded}, shows the total potential field generated from equation (\ref{total-unbounded-potential}) and its corresponding gradient field obtained from (\ref{pot-embedding-unbounded-repulsion}) for $p=1,q=2$.

\begin{figure}
\centering
\vspace{-3pt}
\includegraphics[width=2.9in]{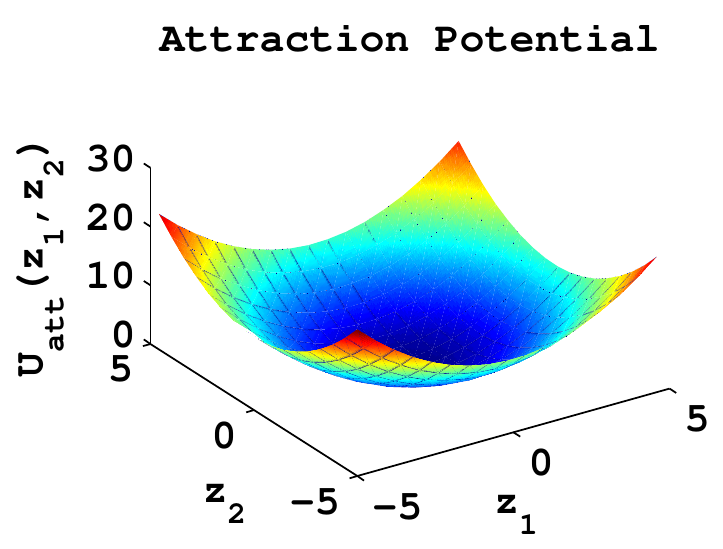}
\includegraphics[width=3.0in]{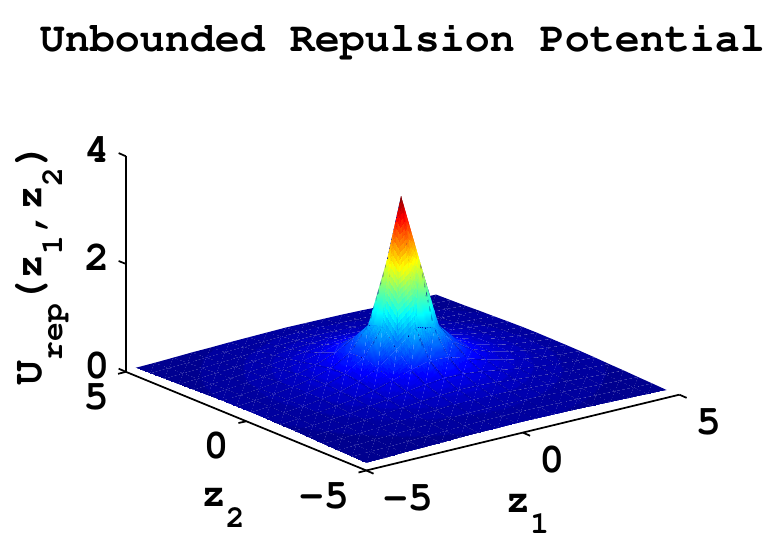}\\
\includegraphics[width=3.0in]{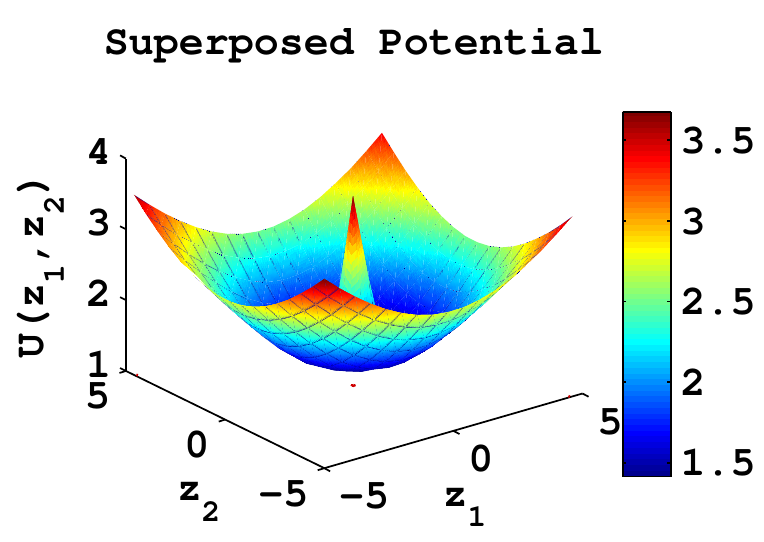}
\includegraphics[width=2.92in]{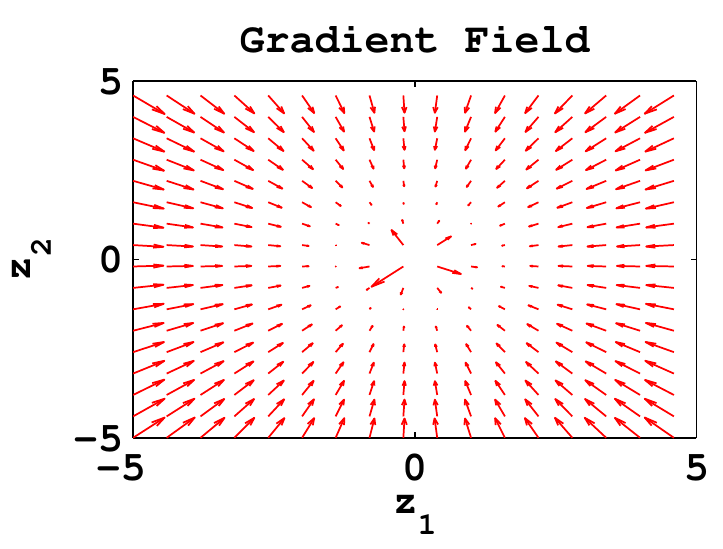}
\vspace{-4pt}
\caption{Illustrations of a superposed-unbounded repulsion potential function for $p=1,q=2$. Arrows indicate the negative gradient force field. The map located at $(z_{1},z_{2})=(0,0)$ has a very strong repulsion force over on short range distances, while the attraction force is dominant over long range distances.}
\label{fig:total-potential-unbounded}
\end{figure}

An important observation on the gradient fields in Figure \ref{fig:total-potential-bounded} and Figure \ref{fig:total-potential-unbounded} demonstrates that without an attraction term in each model, cluster formation would not occur since all pair-wise maps would disperse from each other; whereas by eliminating the repulsion term(by setting $\xi_{r}=0$), all maps would collapse to a single point leading to the crowding problem that has been a weakness in most existing embedding models. Table \ref{tab:GeneralFunction} summarizes typical force field functions that can be used to design embedding models that exhibit aggregation behavior.

\begin{table}
	\centering
	\caption{A summary of functions compatible to design multidimensional artificial force field embedding techniques.}
	\label{tab:GeneralFunction}
\begin{tabular}{c|c|c|c}
\toprule
  \textcolor[rgb]{1.00,0.00,0.00}{Technique}& \textcolor[rgb]{1.00,0.00,0.00}{Long Range Attraction}& \textcolor[rgb]{1.00,0.00,0.00}{Short Range Repulsion} & \\
  \hline
  \textcolor[rgb]{0.00,0.00,1.00}{SNE} & $\sum_{i,j=1}^{N}w_{ij}\|\bmath{z}_{i}-\bmath{z}_{j}\|^{2}$ & $\sum_{i=1}^{N}\log\sum_{j=1}^{N}e^{\{-\|\bmath{z}_{i}-\bmath{z}_{j}\|^{2}\}}$ &$\sum_{j}w_{ij}=1$\\
  \hline
  \textcolor[rgb]{0.00,0.00,1.00}{tSNE} & $\sum_{i,j=1}^{N}w_{ij}\log(1 + \|\bmath{z}_{i}-\bmath{z}_{j}\|^{2})$ & $\sum_{i=1}^{N}\log(\sum_{j=1}^{N}\frac{1}{1 + \|\bmath{z}_{i}-\bmath{z}_{j}\|^{2}})$ & $\sum_{j}w_{ij}=1$\\
  \hline
  \textcolor[rgb]{0.00,0.00,1.00}{SSNE} & $\sum_{i,j=1}^{N}w_{ij}\log(\|\bmath{z}_{i}-\rho\bmath{z}_{j}\|^{m})$& $\sum_{i=1}^{N}\log(\sum_{j=1}^{N}\frac{1}{\|\bmath{z}_{i}-\rho\bmath{z}_{j}\|^{m}})$ & $\bmath{z}_{i}\in\mathbb{S}^{m}$\\
  \hline
  \textcolor[rgb]{0.00,0.00,1.00}{LE} & $\sum_{i,j=1}^{N}w_{ij}\|\bmath{z}_{i}-\bmath{z}_{j}\|^{2}$ & $\bmath{Z}\bmath{Z}^{T}=I$, $\bmath{Z}\bmath{1}=\bmath{0}$& \\
  \hline
  \textcolor[rgb]{0.00,0.00,1.00}{CE} & $\sum_{i,j=1}^{N}\frac{w_{ij}}{\sigma^{2} + \|\bmath{z}_{i}-\bmath{z}_{j}\|^{2}\}}$ & $\bmath{Z}\bmath{Z}^{T}=I$, $\bmath{Z}\bmath{1}=\bmath{0}$& \\
  \hline
  \textcolor[rgb]{0.00,0.00,1.00}{GE} & $\sum_{i,j=1}^{N}w_{ij}\exp\{-\|\bmath{z}_{i}-\bmath{z}_{j}\|^{2}\}$ & $\bmath{Z}\bmath{Z}^{T}=I$, $\bmath{Z}\bmath{1}=\bmath{0}$& \\
  \hline
  \textcolor[rgb]{0.00,0.00,1.00}{MAFEBR} & $\sum_{i,j=1}^{N}w_{ij}\xi_{a}\|\bmath{z}_{i}- \bmath{z}_{j}\|^{p}$ &  $\sum_{i,j=1}^{N}\xi_{r}\sigma\exp\{-\frac{\|\bmath{z}_{i}-\bmath{z}_{j}\|^{q}}{\sigma}\}$& $p,q\ge 1$\\
  \hline
  \textcolor[rgb]{0.00,0.00,1.00}{MAFEUR} & $\sum_{i,j=1}^{N}w_{ij}\xi_{a}\|\bmath{z}_{i}- \bmath{z}_{j}\|^{p}$ & $ \sum_{i,j=1}\frac{\xi_{r}}{\|\bmath{z}_{i}- \bmath{z}_{j}\|^{q}}$ &$p,q\ge 1$\\
  \hline
  \textcolor[rgb]{0.00,0.00,1.00}{MAFEE} & $\sum_{i,j=1}^{N}w_{ij}\sigma_{a}\exp\{-\frac{\|\bmath{z}_{i}-\bmath{z}_{j}\|^{p}}{\sigma_{a}}\}$ &  $\sum_{i,j=1}^{N}\xi_{r}\sigma_{r}\exp\{-\frac{\|\bmath{z}_{i}-\bmath{z}_{j}\|^{q}}{\sigma_{r}}\}$&$p,q\ge 1$\\
 \hline
\textcolor[rgb]{0.00,0.00,1.00}{MAFEH} & $\sum_{i,j=1}^{N}\frac{w_{ij}\xi_{a}}{\|\bmath{z}_{i}- \bmath{z}_{j}\|^{p}}$ & $ \sum_{i,j=1}\frac{\xi_{r}}{\|\bmath{z}_{i}- \bmath{z}_{j}\|^{q}}$ &$p,q\ge 1$\\
\bottomrule
 \end{tabular}
\end{table}

\section{Stochastic Gradient Descent Optimization}
The objective functions for MAFE based models are by far less nonlinear as compared to SNE and tSNE, as such their optimization is relatively simple, requiring no random jitter terms to establish stability. The optimization adapts a variation of the stochastic gradient descent\cite{Plagianakos01} with a common adaptive learning rate
\begin{eqnarray}
\alpha^{(t+1)} = \alpha^{(t)} + \gamma_{1}\langle\nabla U(\mathcal{Z}^{(t-1)}),\nabla U(\mathcal{Z}^{(t)})\rangle + \gamma_{2}\langle\nabla U(\mathcal{Z}^{(t-2)}),\nabla U(\mathcal{Z}^{(t-1)})\rangle
\label{eqn:learning rate}
\end{eqnarray}
where $\alpha^{(t)}$ is the learning rate at iteration $t$, $\gamma_{1}$ and $\gamma_{2}$ are the meta-learning rates. The main characteristics of this fast learning rate adaptation scheme is that it exploits gradient-related information from the current as well as the two previous embedding coordinates in the sequence. This provides an enhancement on the stabilization in the values of the learning rate and helps the gradient descent algorithm to exhibit even fast convergence that leads  to better minimum energy-configuration. A description of the proposed algorithm is given in Algorithm \ref{pfe-main}. The algorithm's termination condition is when $\|\nabla U(\mathcal{Z})\|\leq\epsilon$. The choice of $\gamma_{1}$ and $\gamma_{2}$ is not critical for finding the minimum-energy configuration but only affect the rate at which the algorithm does so. Figure \ref{fig:gradient-trajectories} shows gradient field trajectories that were generated by this optimization scheme for a real world data set.
\begin{algorithm}
\caption{Generalized MAFE Algorithm}
\KwIn{{\small High dimensional observations $\bmath{Y}$} \;}
{\small Initialize optimization parameters: $\alpha^{(1)},\ \gamma_{1},\ \gamma_{2}$\;}
\KwOut{ {\small Embedding coordinates matrix $\bmath{Z}=\{\bmath{z}_{1},\bmath{z}_{2},\cdots,\bmath{z}_{N}\}$}\;}
{\small Compute high dimensional neighborhood graph weights $w_{ij}$ from equation \eqref{eqn:bisim-final} or use a Gaussian kernel\;}
{\small Randomly sample initial maps from a normal distribution {\em i.e.} $\bmath{Z}^{(0)} \sim \ N(0,50I)$\;}
{\small Set $\mathcal{Z}^{(1)}=[\bmath{z}_{1}^{(0)T},\bmath{z}_{2}^{(0)T},\cdots,\bmath{z}_{N}^{(0)T}]^{T}\in\mathbb{R}^{Nm}$ \;}
\While{$\|\nabla U(\mathcal{Z}^{(t)})\|>\epsilon$}{
{\mbox{Set} $t = t + 1$\;}
  {\small Compute new embedding coordinates using\;}
 {$\mathcal{Z}^{(t+1)} = \mathcal{Z}^{(t)} - \alpha^{(t)}\nabla U(\mathcal{Z}^{t})$\;}
 {\small Calculate the new learning rate from}
 {$\alpha^{(t+1)} = \alpha^{(t)} + \gamma_{1}\langle\nabla U(\mathcal{Z}^{(t-1)}),\nabla U(\mathcal{Z}^{(t)})\rangle + \gamma_{2}\langle\nabla U(\mathcal{Z}^{(t-2)}),\nabla U(\mathcal{Z}^{(t-1)})\rangle$\;}
}
\label{pfe-main}
\end{algorithm}

The iterative optimization in \eqref{eqn:generalopt}, or in particular the models described by equations \eqref{pot-embedding-gaussian-repulsion}, and \eqref{pot-embedding-unbounded-repulsion} converges when all pair-equilibrium distances $\epsilon_{ij}$ are established. A strong theoretical argument can be made to assert that the motion of maps is guaranteed to stop and that no oscillatory behavior exist at the minimum-energy configuration state. This is done by letting the invariant set of the equilibrium positions to be $$ \Xi_{equi} = \left\{ \mathcal{Z}:\dot{\mathcal{Z}}=0\right\}.$$  The proof needs to show that as $t\rightarrow\infty$ the state $\mathcal{Z}(t)$ converges to $\Xi_{equi}$, {\em i.e.} the minimum-energy configuration of a graph converges to a stationery state. This extends Theorem 2 of \cite{Gazi02c} to problems of data visualization and dimensionality reduction.

\begin{mytheorem}\label{Thm:Equilibrium}
Consider a graph embedding described by $\dot{\bmath{z}}_{i}=\sum_{j=1, j\ne i}F^{ij}(\bmath{z}_{i}-\bmath{z}_{j}),\ \ i=1,\cdots, N$, with pairwise force field functions\\ $F^{ij}(\bmath{z}_{i}-\bmath{z}_{j})=(\bmath{z}_{i} - \bmath{z}_{j})\left\{F^{ij}_{r}(\|\bmath{z}_{i} - \bmath{z}_{j}\|)- F^{ij}_{a}(\|\bmath{z}_{i} - \bmath{z}_{j}\|)\right\}$. As $t\rightarrow\infty$, it can be shown that $\mathcal{Z}(t)\rightarrow\Xi_{equi}.$
\end{mytheorem}
\begin{proof}
Consider the general energy function
$ U(\mathcal{Z}) = \sum_{i=1}^{N}U_{i}(\mathcal{Z})$, where $U_{i}(\mathcal{Z})$ is defined in \eqref{vertex-pot-field}. Taking the derivative of $U(\mathcal{Z})$ with respect to each $\bmath{z}_{i}$ yields,
\begin{eqnarray}
\nabla_{\bmath{z}_{i}} U(\mathcal{Z}) = 2\sum_{j=1, j\ne i}\left\{\nabla_{\bmath{z}_{i}}w_{ij}U_{att}^{ij}(\|\bmath{z}_{i}-\bmath{z}_{j}\|) -\nabla_{\bmath{z}_{i}}U^{ij}_{rep}(\|\bmath{z}_{i}-\bmath{z}_{j}\|)\right\}= - \dot{\bmath{z}}_{i}\nonumber
\end{eqnarray}
where the negative gradient is observed as direction of motion in the second equality.
Taking the time derivative of $U(\mathcal{Z})$ along the motion of a graph configuration yields,
\begin{eqnarray}
\dot{U}(\mathcal{Z}) = \nabla_{\bmath{z}_{i}} U(\mathcal{Z})^{T}\dot{\mathcal{Z}} = 2\sum_{i=1}^{N}\nabla_{\bmath{z}_{i}} U(\mathcal{Z})^{T}\dot{\bmath{z}}_{i}=2\sum_{i=1}^{N}\{- \dot{\bmath{z}}_{i}\}^{T}\dot{\bmath{z}}_{i}
=-2\sum_{i=1}^{N}\|\dot{\bmath{z}}_{i}\|^{2}\leq 0,\ \ \forall t.\nonumber
\end{eqnarray}
This result shows that the motion will continue in the direction of decreasing $U(\mathcal{Z})$  to a state when all $\dot{\bmath{z}}_{i} = 0$. By invoking the Lasalle Invariance Principle\cite{LaSalle60}, it can be concluded that as $t\rightarrow\infty$ the graph configuration state $\mathcal{Z}(t)$ converges to the largest subset of the set defined as $$ \Xi =\left\{\mathcal{Z}:\dot{U}(\mathcal{Z})=0\right\}=\left\{\mathcal{Z}:\dot{\bmath{z}}_{i}=0\right\}=\Xi_{equi}.$$ Since each $\dot{\bmath{z}}_{i}\in\Xi_{equi}$ is an equilibrium point, $\Xi_{equi}$ is an invariant set and this concludes the proof.
\end{proof}\vspace{-0.2cm}
This general result holds for any function $F^{ij}$ chosen based on the embedding force field properties discussed in Section \ref{sec:mafegeneral}. As such, it guarantees the convergence of our algorithms. However, in practise, the termination condition is set to some $\epsilon$ finite optimization of $U(\mathcal{Z})$. The result also extends to SNE, tSNE, sSNE and other related MAFE reformulations.%

In addition to obtaining optimal embedding coordinates, the optimization of MAFE-BR and MAFE-UR with $p=q=2$, also learns the embedding space neighborhood graph weights. For MAFE-BR, the optimal embedding neighborhood graph is described by
\begin{eqnarray}
\tilde{w}_{ij}^{br} = \xi_{a}w_{ij} - \xi_{r}\exp\{-\frac{\|\bmath{z}_{i}-\bmath{z}_{j}\|^{2}}{\sigma}\}\nonumber
\end{eqnarray}
The iterative optimization of MAFE-UR generates embedding graph weights that are given by
\begin{eqnarray}
\tilde{w}_{ij}^{ur} = \xi_{a}w_{ij} - \frac{\xi_{r}}{\|\bmath{z}_{i}-\bmath{z}_{j}\|^{4}}\nonumber
\end{eqnarray}
In contrast to the high dimensional neighborhood graph weights , $w_{ij}$, that are positive, the lower dimensional graph weights $\tilde{w}_{ij}^{br}$, and $\tilde{w}_{ij}^{ub}$, corresponding to the optimal embedding coordinates can be negative as illustrated in Figure \ref{fig:embeddingweights}. This property establishes another point of significant contrast between the proposed MAFE based techniques in comparison to traditional embedding techniques that are known to enforce learning of positive lower dimensional weights {\em e.g.}   LLE, SNE, tSNE, and Isomap.
\begin{figure}
\centering
\subfigure[]{
\includegraphics[width=4.2in]{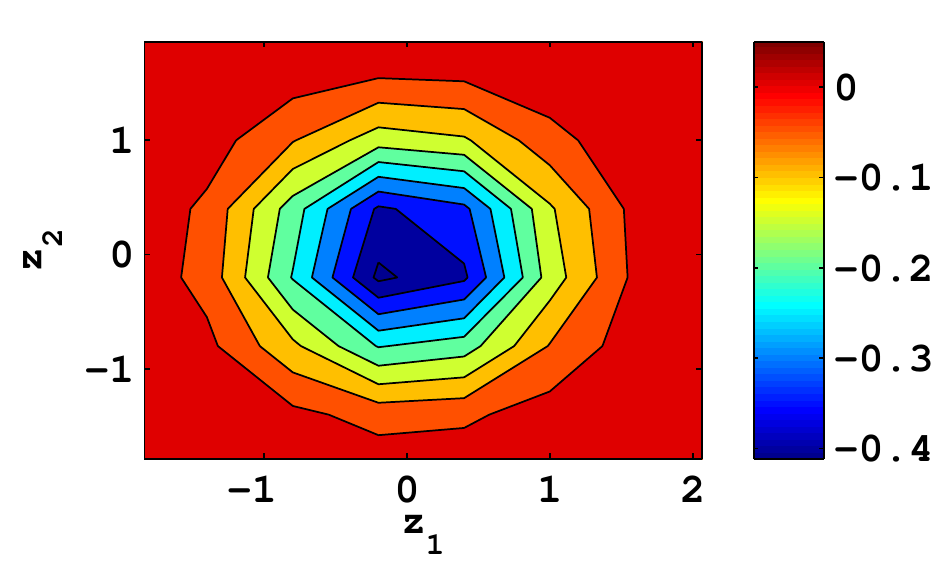}
   \label{fig:bots-rgb}
 }
\subfigure[]{
\hspace{-0.2cm}\includegraphics[width=4.2in]{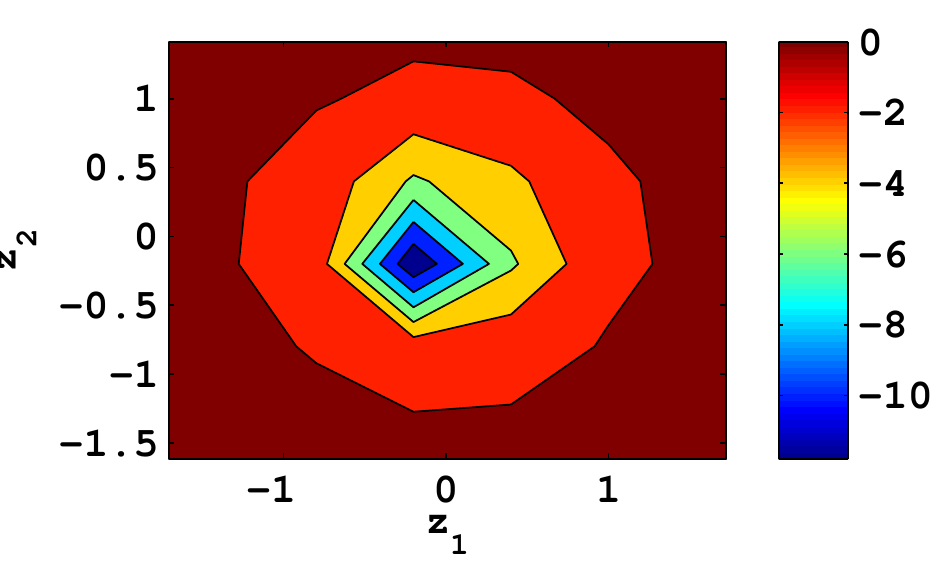}
   \label{fig:bots-labels}
 }
 \caption{Minimum energy graph configuration weights that are learned after 100 iterations for a KSC {\em Graminoid Marsh} pixel map with coordinates $\bmath{z}=[-0.02,-0.003]$. MAFE-BR weights, $\tilde{w}_{ij}^{br}$, are shown in (a) and MAFE-UR weights, $\tilde{w}_{ij}^{ur}$, are shown in (b). For both models, the search for optimal embedding coordinates also leads to learning of negative graph weights. }
\label{fig:embeddingweights}
\end{figure}

\section{Hyperspectral Image Scenes}

\subsubsection{Botswana Hyperion}
Hyperion data with nine identified classes of complex natural vegetation were acquired over the Okavango Delta, Botswana, in May 2001, \cite{Neuenschwander07}. The general class groupings include seasonal swamps, occasional swamps, and woodlands. Signatures of several classes are spectrally overlapped, typically resulting in poor classification accuracies. After removing water absorption, noisy, and overlapping spectral bands, 145 bands were used for classification experiments. Classification error rates on signatures are carried out after data has been embedded to a lower dimensional space.

\subsubsection{Kennedy Space Center (KSC)}
Airborne hyperspectral data were acquired by the National Aeronautics and Space Administration(NASA) Airborne Visible/Infrared Imaging Spectrometer (AVIRIS) sensor at 18-m spatial resolution over Kennedy Space Center during March 1996. Noisy and water absorption bands were removed, leaving 176 features for thirteen wetland and upland classes of interest. Cabbage Palm Hammock (Class 3) and Broad Leaf/Oak Hammock (Class 6) are upland trees; Willow Swamp (Class 2), Hardwood Swamp (Class 7), Graminoid Marsh (Class 8) and Spartina Marsh (Class 9) are trees and grasses in wetlands. Their spectral signatures are mixed and often exhibit only subtle differences. Visualization and classification results for all thirteen classes are reported using the lower dimensional space coordinates.

\subsubsection{Indian Pine}
This scene was gathered by an AVIRIS sensor over the Indian Pines test site in North-western Indiana in June 12, 1992 at 20-m spatial resolution. The 220-band image scene is over an area that is 6 miles west of West Lafayette. The image has 16 classes consisting of agricultural fields that are alfalfa, corn-notill, cornmin, corn, grass/pasture, grass/trees, grass/pasture-mowed, haywindrowed, oats, soybeans-notill, soybean-min, soybeanclean, wheat, woods, dldg-grass-tree-drives, and stone-steel towers. The image size is $145\times 145$ pixels. The pixel resolution is 16 bits, corresponding to 65536 gray levels. 2550 pixels were selected to generate the sample training and testing data. Again, we report on Euclidean and non-Euclidean embedding results as well as evaluation of the computed coordinates based on classification error rates for all 16 classes.

\subsubsection{Salinas-A}
This is a small sub-scene of Salinas image, denoted Salinas-A, that is commonly used in hyperspectral data analysis. It comprises $86\times 83$ pixels located within the same scene at [samples, lines] = [591-676, 158-240] and includes six classes. The original scene was collected by the 224-band AVIRIS sensor over Salinas Valley, California, and is characterized by high spatial resolution (3.7-meter pixels). The area covered comprises 512 lines by 217 samples. The 20 water absorption bands were discarded, \ie bands: [108-112], [154-167], 224.  It includes vegetables, bare soils, and vineyard fields. Salinas ground truth contains 16 classes.

\begin{table}
	\centering
	\small
	\caption{Ground truth classes for the Salinas-A and Indian Pine scenes with their respective samples numbers}
	\label{tab:Datasets}
		\begin{tabular}{llllllll} \toprule
    \multicolumn{2}{c}{Salinas-A} && \multicolumn{2}{c}{Indian Pine}\\
		 \toprule
		c1&  Brocoli-green-weeds-1 (391) && c1& Alfalfa (54)\\
		 c2& Corn-senesced-green-weeds	 (1343)&&c2&Corn-notill (1434) \\
		c3& Lettuce-romaine-4wk	 (616) &&c3& Corn-min (834)\\
		c4& Lettuce-romaine-5wk	 (1525)&& c4&Corn (234)\\
		c5& Lettuce-romaine-6wk	 (674)&& c5&Grass-pasture (497)\\
		c6& Lettuce-romaine-7wk	 (799) && c6&Grass-trees (747)\\
		& && c7&Grass-mowed (26)\\
        & && c8&Hay-windrowed (489)\\
		& && c9&Oats (20)\\
		& && c10&Soybean-notill (968)\\
        & && c11&Soybean-min (2468)\\
        & && c12&Soybean-heavy till (614)\\
         & && c13&Wheat (212)\\
         & && c14 & Woods (1294)\\
          & && c15&Bldg-grass-tree-dr (380)\\
         & && c16&Stone-steel towers (95)\\
         \bottomrule
		\end{tabular}
\end{table}

\begin{table}
	\centering
	\small
	\caption{Ground truth classes for image scenes with their respective samples numbers}
	\label{tab:Datasets}
		\begin{tabular}{lllllll} \toprule
    \multicolumn{2}{c}{Botswana Hyperion} && \multicolumn{2}{c}{Kennedy Space Center}\\
		 \toprule
		c1&  Water (158)&&c1& Scrub (761) \\
		 c2& Floodplain (228)&&c2& Willow swamp (243) \\
		c3& Riparian (237)&&c3& Cabbage hamm (256) \\
		c4& Firescar (178)&&c4& Cabbage palm(252)\\
		c5& Island interior (183)&&c5& Slash pine (161) \\
		c6& Woodlands (199)&&c6& Oak (229)\\
		c7& Savanna (162)&&c7& Hardwood swamp (105)\\
		c8& Short mopane (124)&& c8& Graminoid marsh (431)\\
		c9&Exposed soils (111)&&c9& Spartina marsh (520)\\
		 & &&c10& Cattail marsh (404)\\
        & &&c11& Salt marsh (419)\\
        & &&c12& Mud flats (503)\\
         & &&c13& Water (927)\\
         \bottomrule
		\end{tabular}
\end{table}

\section{Experimental Results}\label{sec:experiments}
MAFE evaluation on lower dimensional representations consists of various comparisons with popular techniques on all four image scenes. The comparison includes experimental results obtained with SNE, tSNE, Isomap and LE, taking as input the Gaussian kernel neighborhood graphs obtained from a principal component analysis step that reduces the original dimension to $40$. In constructing the high dimensional neighborhood graph, we vary the number of neighbors from $k=1$ to $k=50$ and pick an optimal value of $k=15$. The optimal value for $k$ ensures that the embeddings are neither too noisy and unstable nor does the geometry of the observation collapse the coordinates of dissimilar observations. All existing models are implemented as in \cite{Laurens08}, with the perplexity of the conditional distribution induced by the Gaussian kernel determined as $2^{H}$, where $H$ is the entropy. The parameterized bilateral kernel function presented in Section \ref{sec:similarity}, is used to generate the input graphs for MAFE with no principal component analysis step. MAFE-BR results are generated with a setting  $p=q=2$ for \eqref{total-gaussian-potential}, establishing a quadratic attraction potential and a spherical Gaussian repulsion potential. MAFE-BR's magnitude parameters are set as $\xi_{a}=0.4$ and $\xi_{r}=10^{-4}$. MAFE-UR results are generated with a setting  $p=2, q=1$ and parameter values $\xi_{a}=0.03$ and $\xi_{r}=10^{-5}$. The proposed MAFE optimization algorithm uses the norm of the gradient as the termination condition, \ie $\|\nabla U(\mathcal{Z}^{(t)})\| < \epsilon$, with $\epsilon=10^{-5}$. SNE, tSNE and sSNE obtain solutions based on standard gradient descent algorithms whose terminations are defined by the number of iterations $T$, \ie terminate when $T=1000$ iterations. Embedding maps obtained by Isomap are based on the classical formulation that admits the closed-form solution of an eigenvector structured problem, namely picking the leading components of variation, while LE is based on picking the trailing eigenvectors.

\subsection{Trajectories of Gradient Vector Field }
The dynamic equation in \eqref{eqn:latentmotion} serves as an approximation to the true model that describes the central path in search for
optimal data manifolds. Additional insights on the uncovering of the underlying manifolds can be obtained by observing the gradient field trajectories as each map traverses towards the minimum energy configuration state of the embedding graph. The trajectory traces in 2D space provide a visual assessment of how the formation of clusters is affected by different cost functions and their corresponding optimization schemes. The gradient vector fields to be presented were obtained from mapping fifteen hyperspectral pixels from three classes of both the Botswana and the Kennedy Space Center images. Three classes are considered for each image, \ie {\em Woodlands, Firescar , \em Island Interior} for Botswana and {\em Water, Graminoid Marsh, Cabbage Palm} for Kennedy Space Center. MAFE based trajectories are compared to gradient fields obtained with tSNE and SNE. Figure \ref{fig:gradient-bots-trajectories} and Figure \ref{fig:gradient-trajectories} illustrate the formation of clusters under these iterative optimized embedding algorithms. tSNE and SNE results are shown in Figures \ref{fig:sne-trajectory} and \ref{fig:tsne-trajectory} from which a high degree of oscillation and random change in gradient vector directions can be observed. These instabilities causes inefficiencies on the establishment of the equilibrium distances $\epsilon_{ij}$ between pairs of maps. In contrast, starting from the same initial maps, Figures \ref{fig:mafe-br-trajectory} and \ref{fig:mafe-ur-trajectory} demonstrates MAFE-BR  and MAFE-UR techniques with smooth trajectories and no instabilities, i.e. no random change of gradient vector directions. Such smooth gradient fields are due to two important features of the proposed MAFE framework. First, the local stochastic adaptive optimization scheme used in MAFE models retains the local gradient information from its last two sequence computations thereby establishing smooth and stable iterates, it does not depend on random jitter parameters in its update rule. Second, the careful design of pair-dependent attraction and repulsion functions, as well as the sparsity of neighborhood graphs allow for a smooth interaction between long range and short range forces on pairs of maps. The smooth interaction generates a robust motion of clusters that lead to the establishment of the minimum energy configuration state much quicker. As such, the optimization scheme used in MAFE models achieves several magnitudes of convergence speed in contrast to other algorithms as shown in Figure \ref{fig:costfunctiondescrease1}.

\begin{figure}
\centering
\subfigure[Initial Maps.]{
\includegraphics[width=3in]{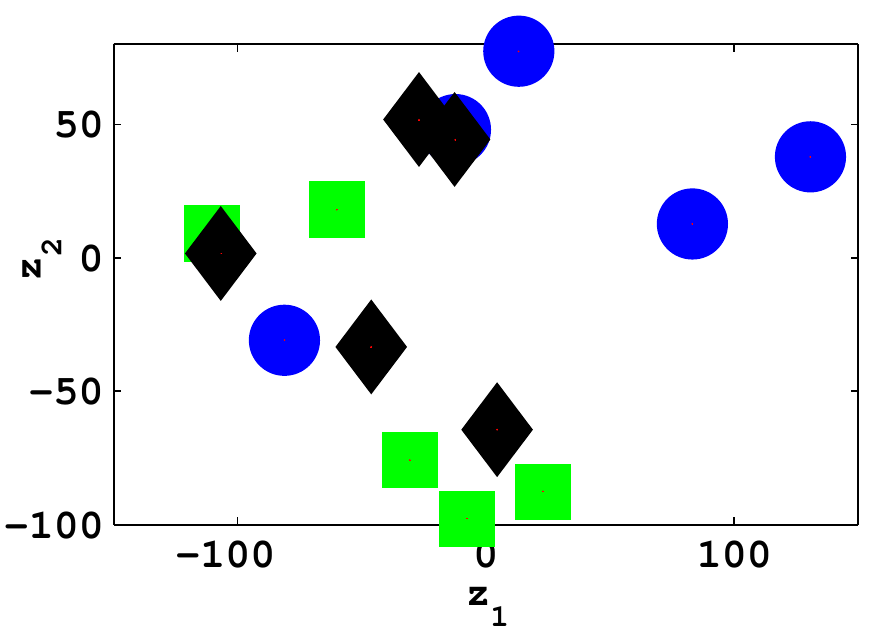}
\vspace{-0.2cm}
   \label{fig:initial-bots}
 }
\subfigure[{\tt MAFE-BR}.]{
\includegraphics[width=3in]{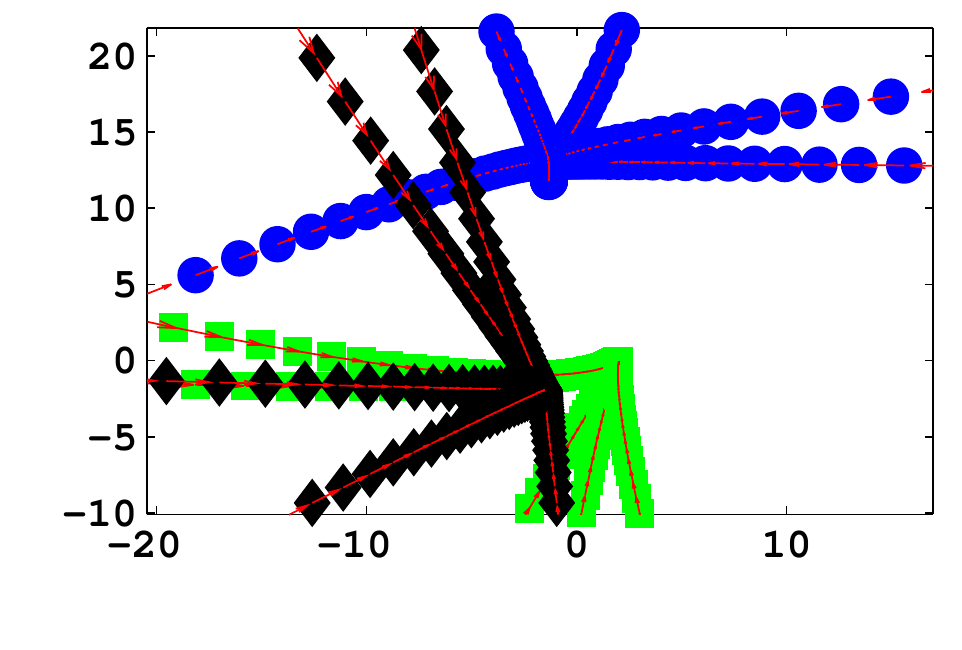}
\vspace{-0.2cm}
 \label{fig:mafe-br-zoom}
 }\\
 
\subfigure[{\tt SNE}.]{
\includegraphics[width=3in]{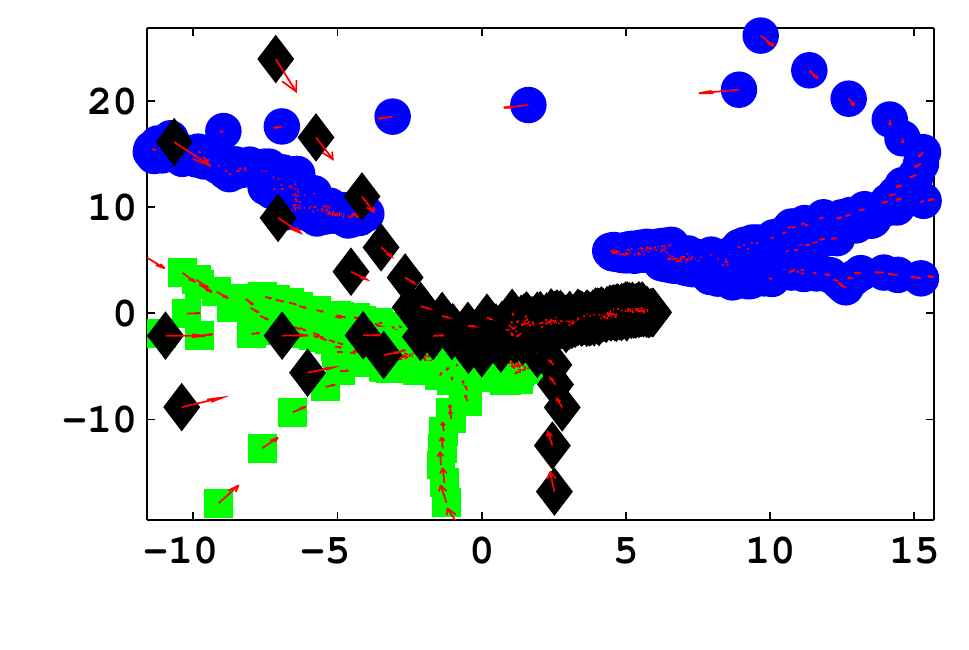}
\vspace{-0.2cm}
 \label{fig:sne-zoom}
 }
\subfigure[{\tt tSNE}.]{
\includegraphics[width=3in]{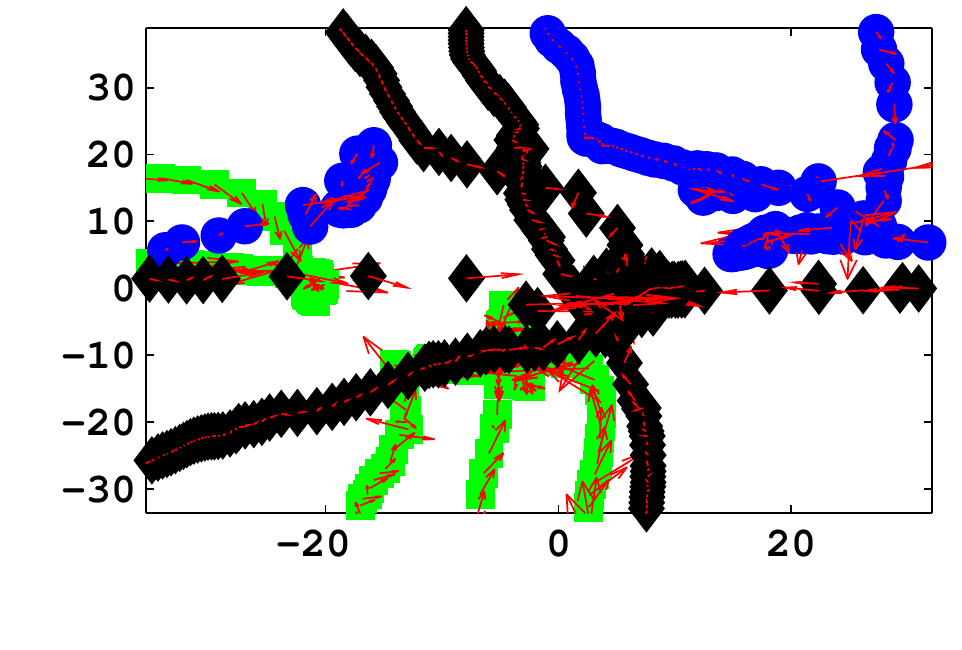}
\vspace{-0.2cm}
 \label{fig:tsne-zoom}
 }
 \vspace{-0.2cm}
\caption{Gradient based optimization illustrations over $100$ iterations. {\tt MAFE-BR}(b), {\tt SNE}(c) , {\tt tSNE}(d), are all initialized from the same seed of $15$ points for Botswana maps of $3$ different classes ({\em Woodlands, Firescar } and {\em Island Interior}). {\tt MAFE-BR} displays smooth gradient trajectories. Similar points cluster and traverse in the direction of the gradient field as indicated by the arrows. {\tt SNE} and {\tt tSNE} trajectories are subject to oscillations, including collisions of maps leading to instabilities as well as the severe local maxima traps that slow down the optimization algorithm. Arrows are shown pointing in the negative direction of the gradient.}
\label{fig:gradient-bots-trajectories}
\end{figure}

\begin{figure}
\centering
\subfigure[Initial Maps.]{
\includegraphics[width=3in]{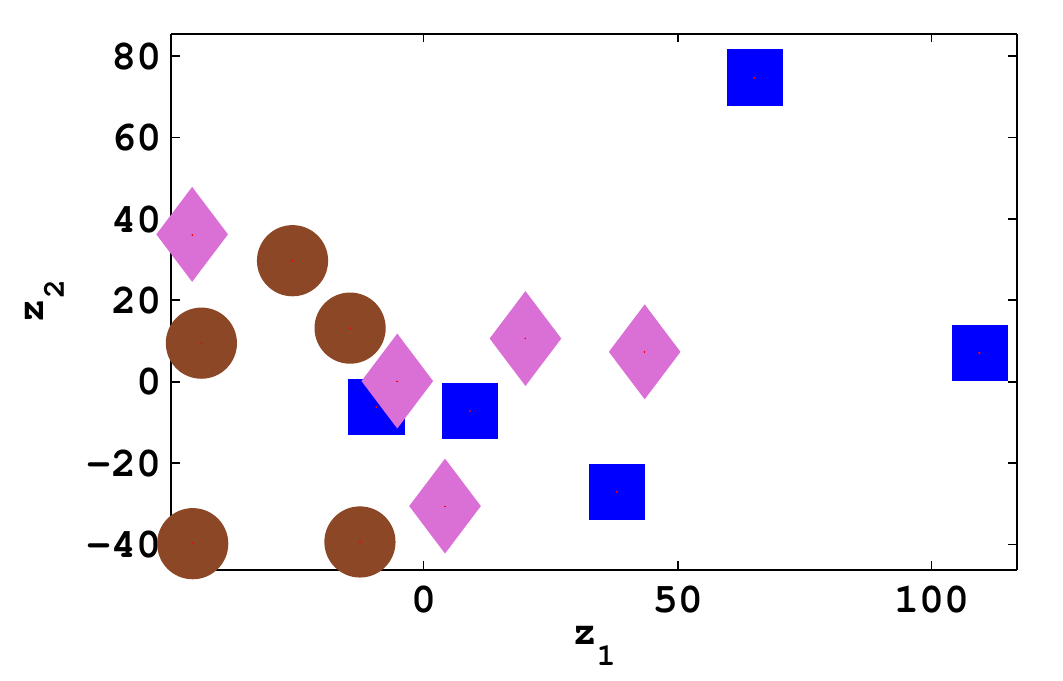}
\vspace{-0.2cm}
  \label{fig:initial-ksc}
}\\
 \subfigure[MAFE-BR.]{
\hspace{-0.2cm}\includegraphics[width=3in]{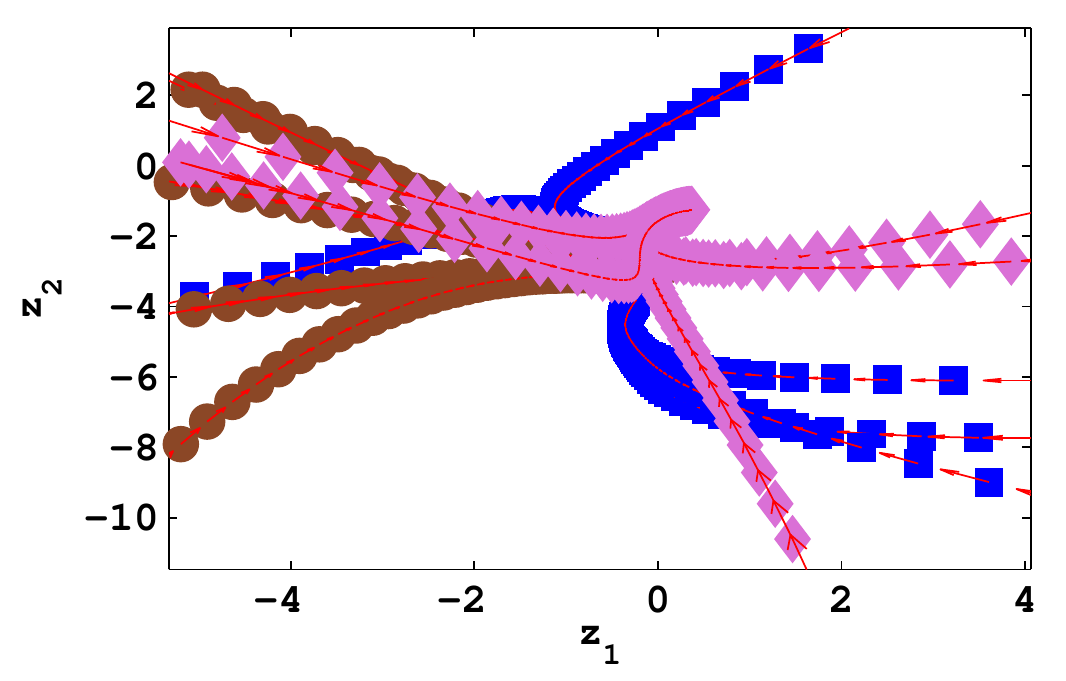}
\vspace{-0.2cm}
\label{fig:mafe-br-trajectory}
}
\subfigure[MAFE-UR.]{
\hspace{-0.4cm}\includegraphics[width=3in]{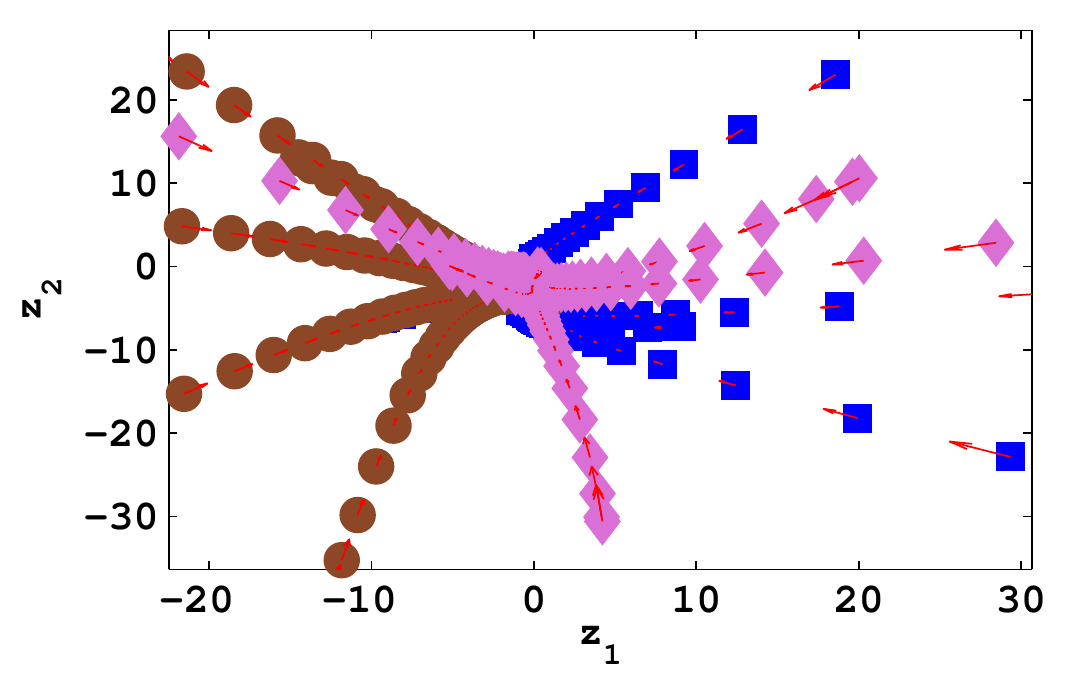}
\vspace{-0.2cm}
 \label{fig:mafe-ur-trajectory}
}\\
\subfigure[SNE.]{
\hspace{-0.2cm}\includegraphics[width=3in]{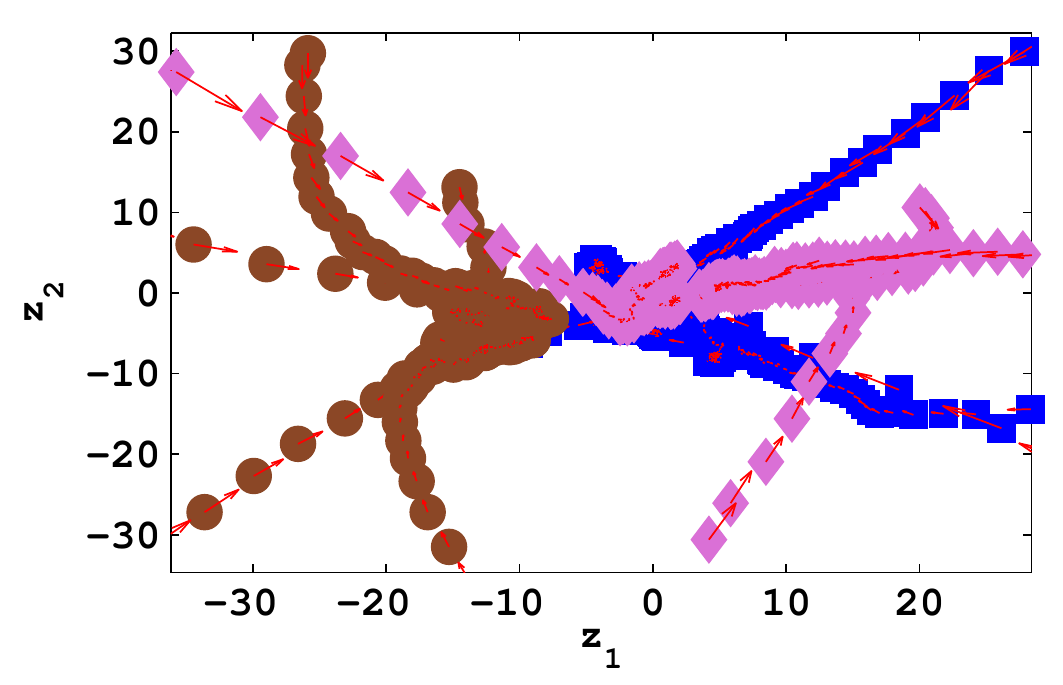}
\vspace{-0.2cm}
\label{fig:sne-trajectory}
}
\subfigure[tSNE.]{
\hspace{-0.4cm}\includegraphics[width=3in]{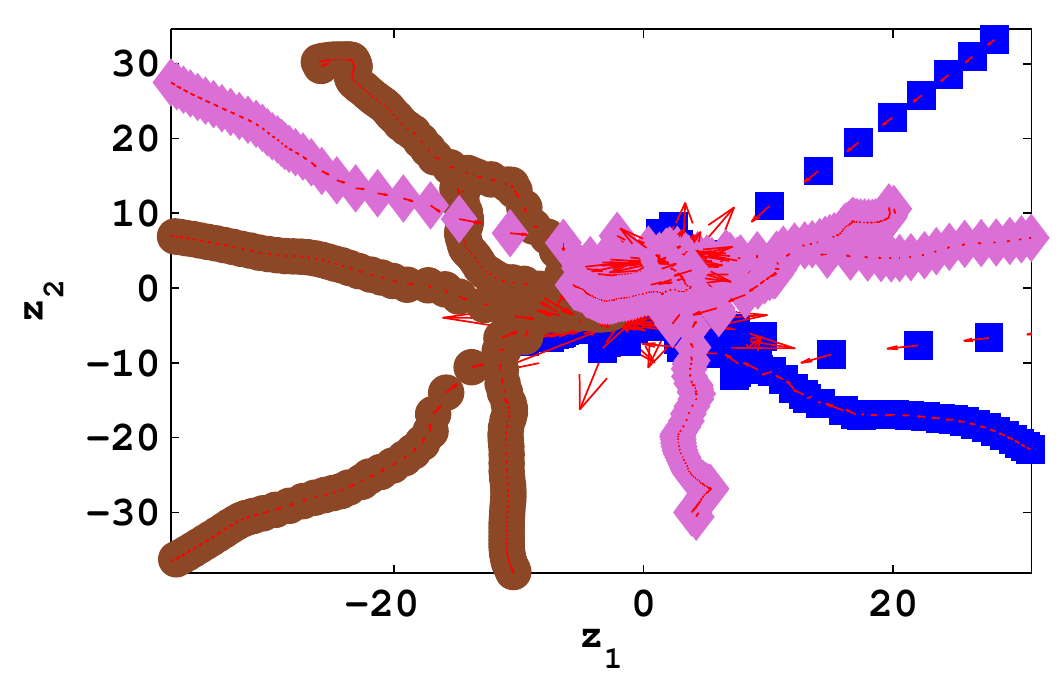}
\vspace{-0.2cm}
 \label{fig:tsne-trajectory}
 }
\caption{Gradient based optimization illustrations over $100$ iterations. All techniques were initialized from the same points to embed fifteen KSC pixel maps consisting of three different classes {\em i.e.} {\em Water, Graminoid Marsh } and {\em Cabbage Palm}. Both MAFE-BR and MAFE-UR demonstrate smooth gradient trajectories with points belonging to the same class clustering together and traversing in the negative direction of the gradient field ( indicated by the arrows). SNE and tSNE trajectories are subject to oscillations and instabilities leading to poor convergence of the algorithms.}
\label{fig:gradient-trajectories}
\end{figure}

\subsection{Visualization of Embeddings}\label{sec:visualization}
Nonlinear dimensionality reduction methods do tend to demonstrate better performance when embedding data sets with smaller number of classes than for problems with many classes \cite{Crawford2011}. This is caused by the complexity of data manifolds represented in problems with many disparate classes.  For example, in hyperspectral imagery; similar classes of data may result in multiple manifolds due to their spatial locations, as such that presents a challenge to many existing dimensionality reduction techniques with a tendency of mapping all similar or related data onto a single cluster. Such characteristics increases the chances of collapsing very different classes to the same embedding representation. Using MAFE-BR and MAFE-UR models,  this study makes an effort to address this challenge and provides high quality lower dimensional visualization results for complex data sets.

Figure \ref{fig:bots-groundtruth} shows the Botswana RGB image and its ground references data. The embedding results obtained for MAFE-BR, SNE, tSNE, MDS, Isomap, and LE are shown in Figure \ref{fig:bots-embed}. MAFE-BR results demonstrate different and superior embedding map compared to other methods. For example, SNE computes coordinates that seem to separate different classes well however there is a significant overlap on the {\em Riparian} and {\em Woodlands} classes. The clusters seem to be more spread implying large variances in the embedding space. MAFE-BR combined with the bilateral kernel function achieves tight spatial disjoint clusters, demonstrating no overlaps on the {\em Riparian} and {\em Woodlands} classes, the most known to be difficult classes to separate for this data set\cite{Crawford2011}. tSNE although with capability to mitigate overcrowding of points and good clustering effect, leads to significant overlap between {\em Riparian} and {\em Woodlands} classes. MDS and Isomap embeddings are similar due to the dependence on the classical MDS solution, both results show a significant overlapping of coordinates. LE produces a solution with very limited interpretation on the separation of classes.

Figure \ref{fig:ksc-groundtruth} shows both the RGB image and ground references of the Kennedy Space Center (KSC) data. Figure \ref{fig:ksc-embed} shows the embedding results obtained with MAFE-BR, MAFE-UR, SNE, tSNE, Isomap, and LE. The $2$-dimensional embedding results demonstrate that MAFE-BR and MAFE-UR construct very tight clusters. In addition, due to the spatially-sensitive kernel the embedding result maintains the disjoint nature observed from the land cover categories of the ground truth. Furthermore, MAFE models introduce the capability to {\em tile} different land cover classes when their boundaries appear to be overlapping. The models achieve better neighborhood properties as illustrated in Figure \ref{fig:frobeniusdistance}. SNE and tSNE  results demonstrate the existence of clusters for {\em Water, Salt Marsh,} and {\em Spartina Marsh} classes. However, there is very little insight on the separability of the other ten classes. Isomap's solution demonstrates the presence of overcrowding and class overlap. LE shades very little meaningful interpretation on the structure of all classes.

Figure \ref{fig:indian-groundtruth} and Figure \ref{fig:salinasA-groundtruth} shows ground truth image data for both the Indian Pine and the Saline-A scenes. Their corresponding $2$-dimensional MAFE based embeddings are shown in Figure \ref{fig:indianpine-embed} and Figure \ref{fig:salinaA-embed}. The results further demonstrate quality visualizations that are generated due to use of a bilateral kernel function in combination with a MAFE based embedding framework. Use of spatial information allows for the disjoint nature of classes to be encoded in the neighborhood graph, \ie it introduces the needed level of sparsity for establishing disjoint neighborhoods. The interactive fields that generates the attraction and repulsion forces in  MAFE techniques enables preserving such topological relations.

\begin{figure}
\centering
 \subfigure[]{
\hspace{-0.45cm}\includegraphics[width=2in]{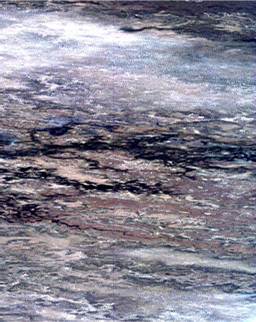}
   \label{fig:bots-rgb}
 }
 \subfigure[]{
\includegraphics[width=3.8in]{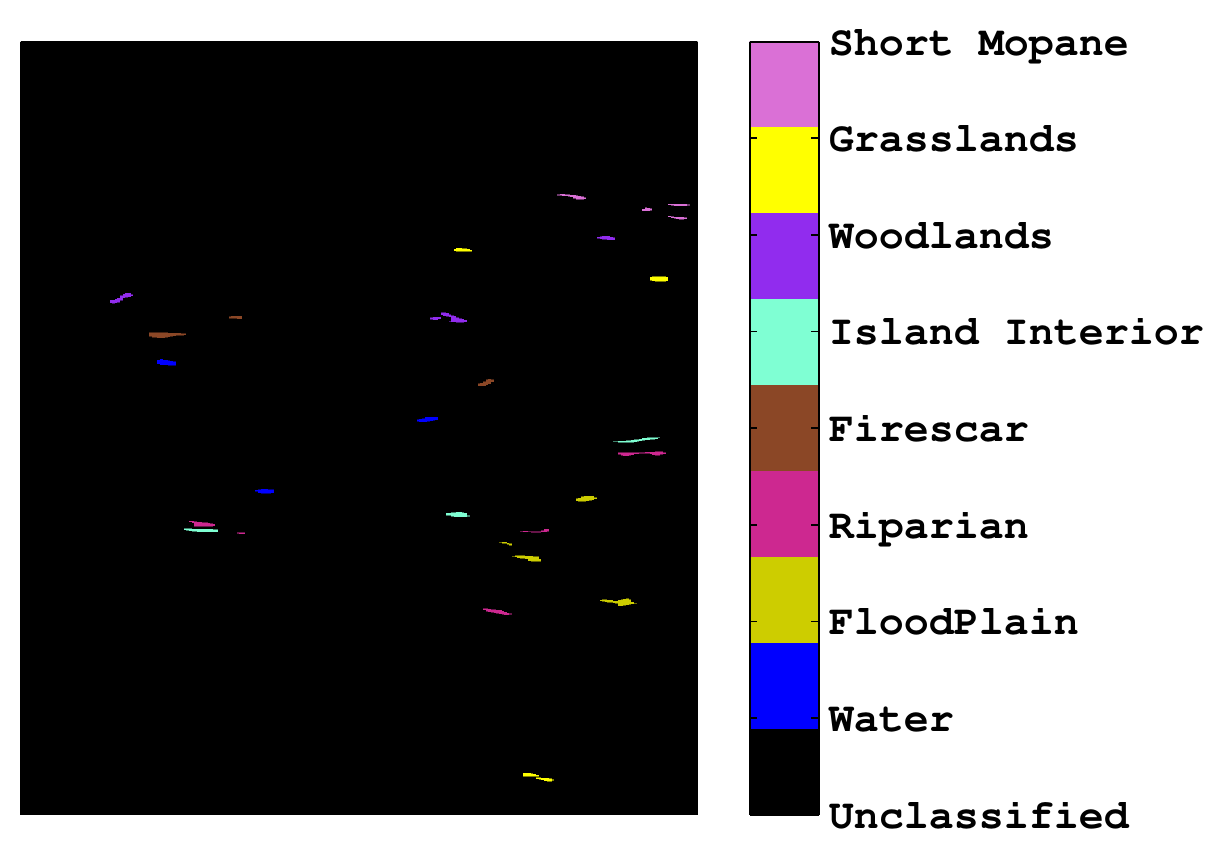}
   \label{fig:bots-labels}
 }
 \caption{ (a) RGB image of Botswana data. (b) Botswana labeled data.}
\label{fig:bots-groundtruth}
\end{figure}

\begin{figure}
\centering
\includegraphics[width=2.3in]{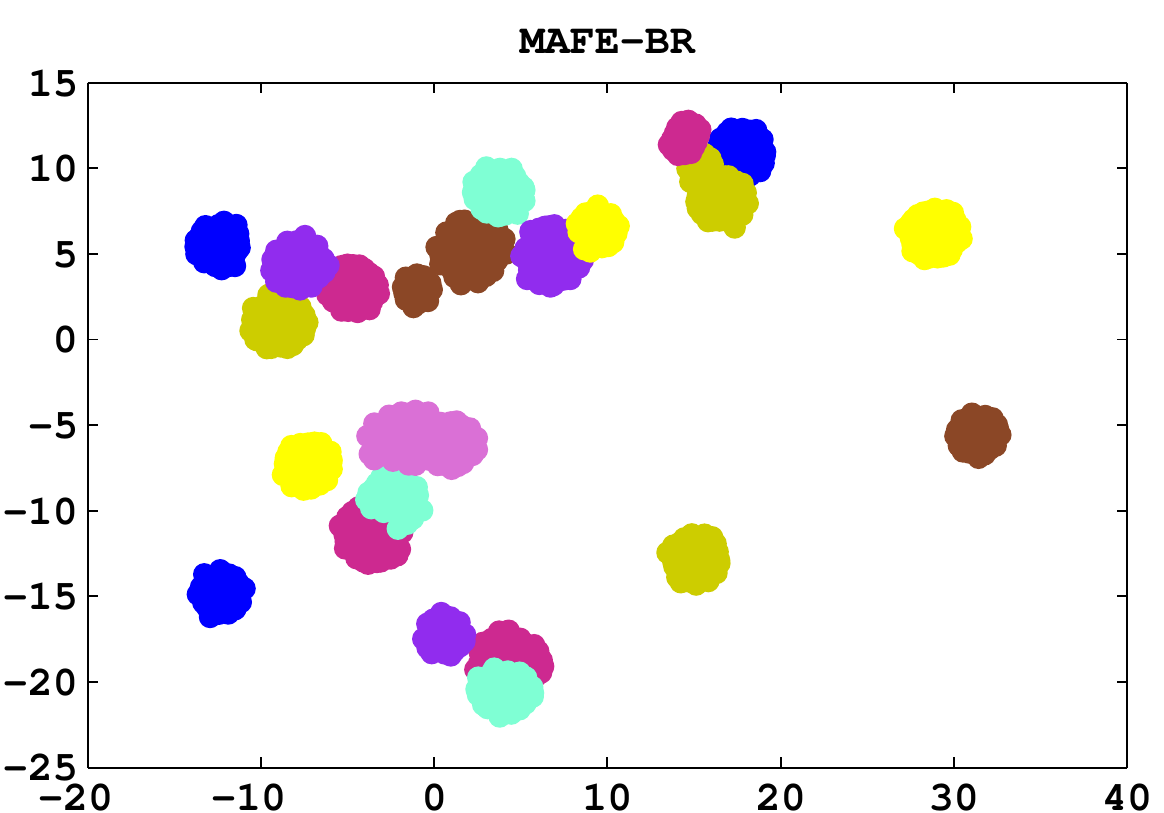}\hspace{0.3cm}\includegraphics[width=2.33in]{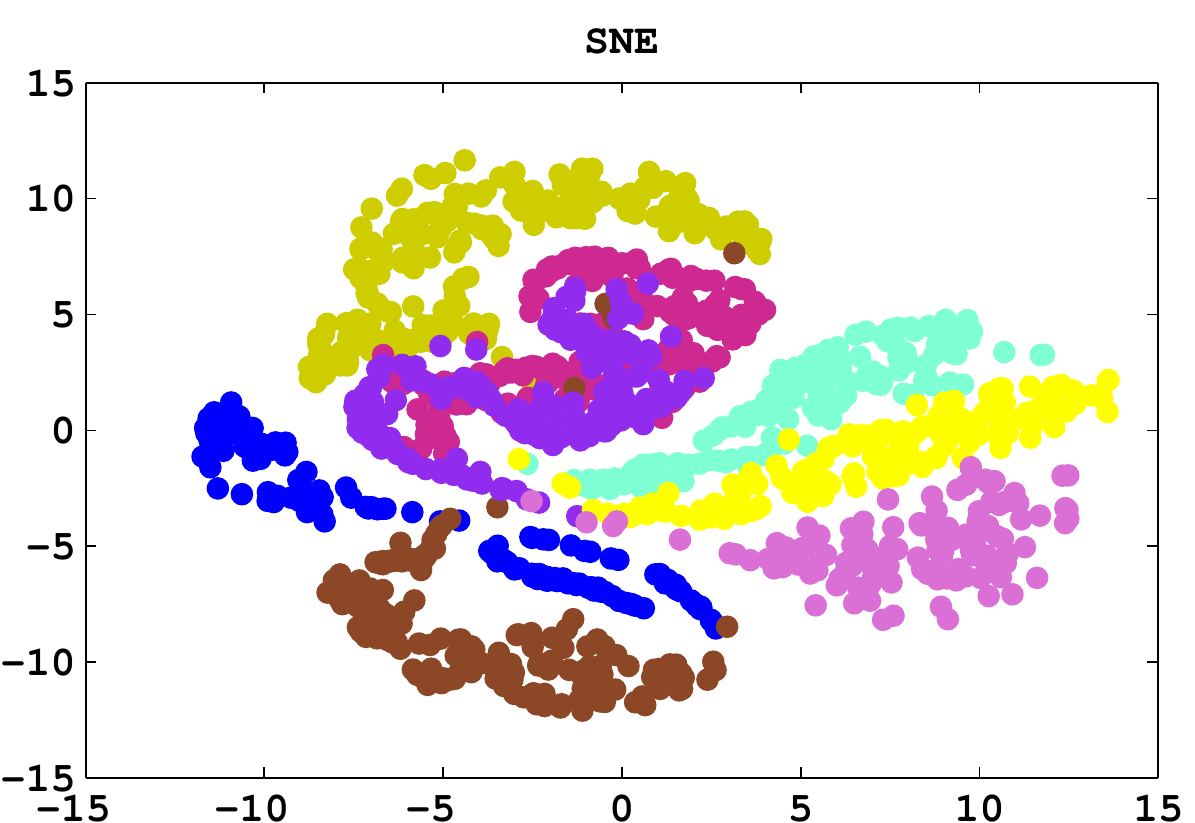}\\
\includegraphics[width=2.3in]{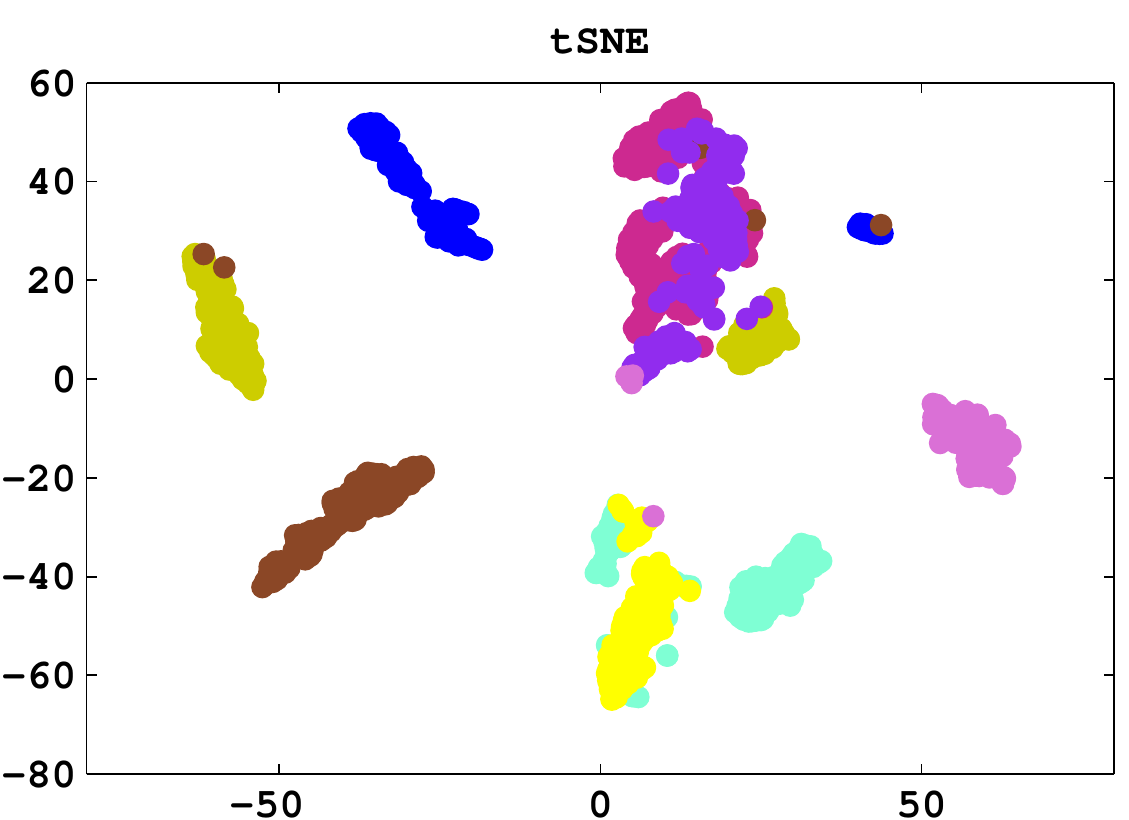}\hspace{0.3cm}\includegraphics[width=2.4in]{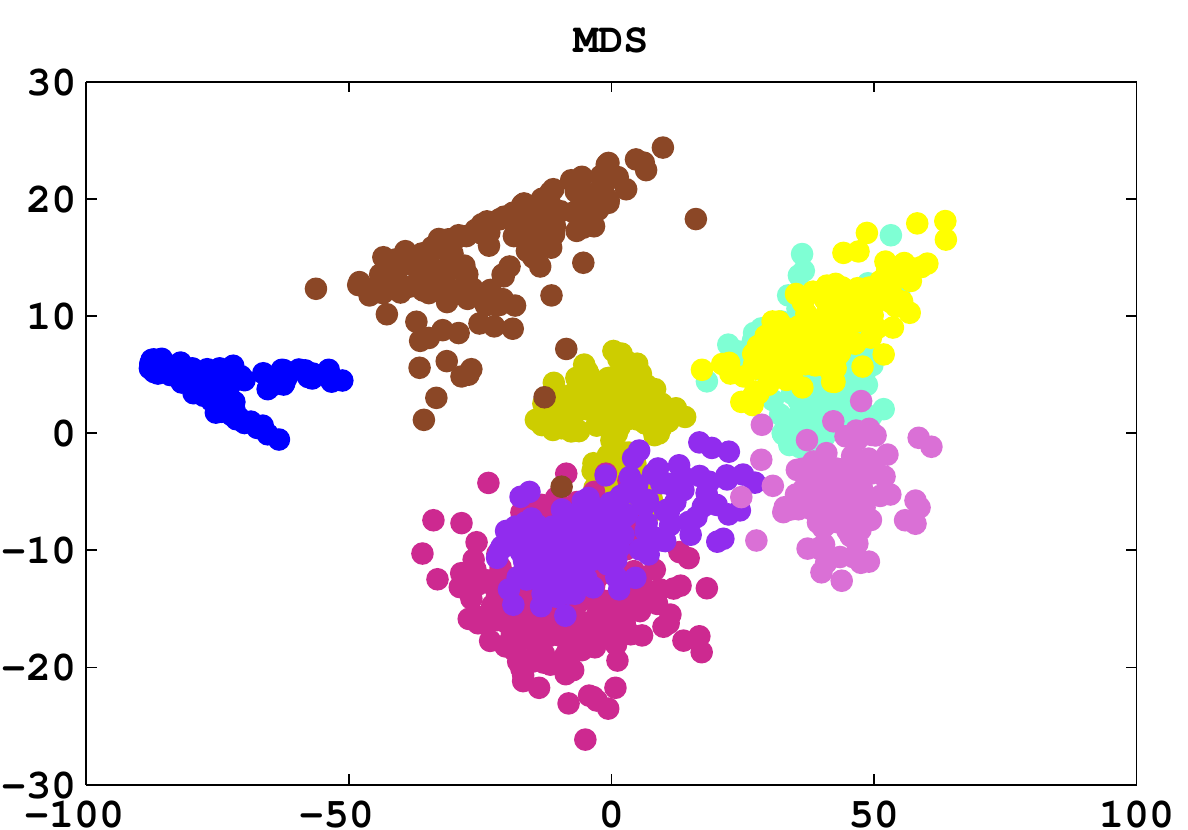}\\
\includegraphics[width=2.3in]{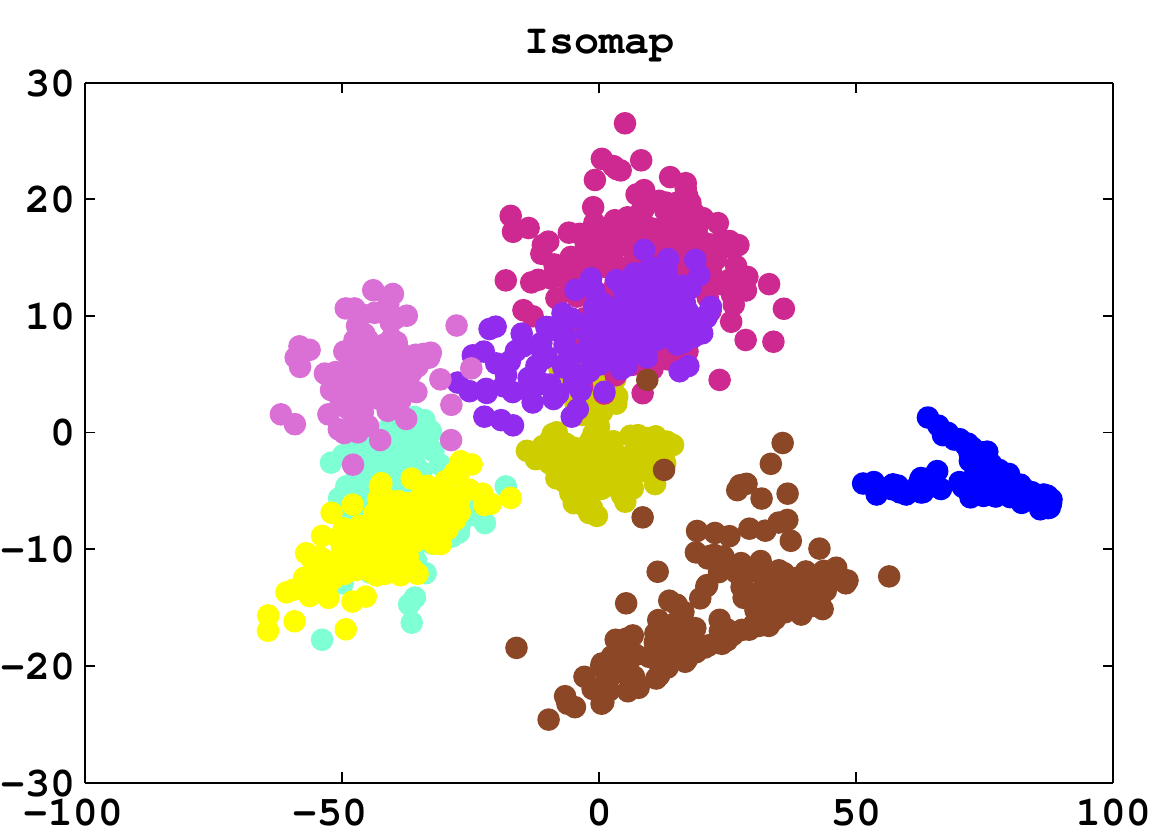}\hspace{0.3cm}\includegraphics[width=2.3in]{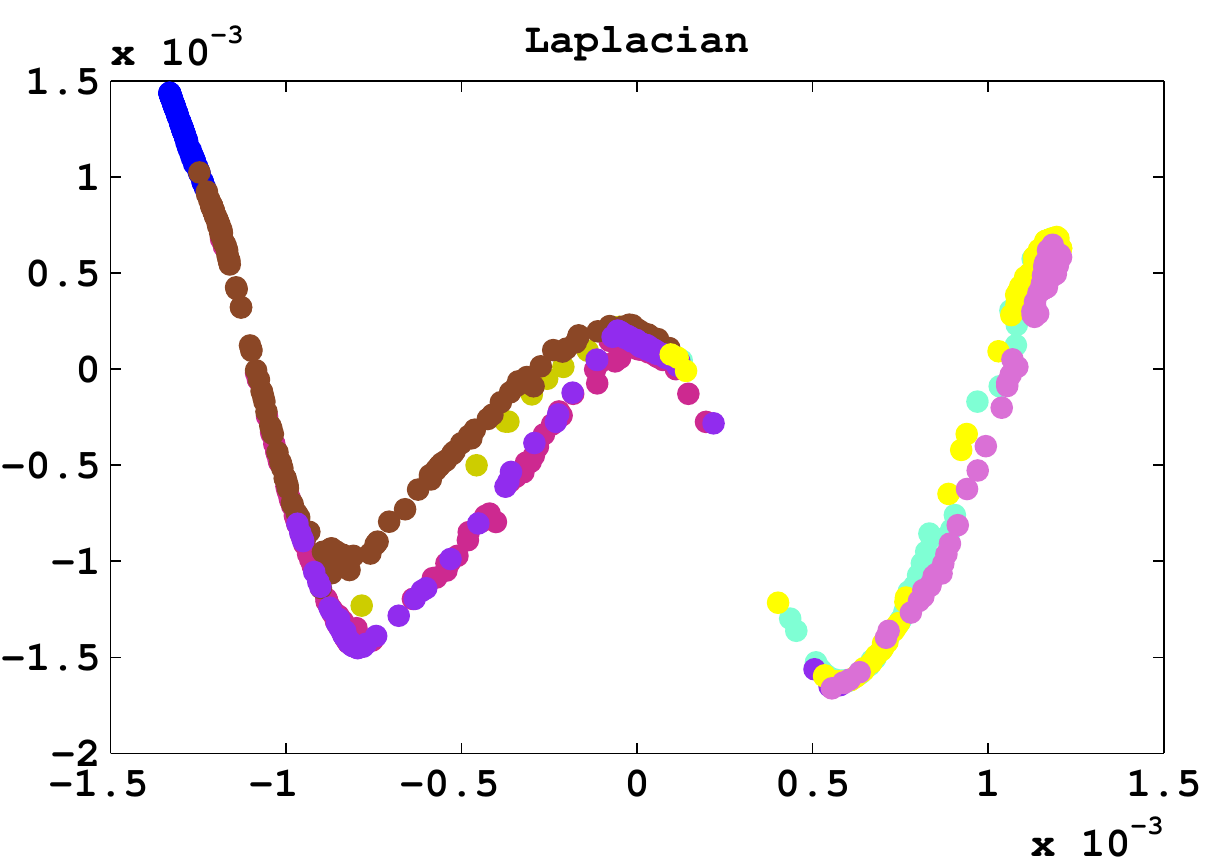}
\caption{Embedding of all labeled samples of the Botswana data.}
\label{fig:bots-embed}
\end{figure}
\clearpage

\begin{figure}
\centering
 \subfigure[]{
\hspace{-0.3cm}\includegraphics[width=2.34in]{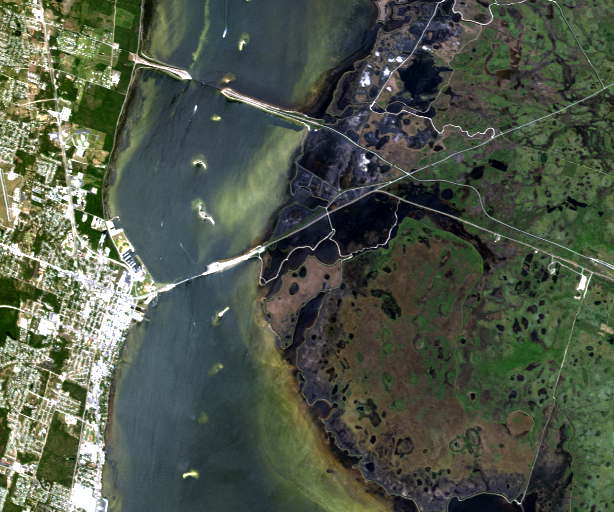}
   \label{fig:ksc}
 } \subfigure[]{
\includegraphics[width=3.51in]{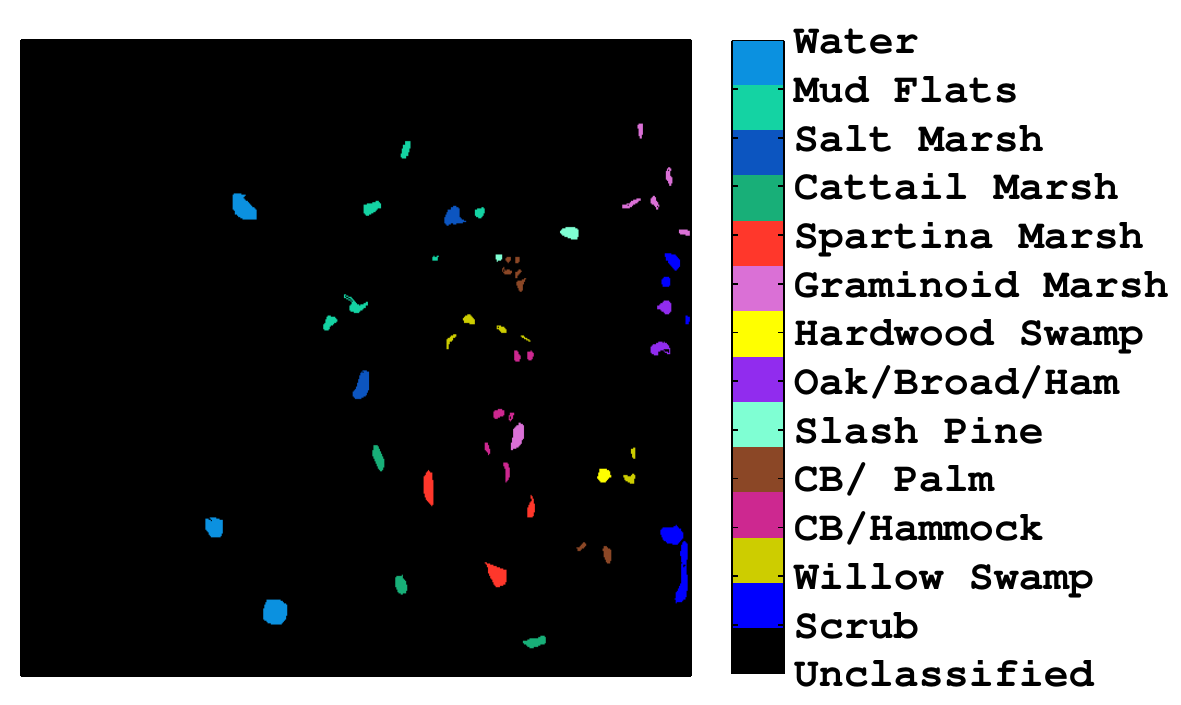}
   \label{fig:ksc-labels}
 }
 \caption{ (a) RGB image of Kennedy Space Center data. (b) Kennedy Space Center labeled data.}
\label{fig:ksc-groundtruth}
\end{figure}

\begin{figure}
\centering
\hspace{0.1cm}\includegraphics[width=2.35in]{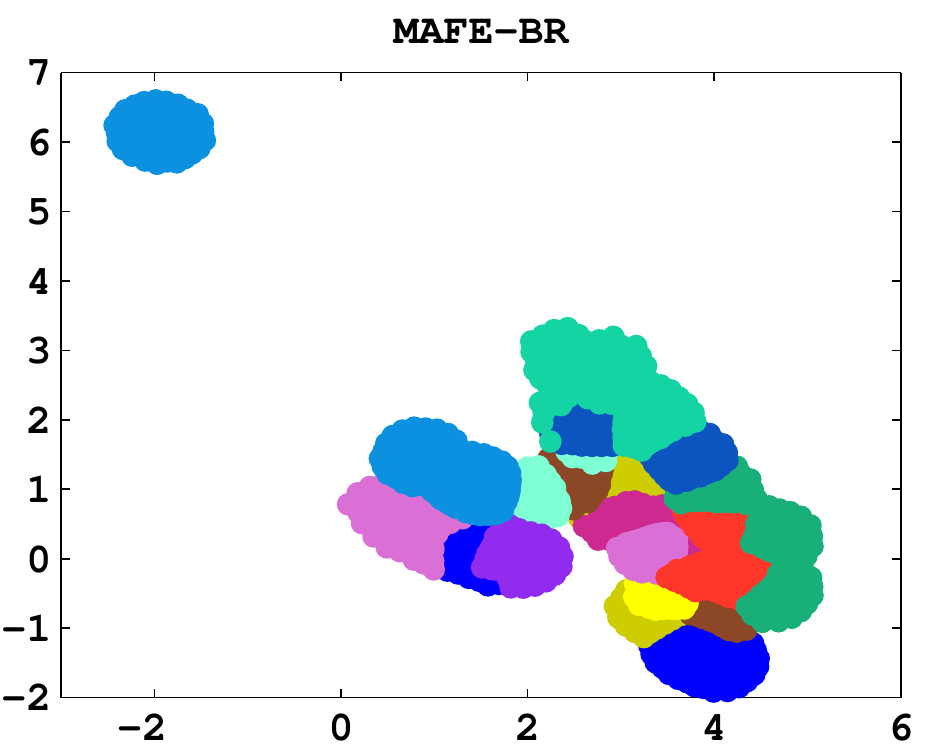}\hspace{1cm}\includegraphics[width=2.42in]{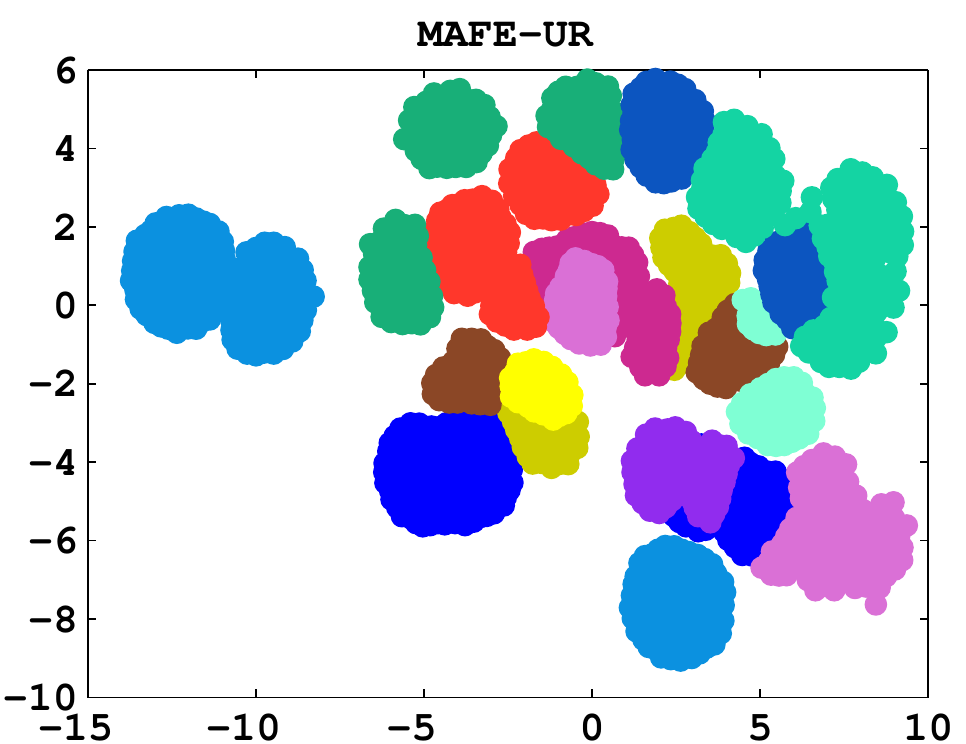}\\
\includegraphics[width=2.45in]{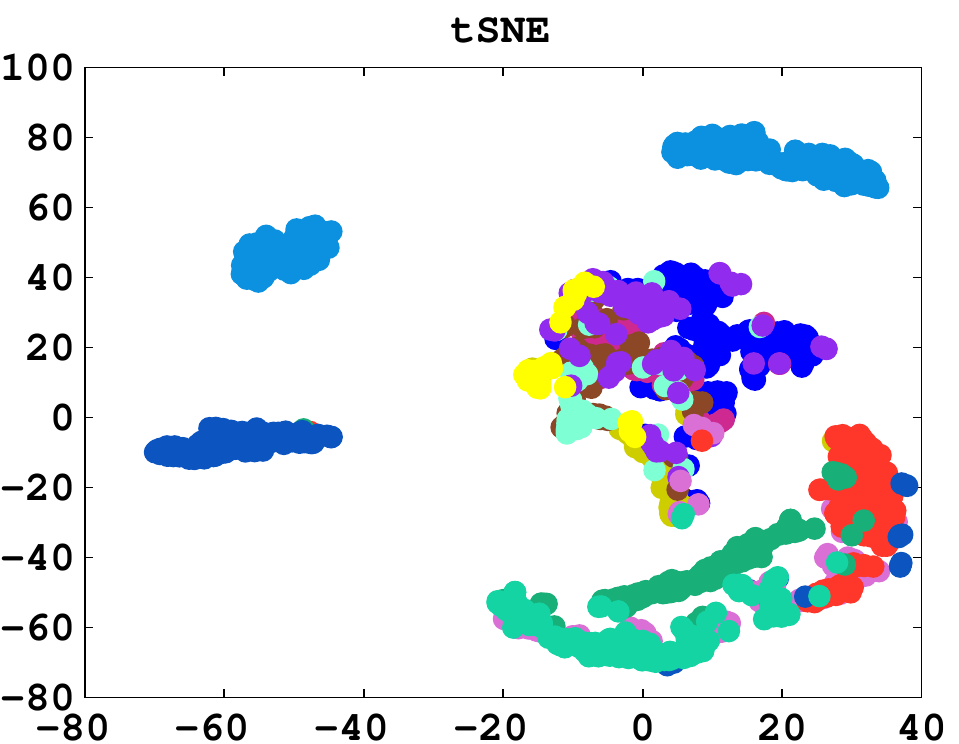}\hspace{0.9cm}\includegraphics[width=2.45in]{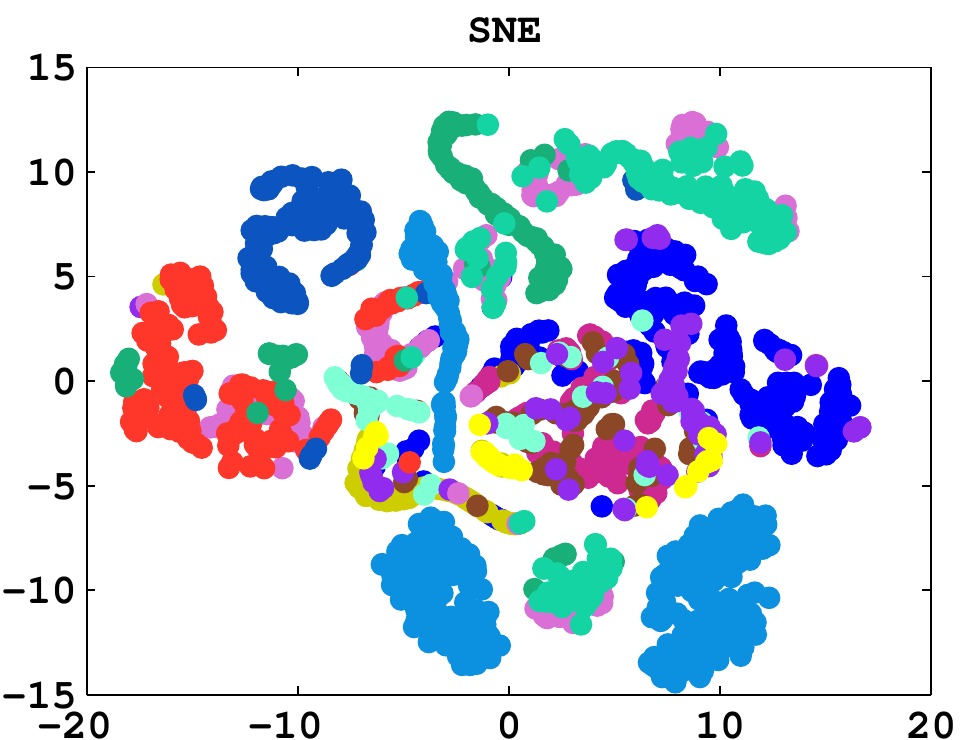}\\
\hspace{-0.1cm}\includegraphics[width=2.56in]{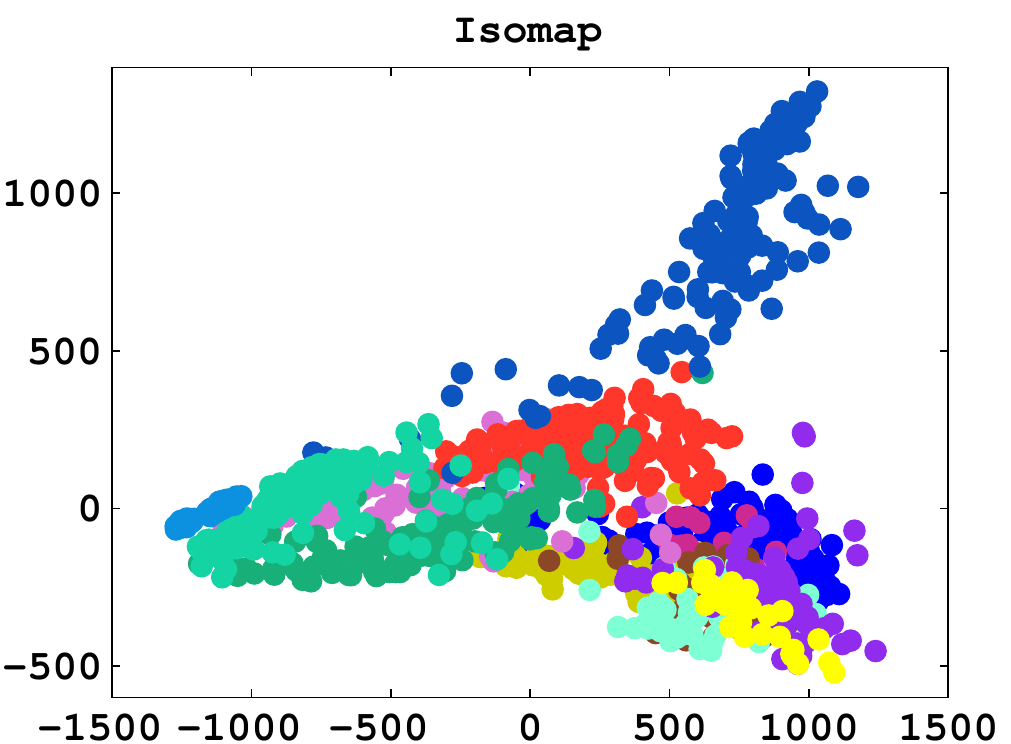}\hspace{0.6cm}\includegraphics[width=2.55in]{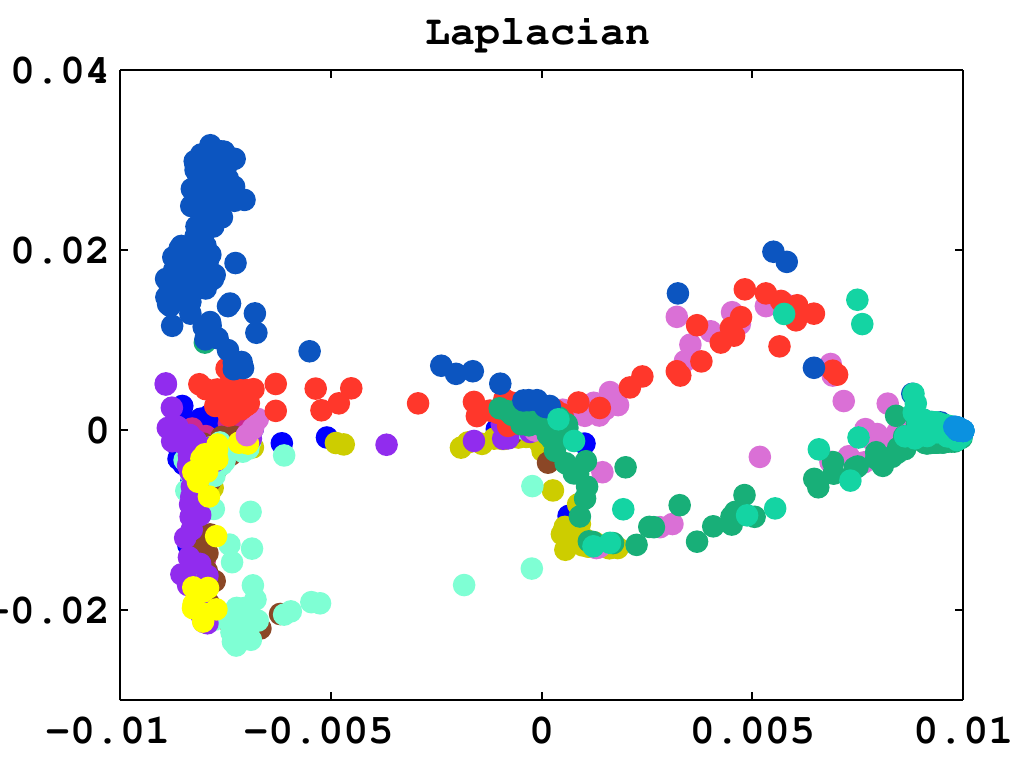}
\caption{Embedding of Kennedy Space Center data (showing $2600$ pixels coordinates).}
\label{fig:ksc-embed}
\end{figure}

\begin{figure}
\centering
\includegraphics[width=3.51in]{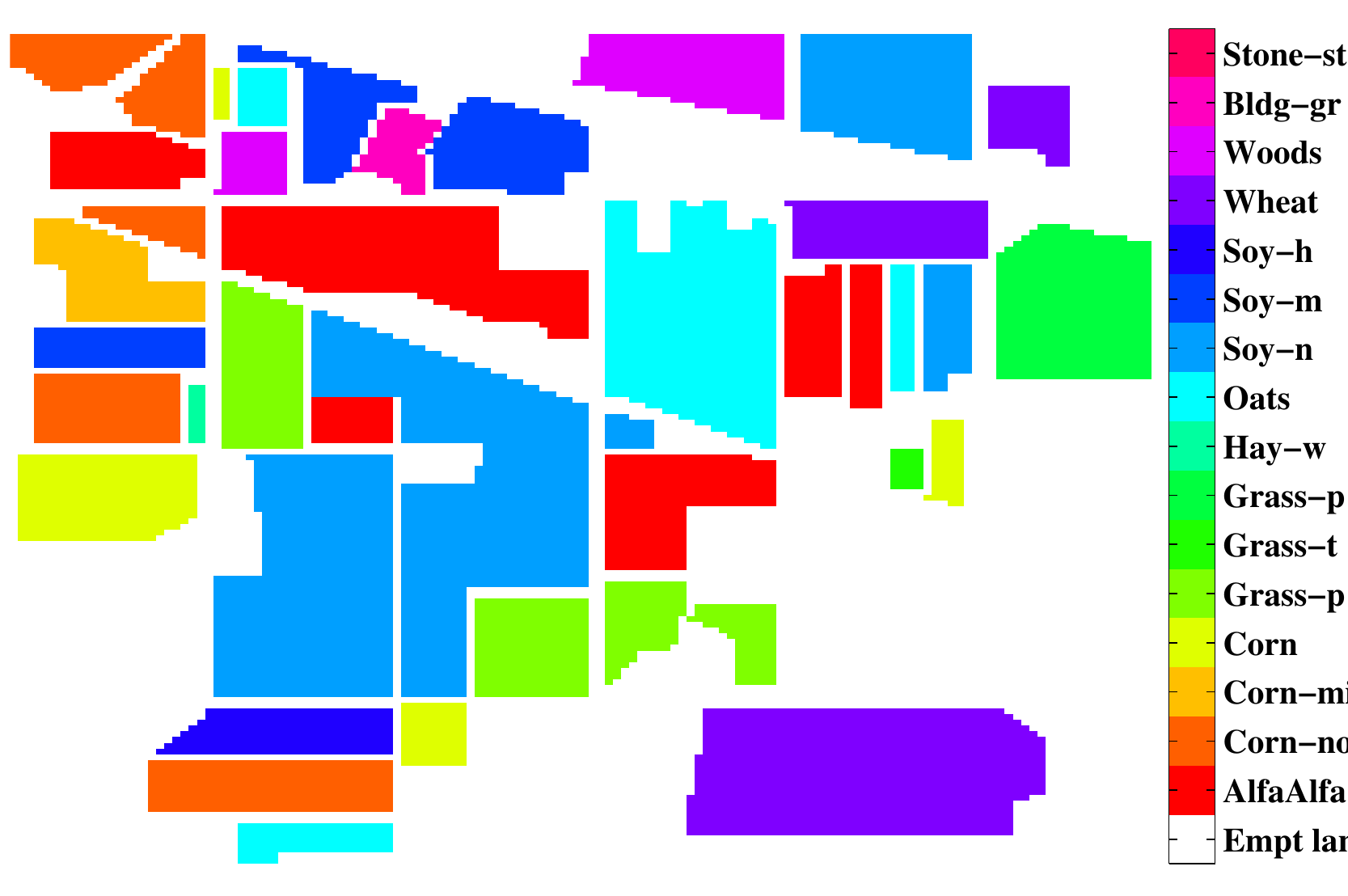}
 \caption{ Indian Pines ground truth image.}
\label{fig:indian-groundtruth}
\end{figure}

\begin{figure}
\centering
\subfigure[MAFE-BR]{
\includegraphics[width=4in]{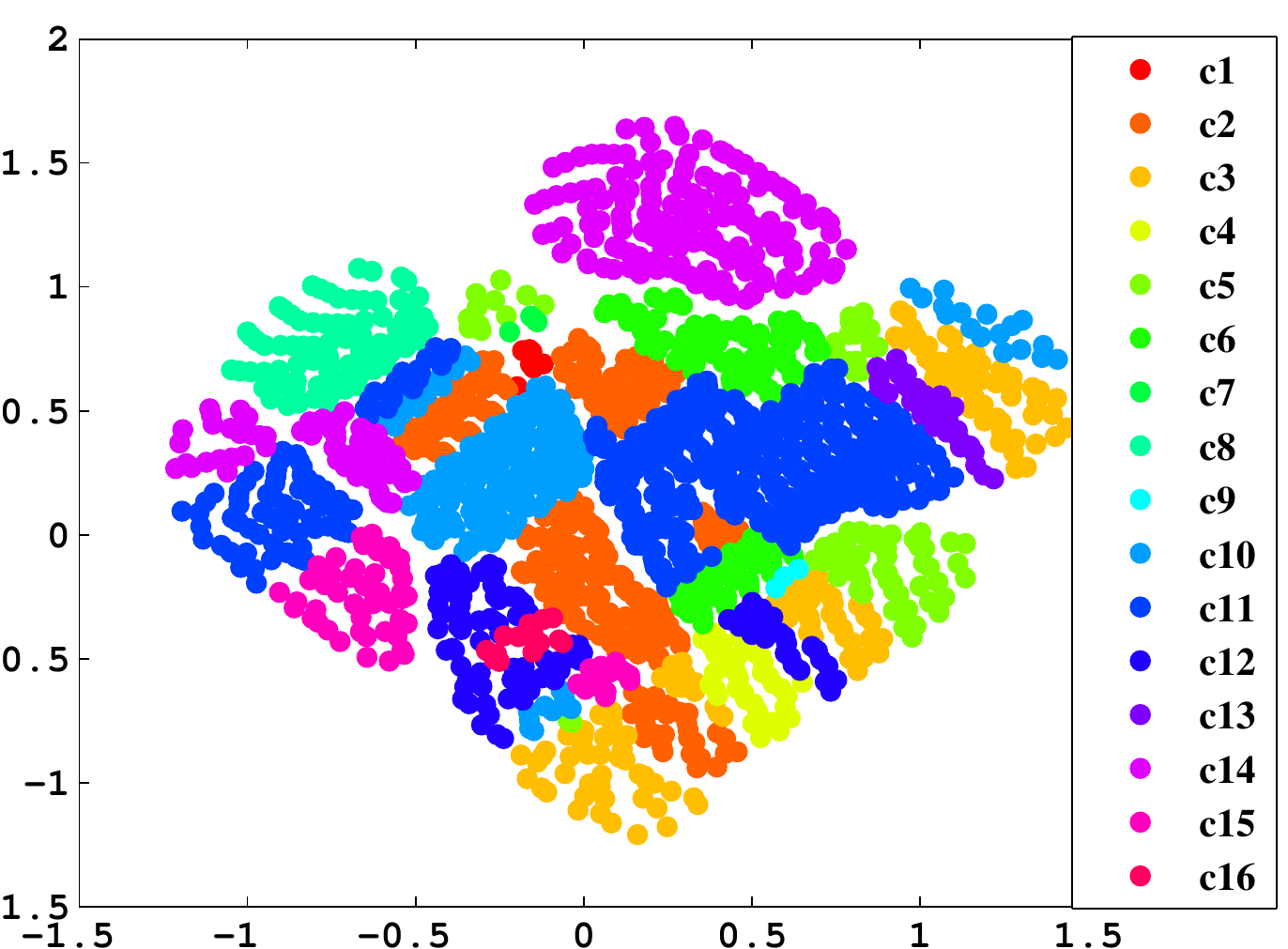}
   \label{fig:mafe-br-indian}
 }
 \subfigure[MAFE-UR]{
\includegraphics[width=4in]{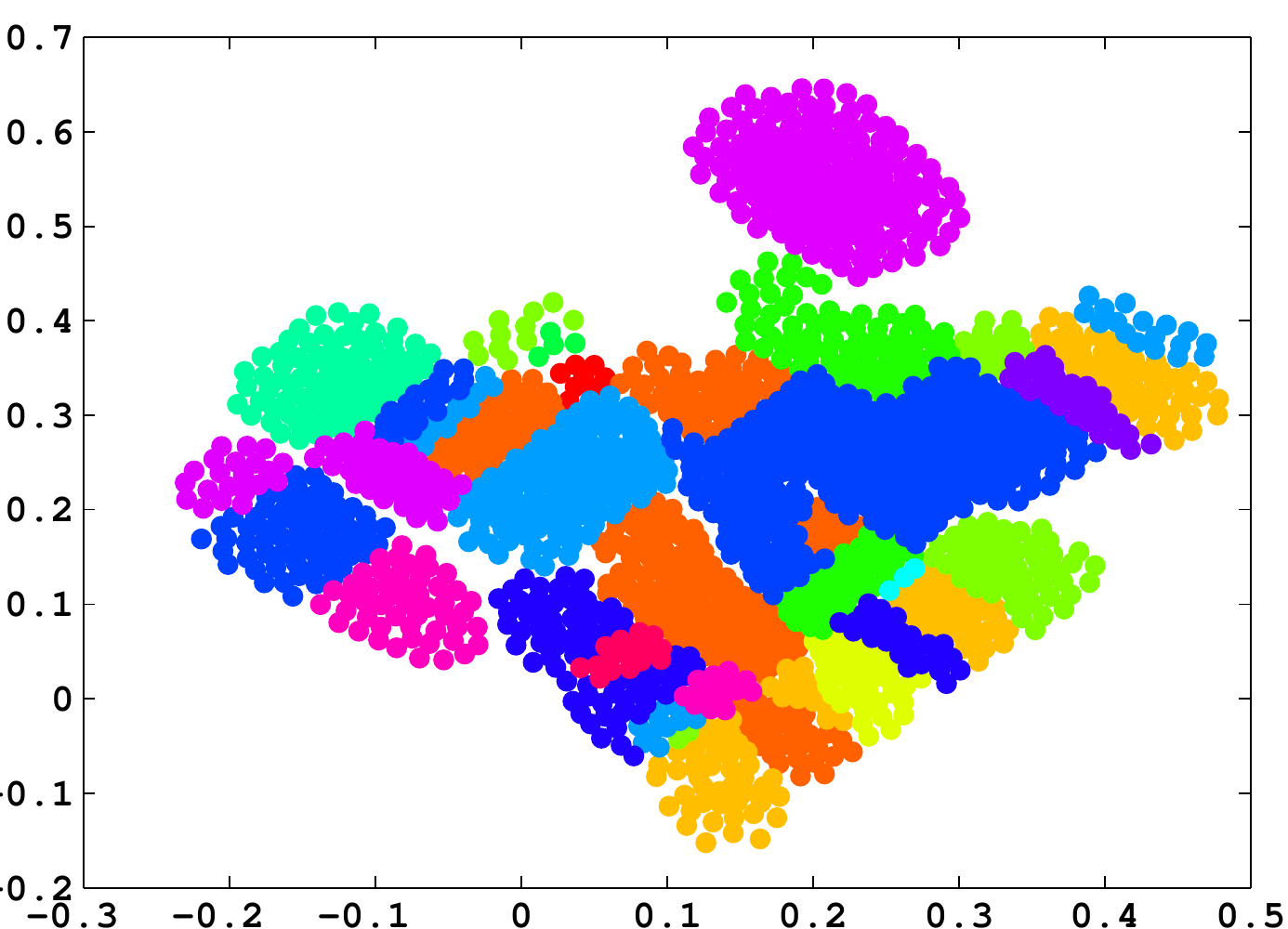}
 \label{fig:mafe-ur-indian}
 }
\caption{Embedding of the Indian Pine Scene(showing $2500$ samples to avoid clutter).}
\label{fig:indianpine-embed}
\end{figure}

\begin{figure}
\centering
\includegraphics[width=3.51in]{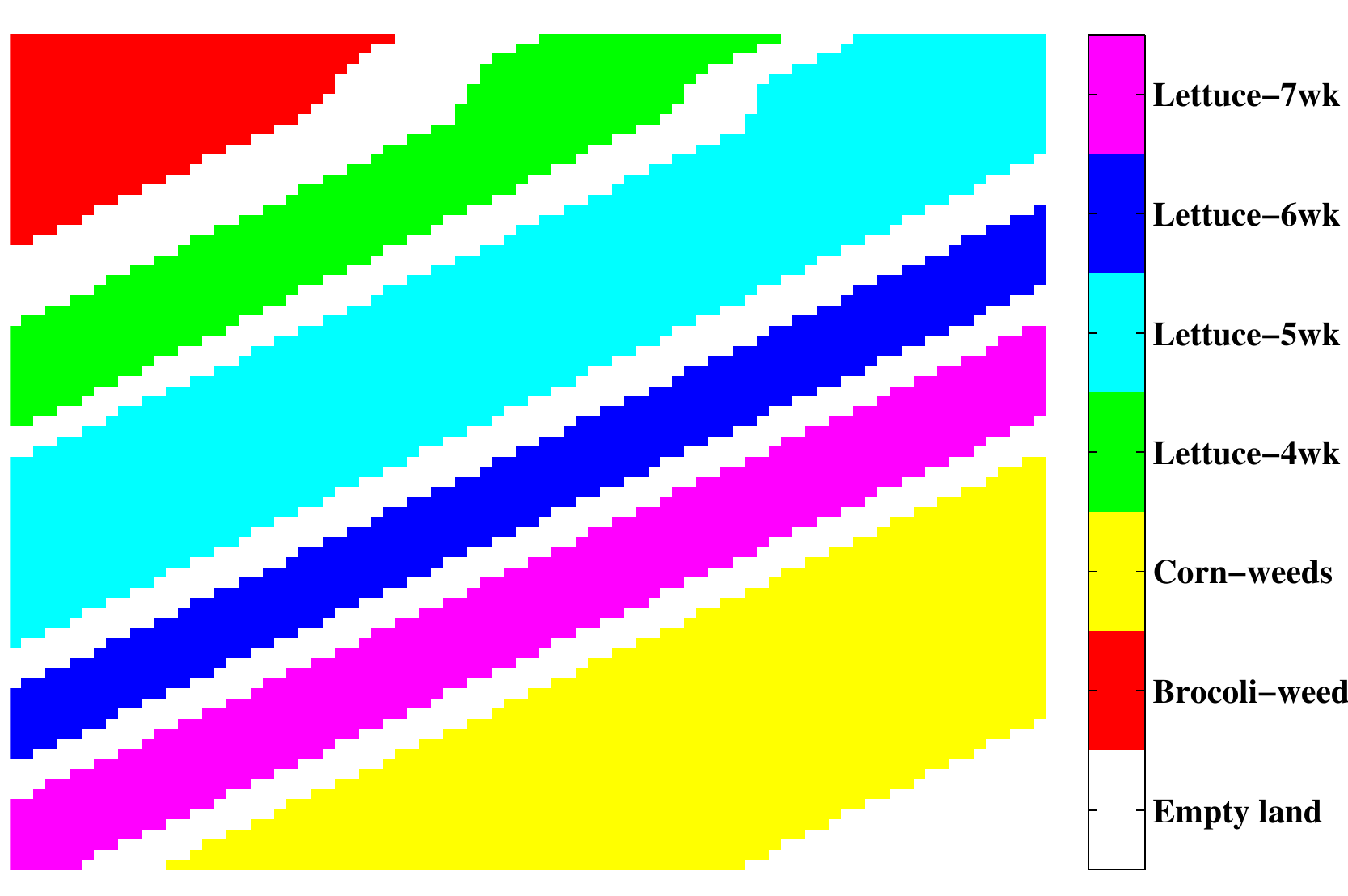}
 \caption{Salinas-A ground truth image.}
\label{fig:salinasA-groundtruth}
\end{figure}

\begin{figure}
\centering
\subfigure[MAFE-UR]{
\includegraphics[width=4in]{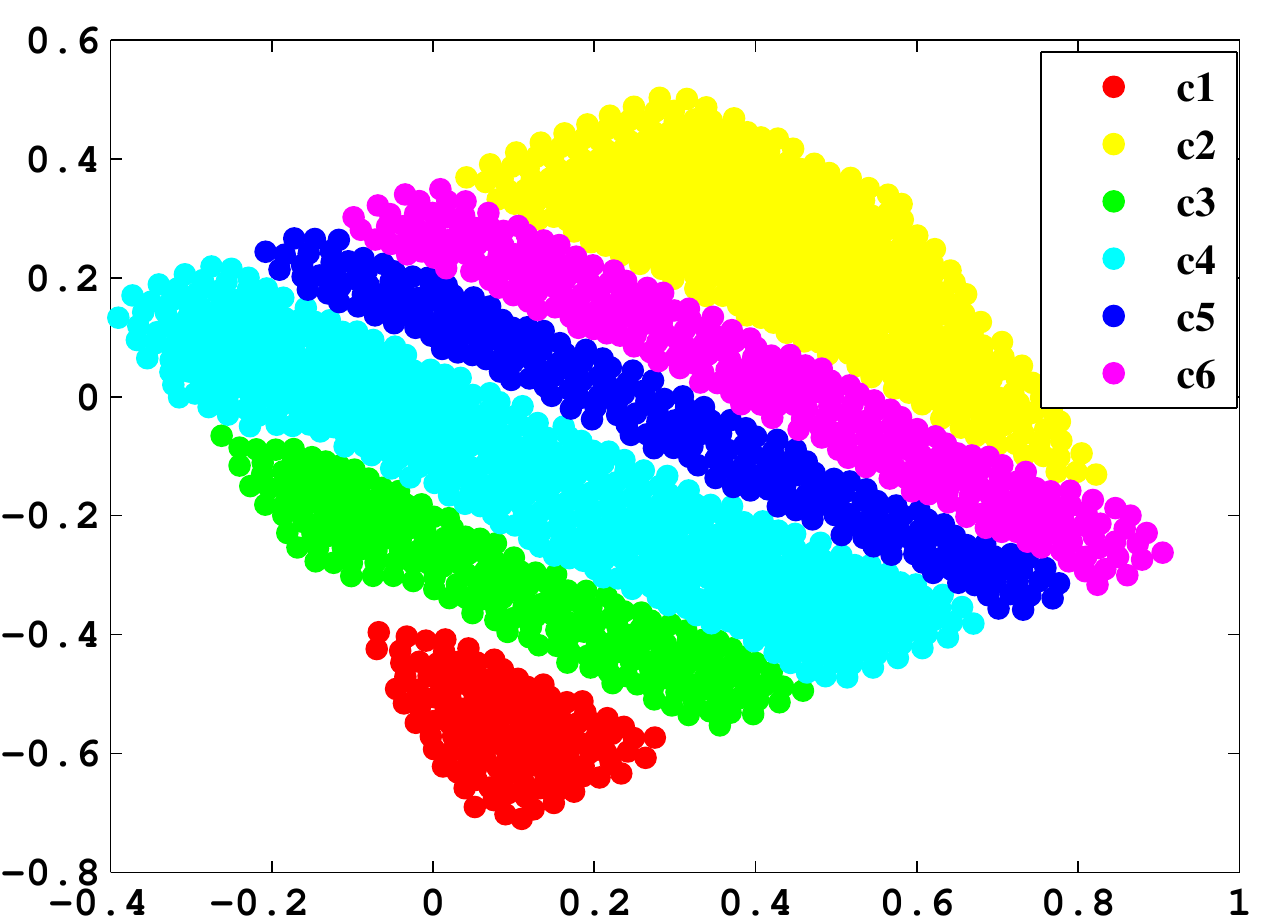}
   \label{fig:mafe-ur-salinaA}
 }
 \subfigure[MAFE-BR]{
\includegraphics[width=4in]{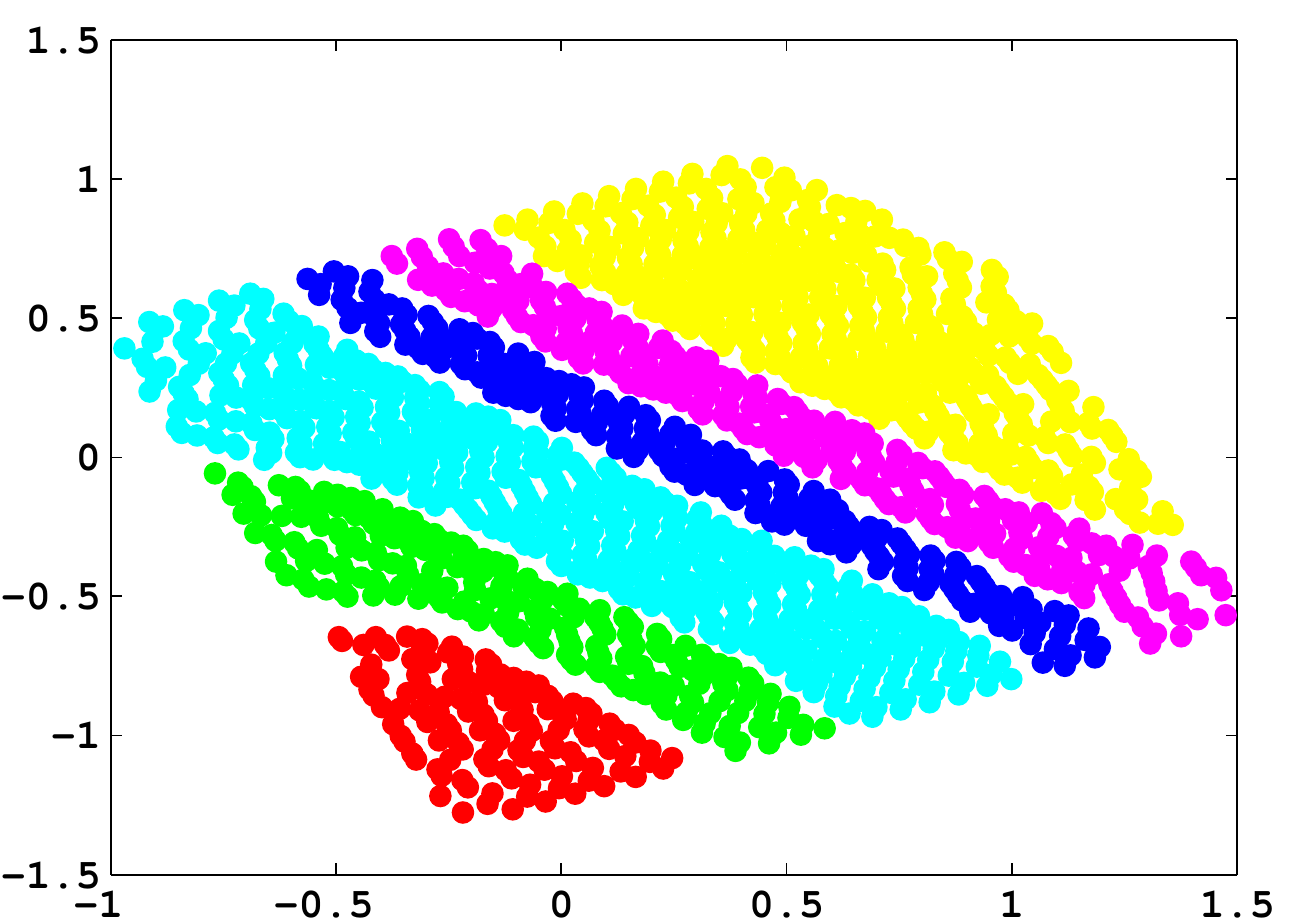}
 \label{fig:mafe-br-salinaA}
 }
\caption{Embedding of the Saline-A Image Scene (showing $2000$ samples to avoid clutter).}
\label{fig:salinaA-embed}
\end{figure}

\subsection{Frobenius Distance}
As in many dimensionality reduction methods, the goal is to maintain the local structure of the data. Here, the study seeks to examine the local topology of manifolds obtained by MAFE models in comparison to other iterative optimized embedding algorithms {\em i.e.} SNE, and tSNE. The most straightforward approach to assess this is by the norm of the residual distance matrix:
\begin{eqnarray}
\|\bmath{D}-\hat{\bmath{D}}\|_{\mathbb{F}} = \sqrt{\sum_{ij}(\hat{D}_{ij} - D_{ij})^{2}}
\label{frobenius}
\end{eqnarray}
where $\|\cdot\|_\mathbb{F}$ denotes the Frobenius norm. It is the square root sum of the squares of elements of $\bmath{D}$ and $\hat{\bmath{D}}$, the high and low dimensional Euclidean distance matrices, respectively. The results shown in Figure \ref{fig:frobeniusdistance} indicate that MAFE models achieve better local distance preserving representations in the sense of \eqref{frobenius}. The results also highlight the efficiency with which the optimization proposed for MAFE based models enables the discovery of neighboring related coordinates. As the number of iterations increase, SNE continues to show reduced Frobenius error due to its small magnitude repulsion which in turn has a negative effect as it leads to overcrowding of maps.

\begin{figure}
\centering
\includegraphics[width=5in]{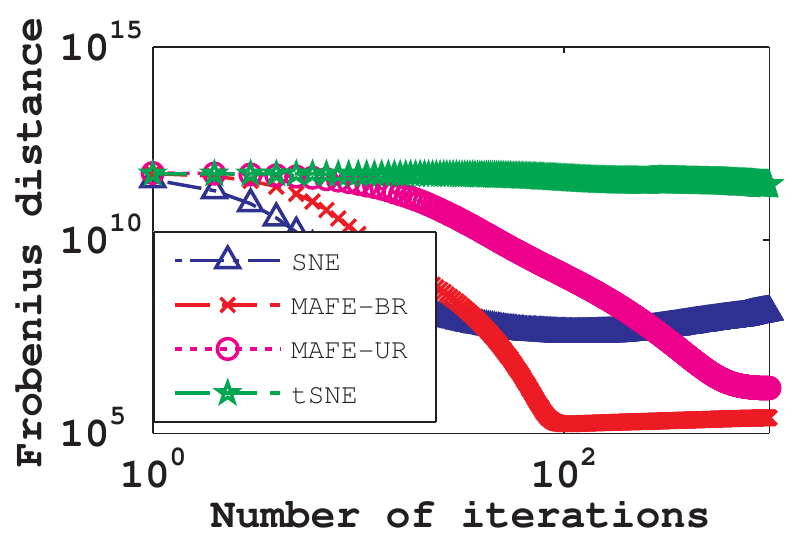}
\caption{Illustration shows how the quality of local distance preservation varies with the number of iterations. The quality is defined by the Frobenius norm of the difference distance matrix , for  Kennedy Space Center. MAFE related models and the SNE model show generally the same performance at preserving the local distances. The Frobenius norm results obtained from tSNE shows poor quality indicating that the method does not preserve local distance for all observed. This is due to its nature of splitting clusters of data that belong to the same class.}
\label{fig:frobeniusdistance}
\end{figure}

In the case of tSNE, its clear that it does not preserve local distances due to its strong repulsion forces. tSNE has a tendency of separating maps of similar instances into multiple small separated clusters, a property that might be useful for visualization and not so suitable for tasks that require keeping similar maps in close proximity, \eg COIL20 data visualization and hyperspectral data embedding.

\begin{figure}
\centering
\includegraphics[width=5in]{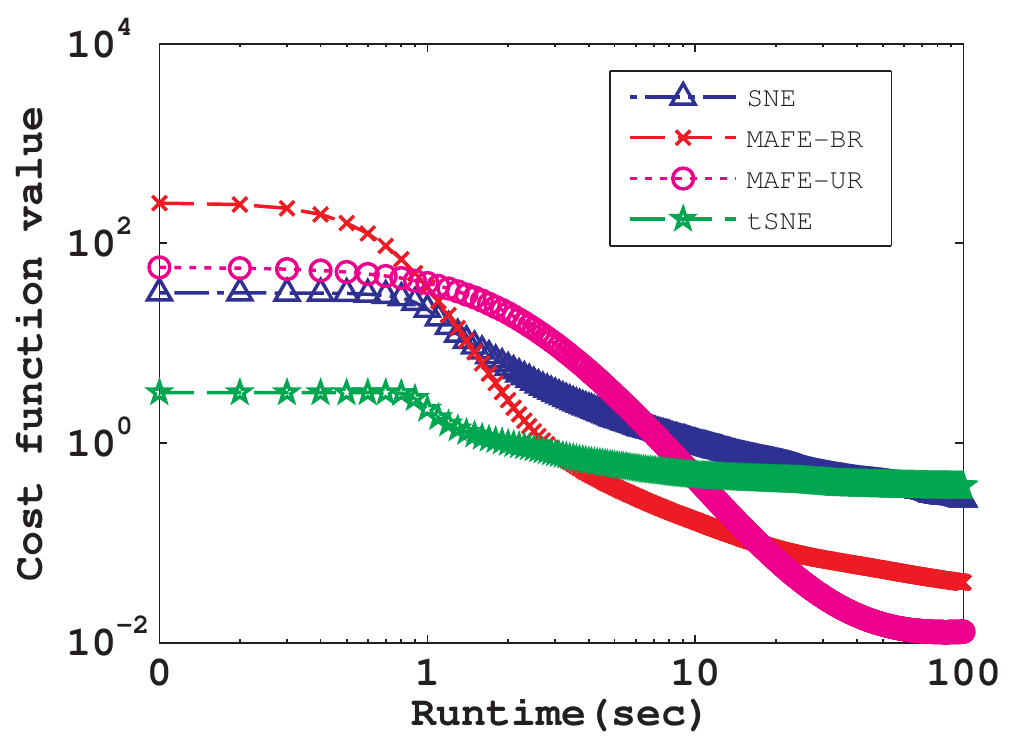}
 \caption{ MAFE, SNE and tSNE objective function optimization run times in seconds for  KSC data set.}
\label{fig:costfunctiondescrease1}
\end{figure}

\subsection{Semisupervised Classification}
In various image analysis applications, including remote sensing, the main task entails establishing an automated image understanding process through object classification or land cover classification. Class label information is used to determine the designation of test or query samples in the lower dimensional space. Our experiments employ a one nearest neighbor(1NN) classification performance accuracy, based on the spectral angle mapper\cite{Yuhas92,Kruse93}, to determine whether the defined object classes or earth land covers, do occupy separable volumes in the lower dimensional space and if they can be discriminated successfully. The experiments are carried out in various embedding spaces. Embedded pixel maps were randomly sampled to generate $70\%$ training samples and $30\%$ testing samples, with results averaged over $10$ runs. All approaches were compared on the same data samples to maintain consistency in error comparison. The results provide several interesting observations that are consistent with the visualization maps from Chapter \ref{sec:visualization}.

Tables \ref{tab:botsconmat} and \ref{tab:kscconmat} illustrates the 1NN classification performance accuracy  per class. The trends observed on  MAFE-BR and MAFE-UR embedding spaces indicate coordinate representations  that outperform other spaces by significant margins in enabling high classification accuracy.  Performance accuracies are reported by observing the mean values obtained from the Kappa statistic (KS) \cite{cohen60} that is computed as,
\begin{eqnarray}
KS = \frac{N\sum_{i=1}^{|C|}t_{cc}-\sum_{c=1}^{|C|}t_{c+}t_{+c}}{N^{2} -\sum_{c=1}^{|C|}t_{c+}t_{+c}}
\nonumber
\end{eqnarray}
where $N$ is the number of testing samples, $t_{cc}$ indicates the number of samples correctly classified in class $c$, $t_{c+}$ denotes the number of testing samples labeled as class $c$, and $t_{+c}$ denotes the number of samples predicted as belonging to class $c$. $|C|$ denotes the total number of classes in the data. The percentage of correctly predicted samples of the total testing samples provides the overall accuracy (OA) which is also reported in the results.

The classification results obtained on the Botswana data leads to interpretations that are consistent with the visualization analysis of Figure \ref{fig:bots-embed}. In general, all methods seem to provide embedding coordinates that leads to reasonable classification performance accuracy. However, the lowest accuracy results are achieved with the LE embedding representations. Furthermore, the lowest accuracy per class is observed between class 3 (c3) and class 6 (c6) corresponding to the {\em Riparian} and {\em Woodlands} classes, respectively. Lower classification results on c3 and c6 are expected because these two classes are the most difficult to separate, in consistency with the visual results of Figure \ref{fig:bots-embed} - similarly as demonstrated in \cite{Crawford2011,Di2011,Ma10}. LMNN using 1NN achieves the second best result on both data set perhaps owing to its ability to make use of class label information in learning the Mahalanobis distance metric. The objective function for LMNN does have a force field structure in which class label information is used to compute the optimal metric with a goal that k-nearest neighbors always belong to the same class (\ie pulled closer by an attraction term), while example samples from other classes are separated  by a large margin (\ie pushed far by a repulsion term). In contrast, both MAFE-BR and MAFE-UR have objective functions that are formulated as dependent on the distance between pairs of points, and no class label information is used during computation of the maps. When classifying samples from the KSC data, a similar trend is observed with both MAFE based techniques providing a coordinate representation from which a higher 1NN classification performance is achieved. The classification results achieved for class 3(c3), class 4 (c4), class 5(c5), and class 6(c6) indicate the lowest performance in all embedding spaces except for the solution achieved by the MAFE based techniques. Based on the visualization result of Figure \ref{fig:ksc-embed}, these classes correspond to the {\em Cabbage Palm Hammock, Cabbage Palm/Oak Hammock, Slash Pine,} and {\em Oak/Broadleaf Hammock}, respectively. These are all categories of very similar upland trees. Their spectral signatures are mixed and often exhibit only subtle differences. However, the neighborhood graph structure computed from the bilateral local kernel function which incorporates spatial details plays a significant boost to the separation of different categories of objects. As such, MAFE based methods take advantage of this in addition to the smooth attraction/repulsion interaction between pairwise maps and yields embeddings that are clearly separable. On the other hand, sSNE achieves very similar performance results with its input graph structure defined to take advantage of the spatial information in images even though it suffers from algorithmic instabilities and inefficiencies associated with the optimization of its highly nonlinear objective function.

Figure \ref{fig:hyperclassify} shows the 1NN misclassification error plots as a function of the embedding dimension, \ie $m= 1\sim20$. Such error plots are commonly used in dimensionality reduction algorithms that rely on the so called {\em manifold projection approach} for estimating the dimension of the embedding space. The manifold projection approach is used to estimate the intrinsic dimension of the embedding space by carefully predefining a higher dimensional space neighborhood graph that achieves a lower dimensional representation with better neighborhood preserving topology. Classification error plots provides one such criterion by which the intrinsic dimension of the data is chosen as the lowest dimension that allows capturing most of the variance (\ie regular information) in the data with higher dimensions only adding to the redundancies. The bilateral kernel function proposed for constructing the higher dimensional neighborhood graph allows a significant automated sparsity property to be induced by estimating the covariance structure of the data. However, the study does acknowledge that choosing the intrinsic dimension is an open difficult research question that warranties a deep separate study, here we simply include an experimental approach necessary for evaluating the proposed models. From Figures \ref{fig:bots-classify} and \ref{fig:ksc-classify}, the optimal dimension for mapping both the $145$ dimension Botswana spectral channels and the $176$ Kennedy Space Center spectral channels is $m=6$. The approximate optimal embedding dimension is chosen based on the smallest lowest {\em elbow drop} point displayed on the 1NN classification error plots in the MAFE embedding spaces.

\begin{figure}
\centering
 \subfigure[]{
  \includegraphics[width=3.5in]{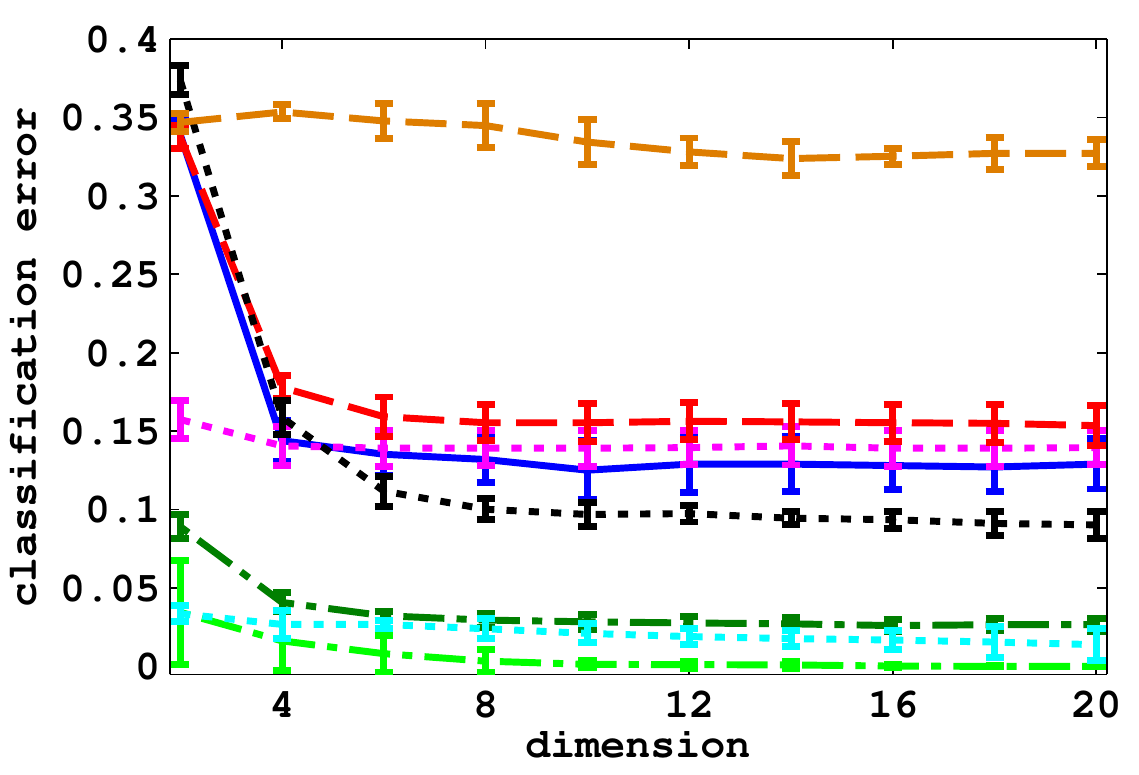}
   \label{fig:ksc-classify}
 } \subfigure[]{
  \includegraphics[width=3.5in]{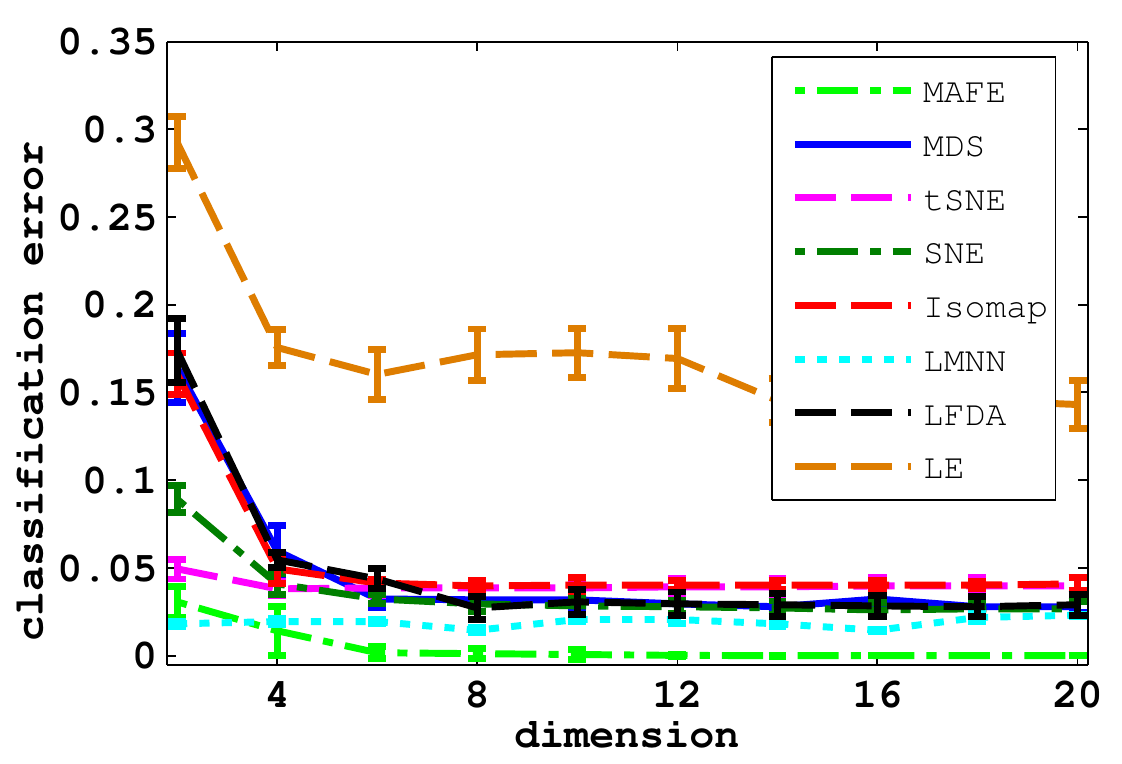}
   \label{fig:bots-classify}
 }
\caption{Mean $\pm$ one standard error misclassification error comparison for 1-nearest neighbor classifier based on various embedding spaces while varying the dimension. (a)  Kennedy Space Center data and (b) Botswana data.}
\label{fig:hyperclassify}
\end{figure}

{\small
\ctable[
    sideways,
	%
    caption = {Botswana Data Classification Results in Different  Embedding Spaces.},
	label   = {tab:botsconmat},
]{lccccccccc}{
	\tnote[$\ast$]{Most difficult classes to separate.}
    \tnote[$\ast\ast$]{High 1NN classification accuracy.}
}{ \FL
Class&\multicolumn{9}{c}{Accuracy on single class (\%)}\\
\cline{2-10}
&{\tt MDS}&{\tt SNE}&{\tt tSNE}&{\tt MAFE-BR}\tmark[$\ast\ast$]&{\tt MAFE-UR}\tmark[$\ast\ast$]&{\tt Isomap}&{\tt LE} & {\tt LMNN}\tmark[$\ast\ast$] & {\tt LFDA}\\
\toprule
c1 & 100& 100& $100$ & ${\bf 100}$ & ${\bf 100}$& $100$ & $100$ & ${\bf 100}$ & $100$\\
 c2 &98.74 & 98.74& $96.86$ & ${\bf 100}$ &${\bf 100}$ &  $100$ & $97.06$ & ${\bf 100}$& $100$\\
 c3\tmark[$\ast$] &  {\bf 89.7}& {\bf 89.09}& {\bf 88.48} & ${\bf 96}$ &${\bf 99.8}$ & {\bf 91.55 }& {\bf 80.28}&${\bf 94.37 }$& {\bf 91.55}\\
 c4 & 97.6& 97.6& $97.6$ & ${\bf 100}$ & ${\bf 100}$ & $100$ & $96.30$& ${\bf 100}$ & $98.11$\\
 c5 &96.09 & 94.53& $95.31$& ${\bf 100}$ & ${\bf 100}$ & $96.36$ & $80$ & ${\bf 100}$ & $94.55$\\
 c6\tmark[$\ast$] & {\bf 92.86}& {\bf 92.14}& {\bf 90} & ${\bf 100}$ & ${\bf 100}$ &{\bf 89.83}& {\bf 69.49}&${\bf 98.31}$ &{\bf 94.92 }\\
 c7 & 96.49& 97.37& $96.49$ & ${\bf 100}$ & ${\bf 100}$ & $97.92$ & $77.08$ &${\bf 97.92}$ & $97.92$\\
 c8 & 96.55& 96.55& $96.55$ &${\bf 100}$ & ${\bf 100}$ & $100$ & $94.59$&${\bf 100}$ & $97.37$\\
 c9 & 96.6& 95& $95$ &${\bf 100}$ & ${\bf 100}$ & $100$ & $97$&${\bf 100}$ & $97.37$\\
\cline{1-10}
 KS &95.08 & 94.75& $93.97$ & ${\bf 99.33}$ &${\bf 100}$ & $95.89$& $88.45$&${\bf 98.82}$& $94.55$
\LL}
}

{\small
\ctable[
     sideways,
	%
	caption = {Kennedy Space Center Data Classification Results in Different  Embedding Spaces.},
	label   = {tab:kscconmat},
]{lccccccccc}{
	\tnote[$\ast$]{Most difficult classes to separate.}
    \tnote[$\ast\ast$]{High 1NN classification accuracy.}
}{ \FL
Class&\multicolumn{9}{c}{Accuracy on single class (\%)}\\
\cline{2-10}
&{\tt MDS}&{\tt tSNE}&{\tt MAFE-BR}\tmark[$\ast\ast$]&{\tt MAFE-UR}\tmark[$\ast\ast$]&{\tt SNE}&{\tt Isomap}&{\tt LE} & {\tt LMNN} & {\tt LFDA}\\
\toprule
 c1&90.91&$90.34$ &${\bf 97.16}$ & ${\bf 98.3}$ &$91.2$ & $ 90.33$ & $72$ & $93.33$ & $92$\\
 c2&90.74&$87.04$&${\bf 100}$&${\bf 100}$& $85.34$ & $62.22$ & $43.48$& $86.96$ & $82.61$\\
 c3\tmark[$\ast$]&${\bf 80}$&${\bf 78.33}$&${\bf 100}$&${\bf 100}$&${\bf 80.4}$ & ${\bf 80.77}$& ${\bf 38.48}$& ${\bf 76.92}$ & ${\bf 88.46}$\\
 c4\tmark[$\ast$]&${\bf 57.89}$&${\bf 45.61}$&${\bf 84.21}$&${\bf 91.21}$&${\bf 51.9}$ & ${\bf 56}$& ${\bf 28}$& ${\bf 68}$& ${\bf 68}$\\
 c5\tmark[$\ast$]&${\bf 51.28}$&${\bf 41.03}$&${\bf 100}$&${\bf 100}$&${\bf 41.20}$& ${\bf 47.06}$& ${\bf 29.41}$& ${\bf 41.18}$& ${\bf 52.94}$\\
 c6\tmark[$\ast$]&${\bf 39.29}$&${\bf 35.71}$&${\bf 82.14}$&${\bf 86.2}$&${\bf 39.30}$ & ${\bf 40}$& ${\bf 28}$& ${\bf 84}$& ${\bf 72}$\\
 c7&80.77&$80.77$&${\bf 100}$&${\bf 100}$&$82.77$ & $90.91$& $36.36$& $90.91$& $100$\\
 c8&74.22&$60$&${\bf 100}$&${\bf 100}$&$63.62$ & $69.05$& $35.71$& $90.48$& $88.10$\\
c9&94.21&$91.74$&${\bf 100}$&${\bf 100}$&$93.56$& $96.15$& $75$& $90.38$& $96.15$\\
c10&89.47&$93.68$&${\bf 100}$&${\bf 100}$&$93.68$ &$92.68$& $80.49$& $100$& $100$\\
c11&93.81&$93.81$&${\bf 98.97}$&${\bf 98.97}$&$93.81$&$97.56$& $85.37$& $95.12$& $97.56$\\
c12&82.20&$80.51$&${\bf 100}$&${\bf 100}$&$81.91$&$92.16$& $80.39$& $98.04$& $98.04$\\
c13&100&$100$&${\bf 100}$&${\bf 100}$&$100$&$100$& $98.39$& $100$& $100$\\
\cline{1-10}
 KS  &83.04 &$80.10$&${\bf 97.86}$&${\bf 99.72}$&$81.90$&$86.40$& $76.71$& $92.1$& $89.2$
\LL}
}

\section{Discussion and Future Work}\label{sec:discussion}
MAFE nonlinear embedding techniques coupled with a bilateral kernel function demonstrated a better approach to account for nonlinear mixtures of similar categories that are captured by high resolution sensors as compared to other techniques. The MAFE general framework opens up valuable avenues coupled with a platform to derive further theoretical insights and development of new nonlinear dimensionality reduction algorithms.

The nonlinear dimensionality reduction algorithms discussed in this study have a kernel dimension that is the square of the number of vectors in the sample space. As such the nature of their objective functions can be summarized as follows.
\begin{itemize}
\item For MAFE based techniques, the objective function $U:\mathbb{R}^{N\times m}\rightarrow\mathbb{R}$ is defined over the space of embedding matrices $\bmath{Z}\in\mathbb{R}^{N\times m}$. The neighborhood graph's pairwise kernel similarities are represented as a $N\times N$ matrix $\bmath{W}$. The complexity scales with the square of the number of observed samples, \ie $O(N^{2})$, both in computation and memory usage.
\item For sSNE, SNE, and tSNE, the objective function $KL:\mathbb{R}^{N\times m}\rightarrow\mathbb{R}$ is defined over the space of embedding matrices $\mathbb{R}^{N\times m}$. The probable neighbors matrix is denoted by an $N\times N$ pairwise kernel matrix $\bmath{W}$. The complexity scales with the square of the number of observations, \ie $O(N^{2})$, both in computation and memory usage.
\item For MDS, the objective function $\sigma_{2}:\mathbb{R}^{N\times m}\rightarrow\mathbb{R}$ is defined over the space of matrices $\mathbb{R}^{N\times m}$. The pixel data in the problem are the geodesic distances, represented as a $N\times N$ kernel matrix $\bmath{D}(\bmath{Y}) = [\mathcal{D}_{\bmath{Y}}(\bmath{y}_{i},\bmath{y}_{j}]$. The complexity scales with the square of the number of samples, \ie $O(N^{2})$, both in computation and memory usage.
\end{itemize}

\subsection{Algorithmic Complexity Challenges}
For non-iterative methods the spectral decomposition of a large dimensional kernel encounters difficulties in at least three aspects: large memory usage, high computational complexity, and computational instability. Methods that are capable of exploiting the sparsity structure in the neighborhood graph may be useful to overcome the difficulties in memory usage and computational complexity. Approaches that incorporate the concepts of rank revealing, randomized low rank approximation algorithms, and  greedy rank-revealing algorithms and randomized anisotropic transformation algorithms, which approximate leading eigenvalues and eigenvectors of dimensionality reduction kernels may lead to faster algorithms. For iterative gradient based methods, efficient second order or Newton approximation algorithms that introduce additional local information about the curvature of the objective function can be developed to speed-up the convergence to the minimum energy configuration state. However, we caution that the computation of gradient directional vectors in each iteration of the optimization scheme do continue to introduce computational hurdles for larger data sets.

As indicated, the complexity of these algorithms scales with the square of the number of pixels i.e. $O(N^{2})$, a number that is prohibitive to obtain efficient solutions on large-scale (order of $10^{5}$ and greater) hyperspectral image pixels. This is a challenge faced by many embedding algorithms. A possible mitigating approach would be to split the pixels into overlapping patches, run a nonlinear embedding method on each patch and use the coordinates of pixels in the overlapping patches as references to merge each pair of resulting submanifolds to obtain a common global manifold. An even harder problem not addressed in this study pertains to the issue of spectral unmixing for pixels that contain more than a single land cover category \cite{Villa11}. If the spatial resolution of sensors is poor, increased overlap of different spectral signatures is imminent for different land cover objects. In such cases, knowledge of the observations should be used to derive kernel representations that can incorporate such characteristics.

\subsection{Objective Function Formulation Challenges}
The formulation of MAFE models does encounter more challenges on the design of the total energy function. For example, the appealing force field embedding intuition is subject to numerous {\em dynamic local maxima behaviors} during cluster formation. This can be explained from observing that the formation of clusters creates local repulsions leading to local traps for maps that still need to move closer to their most similar neighbors (as determined by the neighborhood graph weights). A simple illustration of the dynamic local maxima behavior is shown in Figure \ref{fig:local-minima}. The behavior is exacerbated by a weak choice on the attraction potential energy function which may lead to less effective pulling of trapped maps to their closest similar neighbors.  This shortcoming does affect the convergence of the iterative embedding algorithms. The local maxima generated traps can perhaps be better handled by piecewise attraction potential functions whose short range and long range capability vary at different distances in the hope of increasing the pulling force's magnitudes to overcome the local repulsion forces from dissimilar neighbors.

\begin{figure}[htpb]
  \vspace{-8pt}
  \centering
    \includegraphics[width=2.5in]{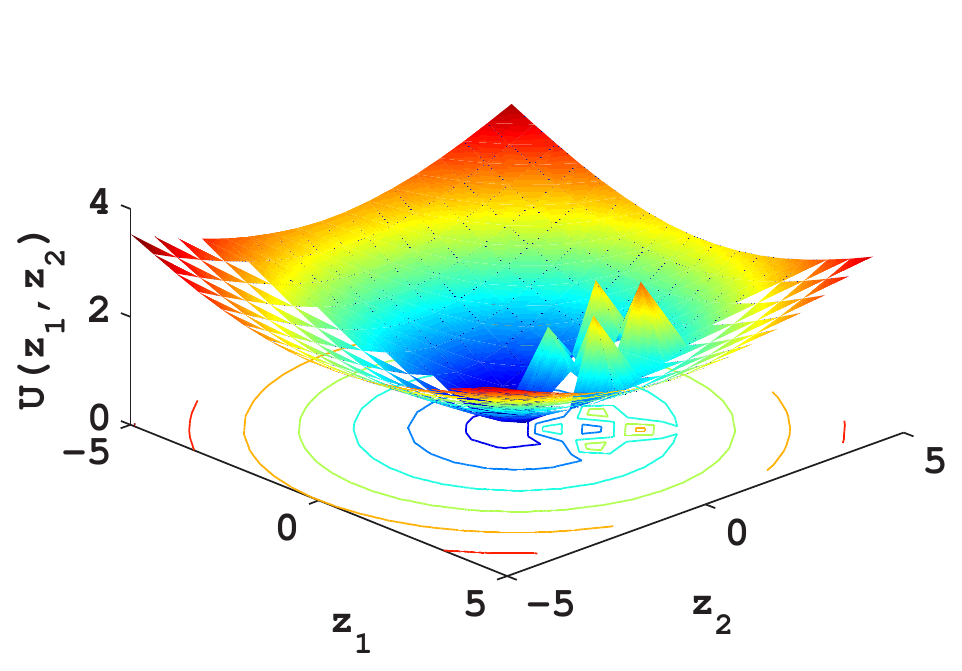}
    \includegraphics[width=2.5in]{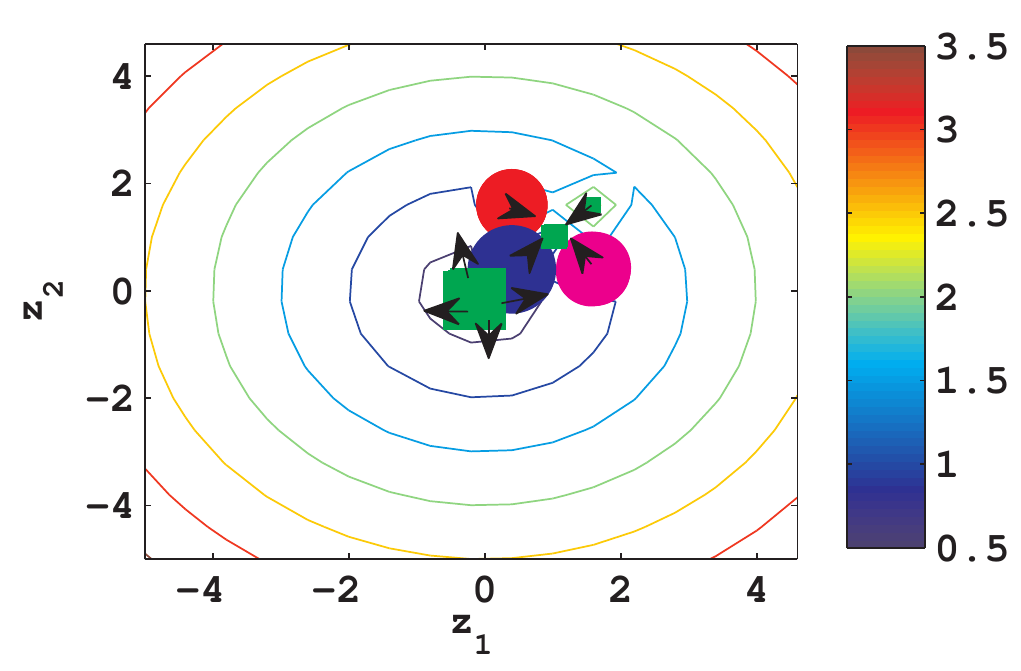}
  \vspace{-5pt}
  \caption{Illustration of dynamic local maxima traps from a MAFE based model. Showing the 2D surface and contour plots. Arrows indicate the repulsive force fields that form during clustering. The disjoint cluster formation caused by the local repulsions on the green color class are shown on the right.(Best viewed in color.)}
  \label{fig:local-minima}
\end{figure}

\section{Conclusions}\label{sec:conclusions}
 This report provides design and algorithmic insights leading to new nonlinear dimensionality reduction models with strong connections to widely used methods. It introduces a novel parameterized joint spatial and photometric distance based bilateral kernel function that improves the capability of capturing regularities within hyperspectral image pixels. Much of the disjoint relations is encoded into the neighborhood graph through spatial information and estimation of the covariance structure using the sparse matrix transform. The study then adapts a graph embedding framework, namely the multidimensional artificial field embedding, featuring the example bounded and the unbounded repulsion models,  and demonstrates the general applicability of the force field intuition to problems in signal and image processing. The idea is to envision the motion of positional maps (representing graph vertices) as motivated by attraction and repulsion forces that enable a long range distance pulling of similar maps and a short range distance repulsion for all maps, respectively. The force field interpretation allows a natural way to capture this imagination and promotes the design of pair-wise interactions functions that act on pixel samples to establish the direction and magnitude of the interaction vectors. An adaptive iterative gradient based algorithm based on this notion was further implemented to yield the minimum energy configuration of the neighborhood graph. The proposed MAFE-UR and MAFE-BR models were applied to data sets acquired from remote sensing. The new algorithms proposed in the study are shown to have desirable properties that preserves the local topology of observations while inducing strong natural global structures, \eg disjoint spatially motivated clusters in hyperspectral imagery. In its general form, the framework yields formulations of current popular dimensionality reduction methods with very few assumptions. Experimental work conducted on visualization, gradient field trajectories and semisupervised classification tasks demonstrates that both MAFE-UR and MAFE-BR do overcome the crowding problem and lead to better classification performance in comparison to other approaches.

\bibliographystyle{unsrt}
\bibliography{generalizedmafe}

\end{document}